\documentclass[a4paper,11pt]{article}
\newcounter{numrel}
\counterwithin*{numrel}{section}

\usepackage{mattsstyle}
\usepackage[normalem]{ulem}
\usepackage{enumitem}
\usepackage{bbm}
\usepackage{bm}
\usepackage{hyperref}
\usepackage{mathtools}
\mathtoolsset{showonlyrefs}

\def\dWp#1{d_{\mathrm{W}^{#1}}}
\def\dWlinpmu#1#2{d_{\mathrm{LW}^{#1},#2}}
\def\TLp#1{\mathrm{TL}^{#1}}
\def\Probac{\cP_{\mathrm{a.c.}}}

\title{Laplace Learning on Wasserstein Space}
\author[1]{Mary Chriselda Antony Oliver}
\author[1,2]{Michael Roberts}
\author[1]{Carola-Bibiane Sch\"{o}nlieb}
\author[3]{Matthew Thorpe}
\affil[1]{Department of Applied Mathematics and Theoretical Physics, The University of Cambridge, Wilberforce Rd, CB3 0WA, Cambridge, UK,}
\affil[2]{Department of Medicine, The University of Cambridge, Cambridge
 Biomedical Campus, CB2 0AW, Cambridge, UK,}
\affil[3]{Department of Statistics, The University of Warwick, Coventry, CV4 7AL UK.}
\date{}

\begin{document}

\maketitle

\begin{abstract}
The manifold hypothesis posits that high-dimensional data typically resides on low-dimensional subspaces. In this paper, we assume manifold hypothesis to investigate graph-based semi-supervised learning methods. In particular, we examine Laplace Learning in the Wasserstein space, extending the classical notion of graph-based semi-supervised learning algorithms from finite-dimensional Euclidean spaces to an infinite-dimensional setting. To achieve this, we prove variational convergence of a discrete graph $p$-Dirichlet energy to its continuum counterpart. In addition, we characterize the Laplace-Beltrami operator on a submanifold of the Wasserstein space. 
Finally, we validate the proposed theoretical framework through numerical experiments conducted on benchmark datasets, demonstrating the consistency of our classification performance in high-dimensional settings.
\end{abstract}

\noindent

\noindent

\section{Introduction}

The curation of large-scale, fully annotated training datasets remains a major bottleneck due to the high cost and expertise required for manual labelling. For example, in biomedical imaging applications such as flow cytometry \cite{Jun_2021,Bruggner_2014}, gene expression microarrays \cite{Dudoit,Fan_2006}, and proteomic assays \cite{Clarke_2008}, modern technologies generate high-dimensional data far faster than it can be annotated. As a result, only a small fraction of samples receive reliable labels, despite their routine use in classification tasks. This motivates graph-based semi-supervised methods, which exploits the geometric structure of the data to improve predictions with limited supervision. In this paper, we focus on a special class of graph-based semi-supervised methods, namely Laplace Learning \cite{Zhu_2003}, to study classification in high-dimensional settings. This method exploits the geometric structure inherent in large quantities of unlabelled data to improve label predictions. However, leveraging the underlying geometry in high-dimensional datasets presents substantial challenges, including the well-known curse of dimensionality \cite{Koppen,Duda} and poor generalization capacity \cite{Clarke_2008}. 

In theory, a well-established trend in statistics suggests that high-dimensional data often possess an intrinsic low-dimensional structure, a concept formalized by the \textit{manifold hypothesis} \cite{Charles_2016}. This hypothesis asserts that data are supported (or nearly supported) on a low-dimensional manifold with a small intrinsic dimension. Formally, the general setup of this hypothesis motivates us to study Laplace Learning in more complex spaces, such as the Wasserstein space which we denote by $\Probac{(\Omega)}$ for $\Omega \subset \bbR^k$. Here, $\Probac{(\Omega)}$ denotes the set of probability measures on $\Omega$ that are absolutely continuous with respect to the Lebesgue measure. Its infinite-dimensional Riemannian geometry was introduced in \cite{otto2001geometry}. In particular, we work `close' to a submanifold $\Lambda \subset \Probac{(\Omega)}$ embedded in the Wasserstein space with a finite-dimensional parameterisation, and denote samples by $\mu_i \iid \bbP_\Lambda \in\mathcal{P}(\Lambda)$. For our specific setting, we are interested in the finite approximation, $\Lambda_n = \{\mu_i\}_{i=1}^n$, of the submanifold $\Lambda \subset \Probac{(\Omega)}$. In practice, we have access to $m$ data points for each of the $n$ probability measures $\mu_i$, and we denote the feature vectors by $\Lambda^{(m)}_n = \{\mu^{(m)}_i\}_{i=1}^n$ with 
$\mu^{(m)}_i= \frac{1}{m}\sum_{j=1}^m \delta_{x^{(i)}_j}$ where $\Omega \ni x^{(i)}_j \iid \mu_i$ for all $j=1,2,\dots,m$ and $i=1,2,\dots,n$. Note that $\Lambda_n^{(m)} \not\subset \Probac(\Omega)$; it does not lie in $\Lambda$, but  serves as an empirical (perturbed) approximation of the submanifold.

Our submanifold assumption implies that we consider a subset $\Lambda\subset\Probac{(\Omega)}$ of the Wasserstein space which is parameterised by a bi-Lipschitz embedding $E: \mathcal{S} \to \Lambda$ where $\mathcal{S} \subset \bbR^d$ is a smooth, compact and finite-dimensional manifold. We consider probability measures $\bbP_\cS \in\mathcal{P}(\cS)$ and $\bbP_\Lambda\in\mathcal{P}(\Lambda)$ that are related by $\bbP_\Lambda = E_{\#}\bbP_\cS$. This allows us to assume $\mathcal{S} \ni \theta_i \iid \bbP_\mathcal{S}$ and define $\mu_i = E(\theta_i)$ which implies $\mu_i\iid\bbP_\Lambda$. We let $\mathcal{S}_n=\{\theta_i\}_{i=1}^n$ be the finite sample approximation of the parameter space $\mathcal{S}$. The empirical distributions using $\{\theta_i\}_{i=1}^n$, $\{\mu_i\}_{i=1}^n$, and $\{\mu_i^{(m)}\}_{i=1}^n$ are denoted by $\bbP_{\cS_n} = \frac{1}{n}\sum_{i=1}^n \delta_{\theta_i}$, $\bbP_{\Lambda_n} = \frac{1}{n}\sum_{i=1}^n \delta_{\mu_i}$ and $\bbP_{\Lambda_n^{(m)}} = \frac{1}{n}\sum_{i=1}^n \delta_{\mu_i^{(m)}}$, respectively.

In the semi-supervised learning setting, we are given the feature vectors $\Lambda_n^{(m)} = \{\mu^{(m)}_i\}_{i=1}^n \subset \cP{(\Omega)}$ and $N < n$ labels $\{y_i\}_{i=1}^N$ with the objective of finding missing labels for the unlabeled feature vectors $\{\mu_i^{(m)}\}_{i=N+1}^n$.
Formally, we seek a function $f_n: \Lambda^{(m)}_n \to \bbR$ that satisfies $f_n(\mu^{(m)}_i)=y_i$ for all $i\leq N$ and is smooth with respect to
the geometry induced by the graph 
$(\Lambda^{(m)}_n,\mathbf{W}^{(m)}_n)$, whose nodes are $\Lambda^{(m)}_n=\{\mu^{(m)}_i\}_{i=1}^n$ and edge weights are $\mathbf{W}^{(m)}_n=({W}^{(m)}_{ij})_{i,j=1}^n$. 
The edge weights are defined using a decreasing function $\eta : [0, \infty) \to [0, \infty)$ (satisfying Assumptions \ref{ass:weight_one}-\ref{ass:weight_four}). For fixed scale $\varepsilon > 0$ we set the weights to be,
\begin{align}\label{eqn:weights}
W^{(m)}_{ij}=\frac{1}{\varepsilon^d}\eta\bigg(\frac{\dWp{2}(\mu^{(m)}_i,\mu^{(m)}_j)}{\varepsilon}\bigg)   
\end{align}
where $\dWp{2}(\mu^{(m)}_i,\mu^{(m)}_j)$ is the $2$-Wasserstein metric between the empirical probability measures $\mu^{(m)}_i, \mu^{(m)}_j \in \Lambda^{(m)}_n$. We only retain those edges with weights $W^{(m)}_{ij} > 0$. 
This is known as the random geometric graph model~\cite{Penrose}. 
Approximately, the parameter $\varepsilon$ determines the range within which two nodes are connected. Here, the nodes represent empirical probability measures. It is chosen to satisfy a lower bound for connectivity.
We will take $\eps\to 0$ as $n\to \infty$ motivated by: (a) the graph's geometry becomes increasingly localised making it easier to detect non-linear class boundaries, and  
(b) the sparsity of the weight matrix $W^{(m)}_{ij}$ affects the algorithm's scalability, with typically smaller computation time required for fewer edges (equivalently, sparser $\mathbf{W}^{(m)}_n$).

In this context, we consider Laplace Learning \cite{Zhu_2003}, a variational machine learning method, defined by minimising the discrete energy $\mathcal{E}_{\varepsilon,m,n}:\Lp{p}(\bbP_{\Lambda_n^{(m)}}) \to [0,\infty)$ with label constraints where,
\begin{align}\label{eqn:discrete_energy_pertubed_intro}
\mathcal{E}_{\varepsilon,m,n}(f_n) =  
\frac{1}{n^2\varepsilon^p}\sum_{i,j=1}^n {W}^{(m)}_{ij}\la {f}_n(\mu^{(m)}_i) - {f}_n(\mu^{(m)}_j)\ra^p. 
\end{align}
This energy functional depends on the parameter (connectivity radius) $\varepsilon$, the number of data points per sample $m$, and the number of samples $n$. The minimization is over $f_n:\Lambda_n^{(m)}\to \bbR$ and subject to the constraints $f_n(\mu_i^{(m)}) = y_i$ for all $i=1,\dots, N$ (agreement with the known labels). That is, our motivation is the variational problem,
\begin{equation*}
\mbox{minimise }
\mathcal{E}_{\varepsilon,m,n}(f_n) \mbox{ over } f_n:\Lambda^{(m)}_n \to \bbR \mbox{ subject to $f_n(\mu^{(m)}_i)=y_i$ for all $i \leq N$}.    
\end{equation*} 

The limiting problem is described by a continuum functional  
which acts on real valued functions:
\begin{align} \label{eqn:continuum_energy_functional_intro}
\begin{aligned}
\mathcal{E}_\infty(f) &:=
\left\{\begin{array}{ll}
\int_{\bbR^d}\int_{\mathcal{S}}|\nabla (f\circ E)(\theta) \cdot h|^p \eta(\|B_\theta(h)\|_{\Lp{2}(E(\theta))})\rho^2_\mathcal{S}(\theta) \, \dd \theta \, \dd h &\quad\text{if } f \in \Wkp{1}{p}(\Lambda,\dWp{2},\bbP_\Lambda)  \\
+\infty &\quad\text{else.}
\end{array}\right.\\  
\end{aligned}
\end{align}

Here, $B_\theta: \bbR^d \to \Ck{1}(\bbR^k;\bbR^k)$ is a linear map such that in a formal sense it maps tangent vectors on $\cS$ to tangent vectors in the Wasserstein space (more precisely, see Assumption \ref{ass:B_map}), and
$\rho_\mathcal{S}$ is the density of the probability measure $\bbP_\cS$ on $\mathcal{S}$. The space $\Wkp{1}{p}(\Lambda,\dWp{2},\bbP_\Lambda)$, is the Sobolev space defined over $\Lambda$ which is equivalent to $\Wkp{1}{p}(\mathcal{S})$, the Sobolev space defined over $\mathcal{S}$ (see Subsection \ref{section:Sobolev_Spaces} for the definition and Proposition \ref{prop:Sobolev_Space_1} for a proof that the two Sobolev spaces are equivalent). The continuum functional can equivalently be written as 
$\mathcal{E}_\infty(f) = \frac{1}{p} \langle f, \Delta_\Lambda f \rangle_{\Lp{2}(\bbP_\Lambda)}$, where $\Delta_\Lambda$ denotes the Laplace-Beltrami operator, with the explicit formulation provided in Subsection~\ref{section:graph_laplacians}.

Our objective is to investigate the asymptotic consistency of Laplace Learning (in the variational sense) from discrete to continuum settings as the limits $n \to \infty$, $m \to \infty$, and $\varepsilon \to 0$, thereby extending the existing methodology from Euclidean spaces to infinite-dimensional settings (see Figure \ref{fig:Main_Idea}). The strategy of the proof of our main result, variational convergence between $\mathcal{E}_{\varepsilon,m,n}$ defined in \eqref{eqn:discrete_energy_pertubed_intro} to $\mathcal{E}_\infty$ defined in \eqref{eqn:continuum_energy_functional_intro}, requires two intermediary graph $p$-Dirichlet energies. The first is defined using the `exact' feature vectors $\{\mu_i\}_{i=1}^n\in\Lambda_n \subset \Lambda$. In particular,
$\mathcal{E}_{\varepsilon,n}:\Lp{p}(\bbP_{\Lambda_n}) \to [0,\infty)$ is defined by 
\begin{align}\label{eqn:discrete_energy_intro}
\mathcal{E}_{\varepsilon,n}(f_n) =
\frac{1}{n^2\varepsilon^p}\sum_{i,j=1}^n W_{ij}\la f_n(\mu_i)-f_n(\mu_j) \ra^p
\end{align}
where $W_{ij}=\frac{1}{\varepsilon^d}\eta\bigg(\frac{\dWp{2}(\mu_i,\mu_j)}{\varepsilon}\bigg)$ in the same spirit as \eqref{eqn:weights}.
And the second is defined on the parameter space using $\{\theta_i\}_{i=1}^n \in \cS_n \subset \cS$. In particular, $\hat{\mathcal{E}}_{\varepsilon_,n}:\Lp{p}(\bbP_{\mathcal{S}_n})\to [0,\infty)$ is defined by
\begin{align}\label{eqn:discrete_energy_parameter_intro}
\hat{\mathcal{E}}_{\varepsilon,n}(\hat{f}_n) =
\frac{1}{n^2\varepsilon^p}\sum_{i,j=1}^n \hat{W}_{ij}\la \hat{f}_n(\theta_i)-\hat{f}_n(\theta_j) \ra^p
\end{align}
where $\hat{W}_{ij}=\frac{1}{\varepsilon^d}\hat{\eta}_{\theta_i}\bigg(\frac{\theta_j-\theta_i}{\varepsilon}\bigg)$, for a particular choice of kernel $\hat{\eta}_\theta: \bbR^d \to [0,\infty)$ for $\theta\in\cS$.

The latter graph $p$-Dirichlet energy has a continuum limit $\hat{\cE}_\infty:\Lp{p}(\bbP_\cS) \to \bbR$ defined by
\begin{align}\label{eqn:continuum_energy_functional_discrete_intro}
\begin{aligned}
\hat{\mathcal{E}}_\infty(\hat{f}) &:=
\left\{\begin{array}{ll}
\int_{\bbR^d}\int_{\mathcal{S}}|\nabla (\hat{f})(\theta) \cdot h|^p \eta(\|B_\theta(h)\|_{\Lp{2}(E(\theta))})\rho^2_\mathcal{S}(\theta) \, \dd \theta \, \dd h &\quad\text{if } \hat{f} \in \Wkp{1}{p}(\mathcal{S})\,  \\
+\infty &\quad\text{else.}
\end{array}\right.
\end{aligned}
\end{align}
We observe that $\cE_\infty(f) = \hat{\cE}_\infty(f\circ E)$.

\subsection{Contributions}
To the best of our knowledge, this paper is the first to present an analysis of Laplace Learning on the Wasserstein submanifold; the results are discussed in greater detail in Section \ref{section:Assumptions_Main_Results}. The key contributions are as follows:
\begin{enumerate}
    \item \textbf{Variational convergence from discrete to continuum.} We rigorously prove the asymptotic consistency of the graph $p$-Dirichlet energy for Laplace Learning, showing convergence from~\eqref{eqn:discrete_energy_pertubed_intro} to~\eqref{eqn:continuum_energy_functional_intro} in Theorem~\ref{thm:gamma_convergence}, and establish compactness results in Theorem~\ref{thm:compactness}. Together, these results provide a theoretical guarantee (via $\Gamma$-convergence) of the existence of a well-defined continuum limit of Laplace Learning in the finite dimensional Wasserstein submanifold.
    \item \textbf{Characterization of the Laplace–Beltrami operator.} In Proposition~\ref{prop:LBO}, we present an explicit form of the Laplace–Beltrami operator on Wasserstein submanifolds. This establishes a precise correspondence between discrete graph operators and continuum differential operators, enabling their characterisation in finite dimensional Wasserstein submanifold.   
    \item \textbf{Numerical validation.} In Section \ref{section:Numerical_Experiments}, we provide numerical experiments demonstrating the robust performance of Laplace Learning in high-dimensional settings on benchmark datasets. These experiments illustrate the practical relevance and effectiveness of our theoretical results.
\end{enumerate}

\subsection{Related Works}
\paragraph{Consistency Results:} A key step in establishing asymptotic consistency is to characterize the suitable notion of convergence from~\eqref{eqn:discrete_energy_pertubed_intro} to~\eqref{eqn:continuum_energy_functional_intro} in this setting.
One relevant framework is that of pointwise consistency of energy functionals (see ~\cite{Mikhail_2007,Ronald_2006,Evarist_2006,Matthias_2006,Matthias_2005,Amir_2006,Ting_2011} for more details), where a fixed (smooth) function $f : \Lambda \to \bbR$ is considered, and the comparison is made between \(\mathcal{E}_{\varepsilon,n}(f\lfloor_{\Lambda_n})\) and \(\mathcal{E}_\infty(f)\). We note that in our setting it is unclear how to define $\mathcal{E}_{\eps,m,n}(f\lfloor_{\Lambda_n^{(m)}})$ since $f$ is not defined at $\mu_i^{(m)}$ (as $\Lambda_n^{(m)} \not\subseteq \Lambda$). Spectral convergence of the graph-Laplacian is sufficient (with a compactness result) to imply the convergence of minimisers when $p=2$, but this does not generalise to $p\neq 2$ \cite{Mikhail_2007,Nicolas_2020,Jeff_2022}. Bayesian consistency results such as in \cite{Nicolas_bayesian_2020} suggest that if the graph parameters are suitably scaled the graph-posteriors converge to a continuum limit as the size of the unlabeled data set increases. Variational consistency, implying the convergence of minimisers, has been established in the finite dimensional setting ~\cite{Nicolas_pointcloud}, for discrete-to-continuum convergence of Laplace Learning in Euclidean spaces \cite{Matt_p_Laplacian}. Furthermore, using PDE-based techniques, the authors in \cite{Flores_2019} explored large data limits for two related problems: the first being Lipschitz learning \cite{Rasmus_2015}, corresponding to the case $p=\infty$, and the second being the game-theoretic $p$-Dirichlet energy \cite{Jeff_2018}. 

Since $\cE_{\eps,m,n}$ and $\cE_{\infty}$ do not have a common domain, we adapt the $\TLp{p}$ topology, introduced in \cite{Nicolas_pointcloud} from the Euclidean settings to general metric spaces. The underlying intuition is to use an optimal transport map $T_n:\Lambda\to \Lambda_n^{(m)}$ such that ${T_n}_\# \mathbb{P}_{\Lambda} = \mathbb{P}_{\Lambda^{(m)}_n}$, maps regions of $\Lambda$ to points in $\Lambda_n^{(m)}$ (effectively creating a partition of $\Lambda$). For a discrete function $f_n$ on $\Lambda_n^{(m)}$ we can define an extended function $\tilde{f}_n = f_n \circ T_n$ on $\Lambda$.
To assess the difference between $\tilde{f}_n$ and $f$, we can apply a  $\Lp{p}$ norm to $\tilde{f}_n - f$, leading to the $\TL_{\dWp{2}}$ distance. This adapted metric and its topology are reviewed in detail in Section \ref{section:TLp}. 

\paragraph{Manifold Learning Results:} Many advanced manifold learning techniques leverage graphs constructed from data to facilitate low-dimensional embeddings. For instance, Isomap \cite{Tenenbaum} uses shortest paths on graphs to approximate geodesics on the manifold, while Diffusion Maps \cite{Ronald_2006} employ the graph Laplacian and diffusion coordinates for embedding. As demonstrated in \cite{Bernstein}, increasing the sampling density in Isomap causes the graph shortest paths in an $\varepsilon$-neighbourhood graph to converge to the manifold geodesics, thereby approximating the ideal embedding. Coifman and Lafon~\cite{Ronald_2006} showed the convergence of the graph Laplacian to the Laplace–Beltrami operator with increased sampling. The Wasserstein space analogue, Wassmap, was explored by \cite{Hamm_2}, though consistency results akin to those for Isomap or Diffusion Maps were not previously established. It is beyond the scope of the paper to prove convergence of the Laplace-Beltrami operator corresponding to~\eqref{eqn:continuum_energy_functional_intro}, but given we have established convergence at the level of the energy, it is reasonable to conjecture (with a possibly more restrictive lower bound on $\eps$) convergence of the graph Laplacian to a continuum Laplace-Beltrami operator in this setting.
\vspace{\baselineskip}

The structure of the paper is outlined as follows. 
In Section \ref{section:Assumptions_Main_Results}, we state the relevant assumptions required and list the main results. In Section \ref{section:Background}, we provide a brief overview of the related concepts relevant for our analysis. In Section \ref{section:Proofs} we present the proofs in two steps: first establishing compactness, and then proving the $\Gamma$-convergence of the discrete-to-continuum graph $p$-Dirichlet functional of Laplace Learning. We also provide the derivation with the explicit form of the graph Laplacian in the discrete setting and the Laplace-Beltrami operator in the continuum setting. Finally, in Section \ref{section:Numerical_Experiments} we provide numerical experiments to illustrate the performance of classification using the Laplace Learning algorithm for different datasets.

\begin{figure}[ht]
\centering
\includegraphics[width=1.\linewidth]{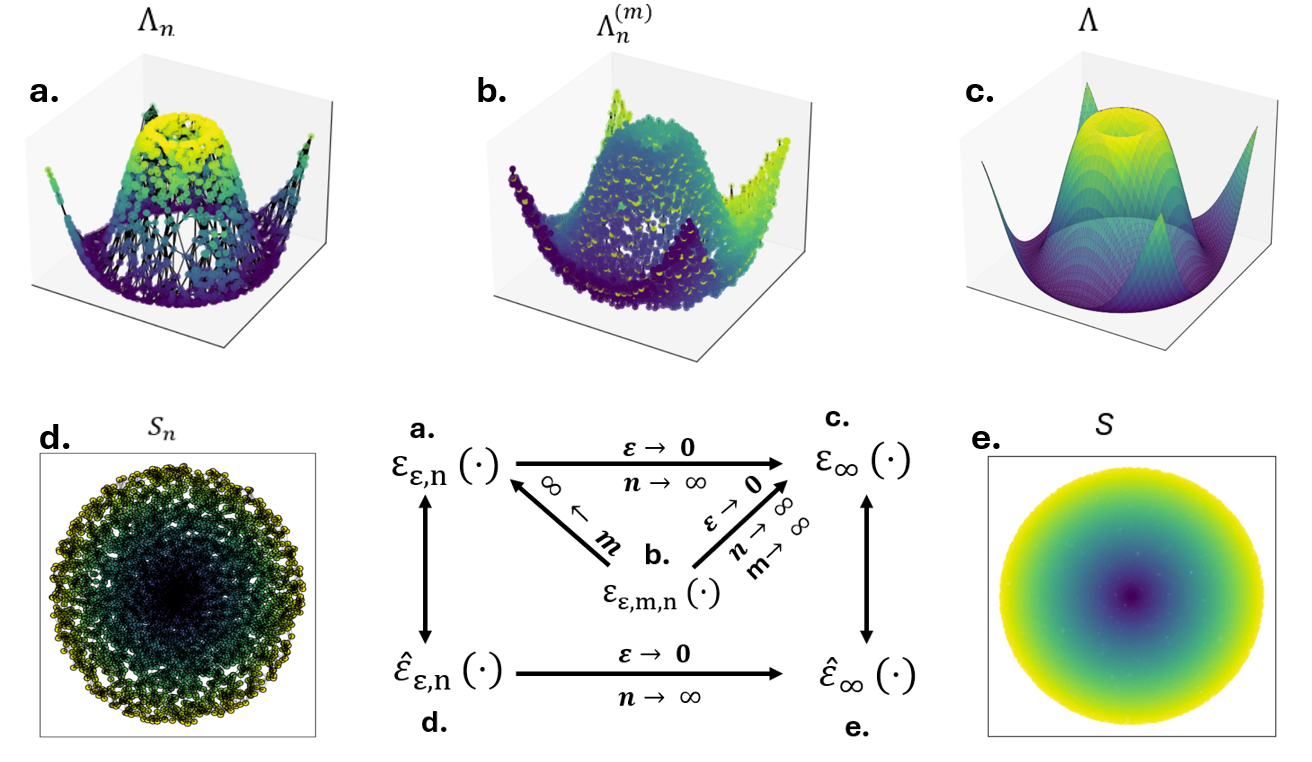}
\caption{A schematic diagram outlining the analytical framework of the study. The primary objective is to demonstrate the $\Gamma$-convergence of $\mathcal{E}_{\varepsilon,m,n}(\cdot)$ to $\mathcal{E}_\infty(\cdot)$ as $n \to \infty$, $m \to \infty$, and $\varepsilon \to 0$, with the latter two limits approaching at a suitable rate. Whilst our main result is to establish variational convergence from (b.) to (c.) our proof follows the sequence (b. $\rightarrow$ a. $\rightarrow$ d. $\rightarrow$ e. $\rightarrow$ c.). For clarity, the feature vectors depicted in the top row of the figure represent probability measures.
}
\label{fig:Main_Idea}
\end{figure}

\section{Assumptions and Main Results} \label{section:Assumptions_Main_Results}

In this section we list the set of assumptions for use in the sequel.
\begin{assumption}
Spaces.
\begin{enumerate}[label=$\bm{\mathrm{A.\arabic*}}$]
\item \label{ass:domain_ll_one}
The base space $\Omega \subset \bbR^k$ is an open, connected, and bounded domain with $k \geq 1$. We also assume that $\Omega$ has smooth boundary. 
\item \label{ass:domain_ll_two}
The Wasserstein submanifold is $\Lambda \subset \Probac(\Omega)$ where we endow $\Probac(\Omega)$, the space of probability measures that are absolutely continuous with respect to the Lebesgue measure, with the Wasserstein metric $\dWp{2}(\mu,\nu)$ defined later in  \eqref{eqn:wasserstein_metric} (see Subsection \ref{Section:Optimal_Transport} for more details).
\item \label{ass:domain_ll_three}  
The parameter space $\mathcal{S} \subset \bbR^d$ is an open, connected, and bounded domain with $d \geq 1$. We also assume that $\mathcal{S}$ has a smooth boundary endowed with the Euclidean metric $d_\mathcal{S}(\theta_i,\theta_j):=|\theta_j-\theta_i|$.
\end{enumerate}
\end{assumption}

\begin{assumption}
Parameterisation.   
\begin{enumerate}[label=$\bm{\mathrm{B.\arabic*}}$]
\item \label{ass:E_map} $E: \mathcal{S} \to \Lambda$ is a bi-Lipschitz  map satisfying $E(\mathcal{S})=\Lambda \subset \Probac(\Omega)$.
\item \label{ass:B_map} For any $\theta\in\cS$, let $s_\theta>0$ be the radius such that $B(\theta,s_\theta)\subset \cS$.
For $\theta\in\cS$ and $h\in B(0,s_\theta)$, let $\zeta:[0,1]\ni \Delta \mapsto \theta+\Delta h\in \cS$ and we assume there exists a linear map $B_\theta:B(0,s_\theta) \to \Ck{1}(\bbR^k;\bbR^k)$ and a constant $C\in (0,\infty)$ (independent of $\theta$) such that for all $h\in B(0,s_\theta)$ and $x\in\bbR^k$:
\begin{enumerate}
    \item\label{item:B.2_a} $\|B_\theta(h)\|_{\Lp{2}(E(\theta))} \in [1/C,C] \cdot \|h\|$;
    \item\label{item:B.2_b} $\|B_\theta(h)\|_{\Ck{1}(\bbR^k)} \leq C \cdot \|h\|$;
    \item\label{item:B.2_c} $|\partial_\Delta[B_{\zeta(\Delta)}(h)](x)| \leq C \cdot \|h\|^2$; and
    \item\label{item:B.2_d} The mapping $\theta \mapsto \| B_\theta(h) \|_{\Lp{2}(E(\theta))}$ is uniformly continuous for all $h \in B(0,s_\theta)$; that is, for every $\delta > 0$, there exists $\psi>0$ such that for all $\theta,\theta' \in \mathcal{S}$ with $|\theta-\theta'|<\psi$ we have
    \[\sup_{h \in B(0,s_\theta)} \left| \| B_\theta(h) \|_{\Lp{2}(E(\theta))} - \| B_{\theta'}(h) \|_{\Lp{2}(E(\theta'))} \right| < \delta.\]
    \item\label{item:B.2_e} The curve $\mu_{\zeta(\Delta)}=E(\zeta(\Delta))$ and the velocity field $B_{\zeta(\Delta)}(h)$ satisfy the continuity equation
    \begin{equation*}
    \partial_\Delta \mu_{\zeta(\Delta)}+\nabla( B_{\zeta(\Delta)}(h) \cdot \mu_{\zeta(\Delta)})=0
    \end{equation*}
    in a distributional sense.
\end{enumerate}
\end{enumerate}
\end{assumption}

We recall that the assumptions listed~\ref{item:B.2_a}-\ref{item:B.2_c} and~\ref{item:B.2_e} on $B_\theta$ are motivated by \cite[Assumption 2.1 (iii)]{Hamm} and \cite[Proposition 1.1, Proposition 2.12]{Hamm}.

\begin{assumption}
Probability Measures.
\begin{enumerate}[label=$\bm{\mathrm{C.\arabic*}}$]
\item \label{ass:prob_measures_S}
$\bbP_\mathcal{S} \in \mathcal{P}(\mathcal{S})$ is the probability measure on $\mathcal{S}$ with a strictly positive Lipschitz continuous Lebesgue density $\rho_\mathcal{S}: \cS \to (0,\infty)$. Moreover, there exists constants $0 < c < C < \infty$ 
such that
\begin{align}
0 < c < \inf_{\theta \in \mathcal{S}} \rho_\mathcal{S}(\theta) \leq \sup_{\theta \in \mathcal{S}} \rho_\mathcal{S}(\theta) < C < + \infty.    
\end{align}
\item \label{ass:prob_measures_Lambda} $\bbP_\Lambda \in \mathcal{P}(\Lambda)$ is the probability measure on $\Lambda \subset \Probac(\Omega)$.
We assume $\bbP_\Lambda = E_{\#}\bbP_\cS$. For all $\mu \in \mathrm{spt}(\mathbb{P}_\Lambda)$, $\mu$ admits a Lipschitz continuous density $\rho_\mu : \Omega \to (0,\infty)$.  
Furthermore, the family $\{\rho_\mu\}_{\mu \in \mathrm{spt}(\mathbb{P}_\Lambda)}$ is uniformly bounded above and below; i.e., there exist constants $0 < c < C < \infty$ independent of $\mu$ such that
\[
0 < c< \inf_{x \in \Omega} \rho_\mu(x) \leq \sup_{x \in \Omega} \rho_\mu(x) < C< + \infty. 
\]
\end{enumerate}
\end{assumption}

\begin{assumption}
Data.    
\begin{enumerate}[label=$\bm{\mathrm{D.\arabic*}}$]
    \item \label{ass:S_n}
 $\mathcal{S}_n=\{\theta_i\}_{i=1}^n \subset \mathcal{S}$ are the iid discrete samples on $\mathcal{S}$ drawn from the measure $\bbP_\mathcal{S}$. We denote the empirical measures associated to the discrete samples by $\bbP_{\mathcal{S}_n}=\frac{1}{n} \sum_{i=1}^n \delta_{\theta_i} \in \mathcal{P}(\mathcal{S})$.
 \item \label{ass:Lambda_n} $\Lambda_n=\{\mu_i\}_{i=1}^n \subset \Lambda$ are the iid discrete samples on $\Lambda \subset \Probac(\Omega)$ drawn from the measure $\bbP_\Lambda$, or equivalently $\mu_i=E(\theta_i)$ with $\theta_i$ satisfying \ref{ass:S_n}.
 We denote the empirical measures associated to the discrete samples by $\bbP_{\Lambda_n}=\frac{1}{n}\sum_{i=1}^n \delta_{\mu_i} \in \mathcal{P}(\Lambda)$.
 \item \label{ass:Lambda_m_n} $\Lambda^{(m)}_n=\{\mu^{(m)}_i\}_{i=1}^n \subset \mathcal{P}(\Omega)$ are the empirical approximation of $\Lambda_n$. 
 That is
$\mu^{(m)}_i=\frac{1}{m}\sum_{j=1}^m \delta_{x^{(i)}_j}$ where $x^{(i)}_j \iid \mu_i$ for $i\in\{1,\dots,n\}$ and $j\in\{1,\dots, m\}$.
We denote the empirical measures associated to the discrete samples by $\bbP_{\Lambda^{(m)}_n}=\frac{1}{n}\sum_{i=1}^n \delta_{\mu^{(m)}_i} \in \mathcal{P}(\cP(\Omega))$.
\end{enumerate}
\end{assumption}

\begin{assumption}
Weight function or kernel.    
\begin{enumerate}[label=$\bm{\mathrm{E.\arabic*}}$]
\item \label{ass:weight_one}
The function ${\eta}: [0, +\infty) \to [0, +\infty)$ is non-increasing. 
\item \label{ass:weight_two} The function $\eta$ is continuous.  
\item \label{ass:weight_three} The function $\eta$ is positive at zero i.e., ${\eta}(0) > 0$.
\item \label{ass:weight_four} The function $\eta$ has compact support in $[0,\tilde{R}]$. Without loss of generality we assume that $\tilde{R}\leq s_\theta$ where $s_\theta$ was defined in Assumption~\ref{ass:B_map}.
\end{enumerate}
\end{assumption}

The assumption that $\tilde{R}\leq s_\theta$ is without loss of generality since we can rescale the support of $\eta$ and absorb the constant into the scaling parameter $\eps$. In particular, say $\eta$ has compact support in $[0,R^\prime]$, then define $\eps^\prime = \frac{\tilde{R}\eps}{R^\prime}$ and $\eta^\prime = \eta\lp\frac{R^\prime}{\tilde{R}}\cdot\rp$. Then $\eta(t/\eps) = \eta^\prime(t/\eps^\prime)$ so $\eta^\prime$ has compact support in $[0,\tilde{R}]$. Our theory applies to the energies defined with weights given by $\eta^\prime$ and scaling $\eps^\prime$. The minimisers of the energies with $\eta,\eps$ and $\eta^\prime,\eps^\prime$ coincide (as they differ by only a multiplicative constant), hence our results hold for the energies defined with $\eta^\prime$ and $\eps^\prime$.

\begin{assumption}
Length scale.
\begin{enumerate}[label=$\bm{\mathrm{F.\arabic*}}$]
\item \label{ass:eps_scale} Let $\varepsilon=\varepsilon_{n}\to 0$ satisfy $\eps_n\gg q_d(n)$, where $q_d(n)$ is defined in \eqref{eqn:rates_qd} (see Theorem~\ref{thm:rates_convergence}).

\item \label{ass:eps_scale_2} Let $\varepsilon=\varepsilon_{n}\to 0$ satisfy $\eps_n\gg \tilde{q}_k(m_n)$, where $\tilde{q}_k(m_n)$ is defined in \eqref{eqn:rates_qd} (see Theorem~\ref{thm:rates_convergence}).
\end{enumerate}
\end{assumption}

\begin{remark}
If~\ref{ass:eps_scale} and~\ref{ass:eps_scale_2} hold, then with probability one there exists an integer $N < \infty$ such that for all $n \geq N$, the graph $(\Lambda_n^{(m)}, \bfW_n^{(m)})$ is connected. This follows from the observation that when $k \geq 5$, the condition $\dWp{2}(\mu_i^{(m)}, \mu_j^{(m)}) \lesssim \eps_n$ implies, with high probability,
\[
\dWp{2}(\mu_i, \mu_j) \leq \dWp{2}(\mu_i, \mu_i^{(m)}) + \dWp{2}(\mu_i^{(m)}, \mu_j^{(m)}) + \dWp{2}(\mu_j^{(m)}, \mu_j) \lesssim m^{-\frac{1}{k}} + \eps_n \lesssim \eps_n,
\]
where the estimate $\dWp{2}(\mu_i, \mu_i^{(m)}) \lesssim m^{-1/k}$ follows from Theorem~\ref{thm:rates_convergence}. 
Furthermore, by~\cite[Theorem 2.8 and Proposition 2.10]{Hamm}, the Wasserstein distance on $E(\cS)$ is equivalent to the Euclidean distance on $\cS$. Hence, the condition $\dWp{2}(\mu_i^{(m)}, \mu_j^{(m)}) \lesssim \eps_n$ implies that $|\theta_i - \theta_j| \lesssim \eps_n$. Finally, by~\cite{Penrose} and Assumption~\ref{ass:eps_scale}, the geometric random graph having nodes $\{\theta_i\}_{i=1}^n$, with edges between points at distance at most $O(\eps_n)$, is connected with high probability. It follows that the graph $(\Lambda_n^{(m)}, \bfW_n^{(m)})$ is also connected.
\end{remark}

Following the assumptions stated above, the main findings of the paper are as follows:
\begin{itemize}
\item \textbf{Compactness.}  
 Let $(f_n)_{n \in \mathbb{N}}$ be a sequence of functions satisfying
\[
\sup_{n \in \mathbb{N}} \|f_n\|_{\Lp{p}(\bbP_{\Lambda^{(m)}_n})} < +\infty
\quad \text{and} \quad
\sup_{n \in \mathbb{N}} \mathcal{E}_{\varepsilon_n, m_n, n}(f_n; \eta) < +\infty.
\]
Then, with probability one, there exists a subsequence $(f_{n_k})_{k \in \mathbb{N}}$ and a function $f \in \Lp{p}(\bbP_\Lambda)$ such that
\[
f_{n_k} \to f \quad \text{in} \quad \TL_{\dWp{2}}(\Probac(\Omega)) \quad \text{as } k \to \infty.
\]
For more details, we refer to Theorem~\ref{thm:compactness}.

\item \textbf{$\Gamma$-Convergence.}  
The sequence of discrete energy functional $\mathcal{E}_{\varepsilon_n, m_n, n}$ $\Gamma$-converges, with probability one, to the continuum functional $\mathcal{E}_\infty$ as $n \to \infty$, i.e.,
\[
    \mathcal{E}_{\varepsilon_n, m_n, n} \xrightarrow{\Gamma} \mathcal{E}_\infty.
\]
For more details, we refer to Theorem~\ref{thm:gamma_convergence}. 

\item \textbf{Laplace–Beltrami Operator.}  
For all $f \in \Wkp{1}{p}(\Lambda, \dWp{2}, \bbP_\Lambda)$, the continuum energy functional $\mathcal{E}_\infty$ admits the representation
    \[
    \mathcal{E}_\infty(f) = \frac{1}{p} \langle \Delta_\Lambda f, f \rangle_{\Lp{2}(\bbP_\Lambda)},
    \]
    where $\Delta_\Lambda$ denotes the Laplace–Beltrami operator associated with the limiting Wasserstein submanifold $(\Lambda, \dWp{2}, \bbP_\Lambda)$. An explicit form of $\Delta_\Lambda$ is given in Proposition~\ref{prop:LBO}.
\end{itemize}

\section{Background}\label{section:Background}
In this section, we review the preliminary concepts used in the main results. We first present optimal transport theory, a framework for efficiently transporting mass between distributions, and standard results on convergence rates for empirical measures in the $p$-Wasserstein distance. Next, we introduce the $\TL_d$ space, which defines the convergence of functions from discrete to continuum spaces in our setting, followed by a characterization of Wasserstein submanifolds in Sobolev spaces. Finally, we briefly review $\Gamma$-convergence, providing a notion of variational convergence.

\subsection{Optimal Transport}\label{Section:Optimal_Transport}

In this section, we recall the notion of optimal transportation between measures and its associated metric. For general background on this topic, see \cite{Cedric,Villani_Old,Rachev,Garling,Santambrogio}.

Let $\Omega \subset \bbR^k$ be an open, connected and bounded domain with smooth boundary. We denote the space of probability measures on $\Omega$  with bounded $p^\textrm{{th}}$ moment, as follows
\begin{equation*}\label{eqn:wasserstein_space}
\mathcal{P}_p(\Omega):= \Bigl\{\mu \in \cP{(\Omega)} : \int_{\Omega} \lvert {x} \lvert^p \,\dd\mu({x}) < +\infty \Bigl\}.
\end{equation*}

Let $\mu \in \cP(\Omega)$ and $T: \Omega \to \Omega$ be a measurable map, the pushforward of $\mu$ by $T$, denoted as $T_\#\mu$ is the measure 
\begin{equation*}
T_{\#}\mu(A)=\mu(T^{-1}(A)).    
\end{equation*} 
for all measurable set $A \subseteq \Omega$.

We briefly introduce some special cases of the optimal transport formulations where the cost function $c: \Omega \times \Omega \to (0,\infty]$ is the $p$-th power of the Euclidean distance i.e., $c(x,y)=|x-y|^p$.

Given the probability measures $\mu \in \cP(\Omega)$ and $\nu \in \mathcal{P}(\Omega)$ the $\textit{Monge formulation}$ of optimal transport is the variational problem
\begin{align}\label{eqn:Monge}
\mathbb{M}(\mu,\nu):=\inf_{T:T_{\#}\mu=\nu}{\int_{\Omega} |x - T(x) |^p \, \dd\mu({x})}.
\end{align}
A map $T$ satisfying $T_{\#}\mu = \nu$ is called a transport map, and the minimizer of the optimization problem in~\eqref{eqn:Monge} is referred to as the optimal transport map. The problem in \eqref{eqn:Monge} is often difficult to solve due to its non-convexity and the nonlinearity in $T$.

We define the set $\Pi(\mu,\nu)$ of couplings between measures $\mu$ and $\nu$ to be the set of probability measures on the product space $\mathcal{P}(\Omega \times \Omega)$ whose first marginal is $\mu$ and the second marginal is $\nu$. We call $\pi\in\Pi(\mu,\nu)$ a transport plan.
For any transport map $T$: $\Omega$ $\rightarrow$ $\Omega$ with $T_{\#}\mu=\nu$ the probability measure,
\begin{equation*}
\pi=(\mathrm{Id} \times T)_{\#}\mu,
\end{equation*}
where $\Id$ denotes the identity map satisfies $\pi\in\Pi(\mu,\nu)$. 

The $\textit{Kantorovich formulation}$ of optimal transport is the variational problem 
\begin{align}\label{eqn:Kantorovich}
\mathbb{K}(\mu,\nu):= \inf_{\pi \in \Pi(\mu,\nu)}{\int_{\Omega \times \Omega} |x-y|^p\,\dd\pi({x},{y})}.
\end{align}
The minimiser of \eqref{eqn:Kantorovich} is the optimal transport plan. When the source measure $\mu$ is absolutely continuous, the \emph{Kantorovich} and \emph{Monge} formulations are equivalent, as the optimal plan is concentrated on the graph of a measurable transport map. In contrast, for discrete measures, mass splitting may occur, so the Monge problem may fail to admit a solution, whereas the Kantorovich problem remains well-posed.

When $\Omega$ is bounded, we have $\mathcal{P}_p(\Omega) = \cP(\Omega)$. This allows us to define the $p$-Wasserstein distance \cite{Cedric}, which represents the minimum transportation cost (in terms of the $p^\textrm{th}$ power of the Euclidean distance) between $\mu \in \cP(\Omega)$ and $\nu \in \mathcal{P}(\Omega)$ for any $p \in [1, \infty]$. It is defined as follows:

\begin{align}\label{eqn:wasserstein_metric}
\begin{aligned}
\quad \dWp{p}(\mu,\nu) & =
\begin{cases}
    \inf_{\pi \in \Pi(\mu,\nu)} \bigg(\displaystyle\int_{\Omega \times \Omega} \lvert {x} - {y} \lvert^p \,\dd\pi({x},{y})\bigg)^{\frac{1}{p}} & \text{for $1\leq p < \infty$,} \\
    \inf_{\pi \in \Pi(\mu,\nu)} \pi $-$ \esssup_{(x,y)}|x-y| & \text{for $p=\infty$.}
\end{cases}\\
\end{aligned}
\end{align}

Recall that $\mu^{(m)}$ and $\nu^{(m)}$ are the empirical distributions supported on $m$ discrete points, i.e., $\mu^{(m)} = \frac{1}{m}\sum_{i=1}^m \delta_{x_i}$ and $\nu^{(m)} = \frac{1}{m}\sum_{j=1}^m \delta_{y_j}$ where $x_i\iid \mu$ and $y_j\iid \nu$. In the sequel we let 
$C, c > 0$  denotes a constant (large, small) that can arbitrarily vary from line to line.

\begin{theorem} \label{thm:rates_convergence} 
Let $\Omega \subset \mathbb{R}^k$ be an open, connected and bounded domain with Lipschitz boundary with $k \geq 1$. Let $\mu$ be a probability measure on $\Omega$ with density $\rho_\mu : \Omega \to (0, \infty)$ such that there exists $ C > c \geq 1$ for which,
\begin{align}
    c \leq \rho_\mu(x) \leq C.
\end{align}
Let $\alpha>2$, $p\in \{2,\infty\}$, and $S^{*}_m$ be an optimal transport map between $\mu$ and $\mu^{(m)}$ with respect to the $p$-Wasserstein distance. Define $q_k(m)$ and $\tilde{q}_k(m)$ by
\begin{align}\label{eqn:rates_qd}
q_{k}(m) & :=
\begin{cases}
\sqrt{\frac{\log \log m}{m}} & \text{for } k=1, \\
\frac{(\log m)^{3/4}}{m^{1/2}} & \text{for } k=2, \\
\lp\frac{\log m}{m}\rp^{1/k} & \text{for } k\geq3,
\end{cases}
\quad & \text{and} \quad \quad \quad
\tilde{q}_{k}(m) & :=
\begin{cases}
\sqrt{\frac{\log m}{m}} & \text{for } k=1 \\
\frac{\log(m)^{3/4}}{m^{1/2}} & \text{for } k=2 \\
\lp\frac{\log m}{m}\rp^{1/3} & \text{for } k=3 \\
\frac{(\log m)^{1/2}}{m^{1/4}} & \text{for } k=4 \\
m^{-1/k} & \text{for } k\geq 5.
\end{cases}
\end{align}
\begin{itemize}
\item If $p=+\infty$, then the $\infty$-transportation distance between the measure $\mu$
and the empirical measure $\mu^{(m)}$ scales with the rate $q_k(m)$, 
i.e.,
\begin{align}\label{eqn:rates_qd_wasserstein_distance}
\dWp{\infty}(\mu,\mu^{(m)}) = \|S^{*}_m - \Id\|_{\Lp{\infty}(\mu)} \leq \hat{C}q_k(m),
\end{align}
with probability at least $1-\mathrm{O}(m^{-\alpha/2})$
\item If $p=2$, then
the $2$-transportation distance between measure $\mu$
and the empirical measure $\mu^{(m)}$ scales with the rate $\tilde{q}_k(m)$, 
i.e.,
\begin{align}\label{eqn:rates_qd_tilde}
\dWp{2}(\mu,\mu^{(m)}) = \|S^{*}_m - \Id\|_{\Lp{2}(\mu)} \leq \hat{C}\tilde{q}_k(m),
\end{align}
with probability at least $1-\mathrm{O}(m^{-\alpha/2})$.  
\end{itemize}
The constant $\hat{C}$ depends only on $\alpha$, $c,C$ and $\Omega$. 
\end{theorem}

\begin{proof}
For $k\geq 2$ the $\dWp{\infty}$ rates follow directly from~\cite[Theorem 1.1]{Nicolas_rate_convergence} and for $k=1$ the rates follow from the law of the iterated logarithm.

For $k\in \{2,3\}$ we use $\dWp{2}(\mu,\mu^{(m)}) \leq \dWp{\infty}(\mu,\mu^{(m)})$ and the first part.
We consider $k=1$ and $k\geq 4$ separately.

For any $k\geq 1$ we have by~\cite[Proposition 20]{weed2019sharp}
\begin{equation} \label{eq:W2EBound}
\bbP\lp \dWp{2}^2(\mu,\mu^{(m)}) \geq \bbE\ls \dWp{2}^2(\mu,\mu^{(m)})\rs + t \rp \leq e^{-2mt^2}.
\end{equation}
We choose $t=\sqrt{\frac{\alpha\log m}{4m}}$ so
\[ \dWp{2}^2(\mu,\mu^{(m)}) \leq \bbE\ls \dWp{2}^2(\mu,\mu^{(m)})\rs + \sqrt{\frac{\alpha\log m}{4m}} \]
with probability at least $1-m^{-\alpha/2}$.

For $k=1$, we have by~\cite[Theorem 5.1]{bobkov2019one} that
\[ \bbE\ls \dWp{2}^2(\mu,\mu^{(m)})\rs \leq \frac{2}{m+1} \int_{-\infty}^\infty \frac{F(x)(1-F(x))}{\rho_\mu(x)} \, \dd x \]
where $\rho_\mu$ is the density of $\mu$ and $F$ is the cumulative distribution function of $\mu$. Using the (uniform) lower bound on $\rho_\mu$ and the upper bound $F(x)\leq 1$ we have
\[ \bbE\ls \dWp{2}^2(\mu,\mu^{(m)})\rs \leq \frac{2}{c(m+1)} \int_{-\infty}^\infty \lp 1-F(x)\rp \, \dd x = \frac{2}{c(m+1)} \bbE[x]. \]
Since we can also uniformly bound the expectation we have $\bbE\ls \dWp{2}^2(\mu,\mu^{(m)})\rs \leq \frac{C}{m}$.
Combining with~\eqref{eq:W2EBound} we can conclude the result for $k=1$.

For $k\geq 4$, we have by~\cite[Theorem 1]{fournier2015rate} and ~\cite[Proposition 5]{weed2019sharp}, that there exists $C>0$ independent of $\mu$ such that
\[ \bbE\ls \dWp{2}^2(\mu,\mu^{(m)})\rs \leq C \lb \begin{array}{ll}   m^{-1/2} \log m & \text{ if } k=4 \\ m^{-2/k} & \text{ if } k \geq 5. \end{array} \rd \]
Combining with~\eqref{eq:W2EBound} completes the proof.
\end{proof}
\begin{remark}
For $k \in \{2,3,4\}$, we expect that the convergence rates in Theorem \ref{thm:rates_convergence} for $\dWp{2}(\mu^{(m)},\mu)$ could be improved to $O(m^{-1/k})$, but to the best of the authors’ knowledge, the rates stated in the theorem remain the best currently known in general. Improved rates are known in certain special cases. For instance, the classical AKT theorem \cite{ajtai1984optimal} establishes this rate when $\mu$ is a uniform distribution on the cube. More generally, on the cube, similar rates (up to logarithmic factors) hold for densities that are bounded above and below \cite[Corollary 8]{manole2024plugin}. A general treatment over the torus is also available \cite[Theorem 1]{Divol2021Short}. For general domains, finer results are known at least in dimension 2, but these require additional regularity assumptions on $\mu$ \cite{ambrosio2022quadratic}. For a comprehensive discussion, see \cite{hundrieser2024empirical} and references therein. We conjecture that similar improved rates could hold for general domains with Lipschitz boundary by adapting the arguments in \cite{bobkov2021simple}, and we expect this approach to be suboptimal requiring a more nuanced treatment.
\end{remark}

\subsection{Metric Transportation \texorpdfstring{$\Lp{p}$}{Lp} Space (\texorpdfstring{$\TL_d(A)$}{TLpd(A)})}\label{section:TLp}

We introduce the notion of $\TL_{d}(A)$ spaces, which have been adapted from \cite{Nicolas_pointcloud} to our context, as a metric to define the convergence from the discrete to the continuum space for point clouds.  Recent work motivated by the $\TL(A)$ space, includes \cite{Fonseca} and \cite{PTLp} which introduced the space of equivalence classes $\mathrm{CL}^p(A)$ to define the space of piecewise smooth segmentations from the Ambrosio-Tortorelli approximation of the Mumford-Shah segmentation model and introduce a new family of metrics for comparing generic signals, benefiting from the robustness of partial transport distances and sliced-partial transport distances \cite{Sliced_PTLp}.

In this subsection, we consider an arbitrary metric space $(A,d)$ in order to define the $\TL_{d}(A)$ distance. Let $A$ be a non-empty set with $p \geq $ 1, we use $\mathcal{P}(A)$ to denote the set of Borel probability measures defined on $(A,d)$.

\begin{mydef}\label{def:TL0_space}
Let $(A,d)$ be a metric space.
We define $\TLp{p}(A)$ to be the set of functions $f$ and probability measures $\bbP$ on $A$ such that $f$ is $\bbP$ measurable, i.e.,
\begin{equation*}
\TLp{p}(A):=\big\{(f,\bbP): f: A \to \bbR \in  \Lp{p}(\bbP) \mbox{ where } \bbP\in\mathcal{P}(A)  \big\}.
\end{equation*}
We define the transportation distance $d_{\TL_{d}(A)}: \TLp{p} \times \TLp{p} \to [0,\infty)$ between the pairs $(f,\bbP) \in \TLp{p}(A)$ and $(g,\bbQ) \in \TLp{p}(A)$ as follows,
\begin{align}\label{eqn:TLp_distance}
d_{\TL_{d}}((f,\bbP),(g,\bbQ))=\inf_{\pi \in \Pi(\bbP,\bbQ)}  \bigg(\int_{A \times A} d(x,y)^p+|f(x)-g(y)|^p \, \textnormal{d}\pi(x,y)\bigg)^{\frac{1}{p}}.    
\end{align}
where $\pi \in \Pi(\bbP,\bbQ)$ is the set of transport plans between $\bbP$ and $\bbQ$. 
\end{mydef}

In short, the $\TL_{d}$ optimal transport plans, which minimise the variational problem defined in  \eqref{eqn:TLp_distance}, strike a middle ground between two objectives: aligning spatial locations closely, as indicated by minimising the distance, i.e., mimizing $\int_{A \times A} d(x,y)^p \, \dd\pi(x,y)$ and matching signal features, as represented by minimising $\int_{A \times A} |f(x)-g(y)|^p \, \dd\pi(x,y)$. 

The original $\TL$ distance proposed in \cite{Nicolas_pointcloud} 
is a special case of the above definition when $A$ is a Euclidean space and $d(x,y):=|x-y|$. We deduce the following proposition below which will help us later in  understanding the limiting behaviour of Laplace Learning. With an abuse of notation, we denote $\|f(x)-g(y)\|_{\Lp{p}(\pi(x,y))}= \bigg(\int_{A \times A}|f(x)-g(y)|^p\,\dd\pi(x,y)\bigg)^{\frac{1}{p}}$ for ease of convenience.

\begin{proposition}\label{prop:TLp_distance}
Let $p \geq 1$ and $(A,d)$ be a metric space. The following statements hold.
\begin{enumerate}
\item $d_{\TL_{d}(A)}$ is a metric. \label{eqn1}
\item \label{(a)} Let $\{(f_n,\bbP_n)\}^\infty_{n=1} \subset \TLp{p}(A)$ and $(f,\bbP) \in \TLp{p}(A)$, then $d_{\TL_{d}(A)}((f_n,\bbP_n),(f,\bbP)) \to 0$ if and only if there exists a sequence of optimal transport plans $\{{\pi}^*_n\}_{n=1}^\infty \in \Pi(\bbP,\bbP_n)$ with $\| d \|_{\Lp{p}(\pi^*_n)} \to 0$ and $\|f(x)-f_n(y) \|_{\Lp{p}(\pi^*_n(x,y))} \to 0$.
\item \label{(b)} Let $\{(f_n,\bbP_n)\}^\infty_{n=1} \subset \TLp{p}(A)$, $(f,\bbP) \in \TLp{p}(A)$ and $d_{\TL_{d}(A)}((f_n,\bbP_n),(f,\bbP)) \to 0$. Then, 
for any sequence of transport plans $\{\pi_n\}_{n=1}^\infty \in \Pi(\bbP,\bbP_n)$ with $\|d\|_{\Lp{p}(\pi_n)} \to 0$ it follows that $\|f(x)-f_n(y)\|_{\Lp{p}(\pi_n(x,y))} \to 0$.
\end{enumerate}
\end{proposition}

\begin{proof}
We generalize the proof in~\cite[Proposition 3.3]{Nicolas_pointcloud}.
\begin{enumerate}
\item
We remark that $d_{\TL_{d}(A)}$ is equal to the $p$-Wasserstein distance $d_{\mathrm{W}^p_{\tilde{d}}}$ between the graphs of functions where $\tilde{d}:(A\times \bbR) \times (A\times \bbR) \to [0,+\infty)$ is defined by $\tilde{d}((x,s),(y,t)) = \sqrt[p]{d(x,y)^p + |s-t|^p}$. In particular, 
let us consider $(f,\bbP),(g,\bbQ) \in \TLp{p}(A)$. For any $\pi \in \mathcal{P}(A \times A)$ we define $\tilde{\pi}=((\Id \times f)\times (\Id \times g))_{\#}\pi$. Then, $\pi \in \Pi(\bbP,\bbQ)$ implies $\tilde{\pi} \in \Pi(\tilde{\bbP},\tilde{\bbQ})$ where $\tilde{\bbP}=(\Id\times f)_{\#}\bbP$ and $\tilde{\bbQ}=(\Id\times g)_{\#} \bbQ$.
Hence, assuming $\pi^*\in\Pi(\bbP,\bbQ)$ is optimal, 
\begin{equation*}
\begin{aligned}
d_{\TL_d(A)}((f,\bbP),(g,\bbQ))& =  \bigg(\int_{A\times A}d(x,y)^p+|f(x)-g(y)|^p \,\dd\pi^*(x,y) \bigg)^{\frac{1}{p}} \\    
& =  \bigg(\int_{(A \times \bbR)\times (A \times \bbR)}d(x,y)^p + |s-t|^p \, \dd\tilde{\pi}^*((x,s),(y,t))\bigg)^{\frac{1}{p}} \\
& \geq \inf_{\tilde{\pi}\in \Pi(\tilde{\bbP},\tilde{\bbQ})} \bigg(\int_{(A \times \bbR) \times (A \times \bbR)} d(x,y)^p + |s-t|^p \,\dd\tilde{\pi}((x,s),(y,t)) \bigg)^\frac{1}{p} \\
& = d_{\mathrm{W}_{\tilde{d}}^p}(\tilde{\bbP},\tilde{\bbQ}).
\end{aligned}
\end{equation*}

For $\tilde{\pi} \in \mathcal{P}(A \times \bbR \times A \times \bbR)$ we define $\pi=(P^{X} \times P^{Y})_{\#}\tilde{\pi}$ where $P^{X}(x,s,y,t)=x$ and $P^{Y}(x,s,y,t)=y$. Then, $\tilde{\pi} \in \Pi(\tilde{\bbP},\tilde{\bbQ})$ implies $\pi \in \Pi(\bbP,\bbQ)$ and $\tilde{\pi}=((\Id \times f),(\Id \times g))_{\#}\pi$. 
Hence, assuming $\tilde{\pi}^*\in\Pi(\bbP,\bbQ)$ is optimal, 
\begin{equation*}
\begin{aligned}
d_{\mathrm{W}_{\tilde{d}}^p}(\tilde{\bbP},\tilde{\bbQ}) 
 & = \bigg(\int_{(A \times \bbR) \times (A \times \bbR)} d(x,y)^p + |s-t|^p \, \dd\tilde{\pi}^*((x,s),(y,t))\bigg)^{\frac{1}{p}}\\
 &= \bigg(\int_{A \times A} d(x,y)^p + |f(x)-g(y)|^p \, \dd\pi^*(x,y) \bigg)^{\frac{1}{p}}\\
 & \geq \inf_{\pi\in\Pi(\bbP,\bbQ)} \lp\int_{A \times A} d(x,y)^p + |f(x)-g(y)|^p \, \dd\pi(x,y) \rp^{\frac{1}{p}} \\
 & = d_{\TL_d(A)}((f,\bbP),(g,\bbQ)).
\end{aligned}    
\end{equation*}
Hence, we deduce that $d_{\TL_d(A)}((f,\bbP),(g,\bbQ)) = d_{\mathrm{W}^p_{\tilde{d}}}(\tilde{\bbP},\tilde{\bbQ})$.  
Since $\tilde{d}$ is a metric in $A \times \bbR$, this implies that $d_{\mathrm{W}^p_{\tilde{d}}}(\tilde{\bbP},\tilde{\bbQ})$ is a metric on $\mathcal{P}(A \times \bbR)$. 
Therefore, we conclude that  $d_{\TL_d(A)}((f,\bbP),(g,\bbQ))$ is a metric on $\mathcal{P}(A \times \bbR)$. 

\item 
Assume that $d_{\TL_{d}(A)}((f_n,\bbP_n),(f,\bbP)) \to 0$. Then, there exists an optimal $\pi^*_n \in \Pi(\bbP,\bbP_n)$ such that 
\[ \bigg(  \int_{A \times A}  d(x,y)^p+|f(x)-f_n(y)|^p \, \dd{\pi}^*_n(x,y) \bigg)^{\frac{1}{p}} \to 0. \]
This implies that $\|d\|_{\Lp{p}(\pi^*_n)} \to 0$ and $\|f(x)-f_n(y) \|_{\Lp{p}(\pi^*_n(x,y))} \to 0$. Alternatively, assume that there exists an optimal $\pi^*_n \in \Pi(\bbP,\bbP_n)$ such that $\|d\|_{\Lp{p}(\pi^*_n)} \to 0$ and $\|f(x)-f_n(y)\|_{\Lp{p}(\pi^*_n(x,y))} \to 0$, then  
\[ d_{\TL_{d}(A)}((f,\bbP),(f_n,\bbP_n)) \leq \bigg(\int_{A \times A} d(x,y)^p + |f(x)-f_n(y)|^p \,\dd\pi^*_n(x,y)\bigg)^{\frac{1}{p}} \to 0. \] 

\item  We assume that $d_{\TL_d(A)}((f_n,\bbP_n),(f,\bbP)) \to 0$ and $\pi_n \in \Pi(\bbP,\bbP_n)$ satisfies $\|d\|_{\Lp{p}(\pi_n)} \to 0$. We know that $\int_{A \times A} |f(x)-f_n(y)|^p + d(x,y)^p \,\dd\pi^*_n(x,y) \to 0$ and $\int_{A \times A} d(x,y)^p \,\dd\pi_n(x,y) \to 0$ where $\pi^{*}_n \in \Pi(\bbP,\bbP_n)$ is the sequence of optimal transport plans from Part~\ref{(a)}. For any $g: A \to \bbR$ we can write, 
\begin{align}\label{eqn:f(x)_f_n(y)}
\begin{aligned}
& \bigg(\int_{A \times A} |f(x)-f_n(y)|^p \, \dd\pi_n(x,y) \bigg)^{\frac{1}{p}} \\
& \qquad \qquad \leq \bigg(\int_{A \times A} |f(x)-g(x)|^p \, \dd\pi_n(x,y) \bigg)^{\frac{1}{p}} + \bigg(\int_{A \times A} |g(x)-g(y)|^p \, \dd\pi_n(x,y) \bigg)^{\frac{1}{p}}  \\
& \qquad \qquad \qquad \qquad + \bigg(\int_{A \times A} |g(y)-f_n(y)|^p \, \dd\pi_n(x,y) \bigg)^{\frac{1}{p}}
\end{aligned}    
\end{align}
We choose a Lipschitz continuous function $g$ such that it approximates $f \in \Lp{p}(\bbP)$ with error $\frac{\delta}{2}$ in $\Lp{p}(\bbP)$.
We can express \eqref{eqn:f(x)_f_n(y)} as follows,
\begin{align*}
& \bigg(\int_{A \times A} |f(x)-f_n(y)|^p \,\dd\pi_n(x,y) \bigg)^{\frac{1}{p}} \\
& \qquad \qquad \leq \bigg(\int_{A} |f(x)-g(x)|^p \,\dd\bbP(x) \bigg)^{\frac{1}{p}} + \Lip(g) \bigg(\int_{A \times A} d(x,y)^p \,\dd\pi_n(x,y) \bigg)^{\frac{1}{p}} \\
& \qquad \qquad \qquad \qquad + \bigg(\int_{A \times A}|g(y)-f_n(y)|^p\,\dd\pi^{*}_n(x,y) \bigg)^{\frac{1}{p}}\\
& \qquad \qquad \leq \frac{\delta}{2} + \Lip(g) \bigg(\int_{A \times A} d(x,y)^p\,\dd\pi_n(x,y) \bigg)^{\frac{1}{p}} + \bigg(\int_{A \times A} |g(x)-g(y)|^p \,\dd\pi^{*}(x,y) \bigg)^{\frac{1}{p}} \\
& \qquad \qquad \qquad \qquad + \bigg(\int_{A \times A} |g(x)-f(x)|^p \,\dd\pi^{*}(x,y) \bigg)^{\frac{1}{p}} + \bigg(\int_{A \times A} |f(x)-f_n(y)|^p \,\dd\pi^{*}(x,y) \bigg)^{\frac{1}{p}}  \\
& \qquad \qquad \leq \delta + \Lip(g)\bigg(\|d\|_{\Lp{p}(\pi_n)} + \|d\|_{\Lp{p}(\pi^{*}_n)}\bigg) + \bigg(\int_{A \times A} |f(x)-f_n(y)|^p \,\dd\pi^{*}(x,y) \bigg)^{\frac{1}{p}}\\
& \qquad \qquad \to \delta \quad \mbox{ as $n$ $\to \infty$.}
\end{align*}

Taking $\delta \to 0$, we find that $\bigg(\int_{A \times A} |f(x)-f_n(y)|^p \, \dd\pi_n(x,y) \bigg)^{\frac{1}{p}} \to 0$. \qedhere
\end{enumerate}
\end{proof}

\begin{remark}
\label{rem:Back:OT:TLpConv}
If the sequence of transport plans $\pi_n \in \Pi(\bbP,\bbP_n)$ can be expressed as $\pi_n= (\Id \times T_n)_\# \bbP$ where $T_{n\#}\bbP=\bbP_n$, then we conclude that
\begin{enumerate}
    \item $\|d\|_{\Lp{p}(\pi_n)} \to 0$ if and only if $\|d(\cdot,T_n)\|_{\Lp{p}(\bbP)} \to 0$, and 
    \item $\| f(x) - f_n(y)\|_{\Lp{p}(\pi_n(x,y))} \to 0$ if and only if $\|f-f_n \circ T_n \|_{\Lp{p}(\bbP)} \to 0$.  
\end{enumerate}
\end{remark}

\subsection{Sobolev Spaces on Wasserstein Submanifolds}\label{section:Sobolev_Spaces}

In this section, we review the notion of Sobolev spaces on an arbitrary metric space, adapted from \cite{Piotr_1,Piotr_2,Piotr_3} to our current context. Several other works, such as \cite{Fornasier_1} and \cite{Fornasier_2} have used the construction of Sobolev spaces on arbitrary metric spaces for relaxing the Cheeger energy using  Lipschitz functions. 
Let $\Omega \subset \bbR^k$ be an open set with $1\leq p < \infty$.  
We recall the standard definition of the Sobolev space,
\begin{equation*}
\begin{aligned}
\Wkp{1}{p}(\Omega) & =\{f \in \Lp{p}(\Omega) \,:\, \nabla f \in \Lp{p}(\Omega)\},
\end{aligned}    
\end{equation*}
where $\nabla f$ is understood in the weak sense.
The space $\Wkp{1}{p}(\Omega)$ is Banach space when endowed with the norm $\|f\|_{\Wkp{1}{p}}:=\|f\|_{\Lp{p}}+\|\nabla f \|_{\Lp{p}}$.
We define the seminorm $\|f\|_{\Lp{1,p}}:=\|\nabla f\|_{\Lp{p}}$. 
The above definition of a Sobolev space strongly relies on the Euclidean structure of the underlying domain $\Omega$. In order to define Sobolev spaces on arbitrary metric measure spaces, we need to consider an alternative definition that does not involve derivatives. In order to do so, firstly we state a Lipschitz type characterization of Sobolev spaces in the Euclidean setting.
 
\begin{theorem}\cite[Theorem 1]{Piotr_1}\label{Thm:Lipschitz}
Let $f$ be a measurable function on $\Omega \subset \bbR^k$. Then, $f \in \Lp{1,p}(\Omega)$ for $1 < p \leq \infty$ if and only if there exists a non-negative function $g \in \Lp{p}(\Omega)$ such that $|f(x)-f(y)| \leq |x-y|(g(x)+g(y))$ holds a.e.    
\end{theorem}

This characterization of Sobolev spaces does not require derivatives,  particularly significant in our context, as we need this definition to characterize Sobolev spaces on Wasserstein Submanifolds. 
We adapt the definition from \cite[Section 3]{Piotr_1} to our context.

\begin{mydef} 
\label{def:Sobolev_Space_Wasserstein_Submanifolds}
Let $\Omega \subset \bbR^k$ be an open and bounded domain with $k \geq$ 1. Let $\Lambda \subset \Probac{(\Omega)}$ and $\bbP_\Lambda \in \cP(\Lambda)$. We define the spaces $\Lp{1,p}(\Lambda,\dWp{2},\bbP_{\Lambda})$ and $\Wkp{1}{p}(\Lambda,\dWp{2},\bbP_{\Lambda})$ by
\begin{equation*}
\begin{aligned}
\Lp{1,p}(\Lambda,\dWp{2},\bbP_\Lambda) =\big\{f: \Lambda \to \bbR \textrm{ s.t. } f \text{ is measurable and } \exists \Lambda^{\prime}\subset\Lambda \textrm{ and } 0\leq g\in \Lp{p}(\bbP_\Lambda)\\
\textrm{ with } \bbP_{\Lambda}(\Lambda^{\prime})=0 \textrm{ and } |f(\mu)-f(\nu)| \leq \dWp{2}(\mu,\nu)(g(\mu)+g(\nu)) \textrm{ for all } \mu,\nu \in \Lambda \setminus \Lambda^{\prime} \big\},\\
\end{aligned}
\end{equation*}
and 
\begin{equation*}
\Wkp{1}{p}(\Lambda,\dWp{2},\bbP_\Lambda)=\Lp{p}(\bbP_\Lambda) \cap \Lp{1,p}(\Lambda,\dWp{2},\bbP_\Lambda).   
\end{equation*}
\end{mydef} 

We also state that $\Wkp{1}{p}(\Lambda,\dWp{2},\bbP_\Lambda)$ is Banach space when endowed with the norm $\|f\|_{\Wkp{1}{p}}=\|f\|_{\Lp{p}(\bbP_\Lambda)}+\|f\|_{\Lp{1,p}(\Lambda,\dWp{2},\bbP_\Lambda)}$ and $\Lp{1,p}(\Lambda,\dWp{2},\bbP_\Lambda)$ is endowed with the semi-norm $\|f\|_{\Lp{1,p}(\Lambda,\dWp{2},\bbP_\Lambda)}=\inf_g \|g\|_{\Lp{p}(\bbP_\Lambda)}$ where the infimum is taken over all $g$ satisfying the following condition,
\begin{equation*}
|f(\mu)-f(\nu)| \leq \dWp{2}(\mu,\nu)(g(\mu)+g(\nu)),    
\end{equation*}
for all $\mu,\nu \in \Lambda \setminus \Lambda^{\prime}$ and for some $\Lambda^{\prime} \subset \Lambda$ with $\bbP_\Lambda(\Lambda^\prime)=0$.

\begin{proposition}\label{prop:Sobolev_Space_1}
Assume $\Omega$, $\Lambda$, $\mathcal{S}$, $\bbP_\Lambda$, $\bbP_\mathcal{S}$, and $E$ satisfy~\ref{ass:domain_ll_one}-\ref{ass:domain_ll_three}, \ref{ass:prob_measures_S}-\ref{ass:prob_measures_Lambda} and~\ref{ass:E_map}. Then for any $p \geq 1$,
\begin{equation*}
\Wkp{1}{p}(\Lambda,\dWp{2},\bbP_\Lambda) = \big\{f:\Lambda \to \bbR \textrm{ such that } f\circ E \in \Wkp{1}{p}(\mathcal{S})\big\}.   
\end{equation*}
\end{proposition}

\begin{proof}
We prove $\subseteq$ in Step 1 and $\supseteq$ in Step 2.
\begin{steps}
\item \label{part_1} 
Let us take $f \in \Wkp{1}{p}(\Lambda,\dWp{2},\bbP_{\Lambda})$ and define $\tilde{f}=f\circ E : \mathcal{S} \to \bbR$. 
First, we claim that if $f \in \Lp{p}(\bbP_\Lambda)$, then $\tilde{f} \in \Lp{p}(\mathcal{S})$. For $c= \bigg(\inf_{\theta \in \mathcal{S}} \rho_\mathcal{S}(\theta)\bigg)^{\frac{1}{p}}$ we have that,

\begin{align}\label{eqn:tildef}
\|f\|_{\Lp{p}(\bbP_\Lambda)}  =  \lp\int_{\Omega} |f|^p  \, \dd\bbP_\Lambda\rp^{\frac{1}{p}} = \lp\int_{\cS} |f\circ E|^p  \, \dd \bbP_\cS\rp^{\frac{1}{p}} = \lp \int_\cS |\tilde{f}|\rho_\cS \, \dd \theta \rp^{\frac{1}{p}} \geq c\| \tilde{f} \|_{\Lp{p}(\mathcal{S})}.
\end{align}
Next, we let $L$ be the Lipschitz constant for $E$ so that,
\begin{equation*}
L^{-1}|\theta-\tau| \leq \dWp{2}(E(\theta),E(\tau)) \leq L|\theta-\tau|.    
\end{equation*}
and let $g,\Lambda^{\prime} \subset \Lambda$ satisfy the definition of $\Wkp{1}{p}(\Lambda,\dWp{2},\bbP_\Lambda)$ in Definition \ref{def:Sobolev_Space_Wasserstein_Submanifolds} for $f$.
We define $\tilde{g}:=Lg \circ E$. We have that,
\begin{equation*}\label{eqn:eq_euclidean_dist}
\begin{aligned}
|f(\mu)-f(\nu)| & \leq \dWp{2}(\mu,\nu) (g(\mu)+g(\nu)) \quad \mbox{for all $\mu,\nu \in \Lambda \setminus \Lambda^\prime$} \\
\implies |f\circ E(\theta)-f\circ E(\tau)| & \leq L |\theta-\tau| (g\circ E(\theta)+g\circ E(\tau)) \quad \mbox{for all $E(\theta),E(\tau) \in \Lambda \setminus \Lambda^\prime$}\\
 \implies |f\circ E(\theta)-f\circ E(\tau)| & \leq |\theta-\tau| (Lg\circ E(\theta)+Lg\circ E(\tau)) \quad \mbox{for all $E(\theta),E(\tau) \in \Lambda \setminus \Lambda^\prime$} \\
\implies  |\tilde{f}(\theta) - \tilde{f}(\tau)| & \leq |\theta-\tau| (\tilde{g}(\theta)+\tilde{g}(\tau)) \quad \mbox{for all $E(\theta),E(\tau) \in \Lambda \setminus \Lambda^\prime$}
\end{aligned}    
\end{equation*}
where $\bbP_\Lambda(\Lambda^\prime)=0$. From  \eqref{eqn:tildef} we see that if $g \in \Lp{p}(\bbP_\Lambda)$, then $\tilde{g} \in \Lp{p}(\mathcal{S})$. 
Further, we claim that $\bbP_\Lambda(\Lambda^{\prime})=0$ if and only if $\Vol(\{\theta:E(\theta) \in \Lambda^{\prime}\}) = \Vol(E^{-1}(\Lambda^{\prime}))=0$. By definition, for $C=\frac{1}{\inf_{\theta \in \mathcal{S}} \rho_\mathcal{S}(\theta)}$,
\begin{equation*}\label{eqn:Vol_E_inverse}
 \Vol(E^{-1}(\Lambda^{\prime})) \leq \int_{E^{-1}(\Lambda^\prime)} \frac{\rho_{\mathcal{S}}(\theta)}{\inf_{\theta \in \mathcal{S}}\rho_{\mathcal{S}}(\theta)} \, \dd\theta = C \int_{E^{-1}(\Lambda^\prime)} \,\dd\bbP_\mathcal{S}= C \int_{\Lambda^\prime} \, \dd\bbP_\Lambda=0.
\end{equation*}
Hence, we have $\tilde{f} \in \Wkp{1}{p}(\mathcal{S})$ by Theorem \ref{Thm:Lipschitz}.

\item \label{part_2}
Assume $\tilde{f}\in\Wkp{1}{p}(\cS)$ and define $f = \tilde{f}\circ E^{-1}$.
We will show $f \in \Wkp{1}{p}(\Lambda,\dWp{2},\bbP_{\Lambda})$.
First, we show that if $\tilde{f} \in \Lp{p}(\mathcal{S})$, then $f \in \Lp{p}(\bbP_\Lambda)$. For $c=\bigg(\frac{1}{\sup_{\theta \in \mathcal{S}}\rho_\mathcal{S}(\theta)}\bigg)^{\frac{1}{p}}$ it follows that,
\begin{align}\label{eqn:f}
\|\tilde{f}\|_{\Lp{p}(\mathcal{S})}  = \bigg( \int_{\mathcal{S}} |\tilde{f}|^p  \, \dd\theta \bigg)^{\frac{1}{p}} \geq \frac{1}{(\sup_{\theta \in \mathcal{S}}\rho_\mathcal{S}(\theta))^\frac{1}{p}} \|\tilde{f}\|_{\Lp{p}(\bbP_\mathcal{S})}
 = c\| \tilde{f} \|_{\Lp{p}(\bbP_\mathcal{S})}   = c\| f \|_{\Lp{p}(\bbP_\Lambda)}.
\end{align}
Next, we claim by Theorem \ref{Thm:Lipschitz} that there exists 0 $\leq \tilde{g} \in \Lp{p}(\mathcal{S})$ and $\mathcal{S}^{\prime}$ such that $|\tilde{f}(\theta)-\tilde{f}(\tau)| \leq |\theta -\tau| (\tilde{g}(\theta)+\tilde{g}(\tau))$ for all $\theta,\tau \in \mathcal{S}\setminus \mathcal{S}^\prime$ and $\Vol(\mathcal{S}^\prime)=0$. Define $g=L\tilde{g}\circ E^{-1}$ so,
\begin{equation*}\label{eqn:eq_wasser_dist}
\begin{aligned}
|\tilde{f}(\theta) - \tilde{f}(\tau)| & \leq |\theta-\tau| (\tilde{g}(\theta)+\tilde{g}(\tau)) \quad \mbox{ for all $\theta,\tau \in \mathcal{S} \setminus \mathcal{S}^\prime$} \\
\implies |f\circ E(\theta) - f\circ E(\tau)| & \leq L^{-1}|\theta-\tau| (g \circ E(\theta)+g \circ E(\tau)) \quad \mbox{ for all $\theta,\tau \in \mathcal{S} \setminus \mathcal{S}^\prime$}\\
\implies |f(\mu) - f(\nu)| & \leq \dWp{2}(\mu,\nu) (g(\mu)+g(\nu)) \quad \mbox{ for all $E^{-1}(\mu),E^{-1}(\nu) \in \mathcal{S} \setminus \mathcal{S}^\prime$} \\
\end{aligned}    
\end{equation*}
From  \eqref{eqn:f} we see that if $\tilde{g} \in \Lp{p}(\mathcal{S})$ then $g \in \Lp{p}(\bbP_\Lambda)$. Further, we claim that $\Vol(\mathcal{S}^\prime)=0$ implies that $\bbP_\Lambda(\Lambda^\prime)=0$ where $\Lambda^\prime=E(\mathcal{S}^\prime)$.
In particular, for $C = \sup_{\theta \in \mathcal{S}} \rho_\mathcal{S}(\theta)$ we have,
\begin{equation*}\label{eqn:Vol_E}
\begin{aligned}
\int_{\Lambda^\prime} \, \dd\bbP_\Lambda & =  \int_{E^{-1}(\Lambda^\prime)} \, \dd\bbP_\mathcal{S}  = \int_{E^{-1}(\Lambda^\prime)}\rho_\mathcal{S}(\theta) \,\dd\theta \leq C \int_{E^{-1}(\Lambda^\prime)} \, \dd\theta  = C\Vol(E^{-1}(\Lambda^{\prime})) = C\Vol(\mathcal{S}^\prime) =0. \\
 \end{aligned}
\end{equation*}
Hence, we have that $f\in \Wkp{1}{p}(\Lambda,\dWp{2},\bbP_\Lambda)$.
\end{steps}
Therefore, $\Wkp{1}{p}(\Lambda,\dWp{2},\bbP_\Lambda) = \big\{f:\Lambda \to \bbR \textrm{ such that } f\circ E \in \Wkp{1}{p}(\mathcal{S})\big\}$
\end{proof}

Next, we characterise the continuum topology that arises from the large data limit of Laplace Learning. By the following lemma our definition of $\cE_\infty$ in~\eqref{eqn:continuum_energy_functional_intro} is well-defined.

\begin{lemma} 
Assume $\Omega$, $\Lambda$, $\mathcal{S}$, $\bbP_\Lambda$, $\bbP_\mathcal{S}$, $E$, $B_\theta$, and $\eta$ satisfy~\ref{ass:domain_ll_one}-\ref{ass:domain_ll_three}, \ref{ass:prob_measures_S}-\ref{ass:prob_measures_Lambda}, \ref{ass:E_map}-\ref{ass:B_map} and~\ref{ass:weight_one}-\ref{ass:weight_four}. Define $\mathcal{E}_\infty$ by  \eqref{eqn:continuum_energy_functional_intro}. For any $p \geq 1$ and $0<c<C<\infty$, let $f\in \Lp{p}(\bbP_\Lambda)$ then 
$f \in \Wkp{1}{p}(\Lambda,\dWp{2},\bbP_\Lambda)$ if and only if
\begin{align} \label{eq:Back:Sobolev_Spaces:FinE}
\int_{\bbR^d}\int_{\mathcal{S}}|\nabla (f\circ E)(\theta) \cdot h|^p \eta(\|B_\theta(h)\|_{\Lp{2}(E(\theta))})\rho^2_\mathcal{S}(\theta) \, \dd \theta \, \dd h < \infty.
\end{align}
\end{lemma}

\begin{proof}
First, we assume that $f \in \Wkp{1}{p}(\Lambda,\dWp{2},\bbP_\Lambda)$. 
By~\ref{ass:B_map} and~\ref{ass:weight_one},
\begin{equation*}
\begin{aligned}
& \int_{\bbR^d}\int_{\mathcal{S}}|\nabla (f\circ E)(\theta) \cdot h|^p \eta(\|B_\theta(h)\|_{\Lp{2}(E(\theta))})\rho^2_\mathcal{S}(\theta) \, \dd \theta \, \dd h \\
& \qquad \qquad \leq  \sup_{\theta \in \mathcal{S}}\rho_\mathcal{S}^2(\theta)\int_{\bbR^d}\int_\mathcal{S}\eta(c|h|)|\nabla (f\circ E)(\theta)\cdot h|^p \, \dd\theta \, \dd h   \\
& \qquad \qquad = c^{-p-d} \sup_{\theta \in \mathcal{S}} \rho_\mathcal{S}^2(\theta) \int_{\bbR^d}  \eta(|g|)|g_1|^p\,\dd g \int_\mathcal{S}
|\nabla (f\circ E)(\theta)|^p \, \dd\theta     \\
& \qquad \qquad = c^{-p-d} \sup_{\theta \in \mathcal{S}} \rho_\mathcal{S}^2(\theta)  \int_{\bbR^d} \eta(|g|)|g_1|^p \, \dd g \, \|f\circ E\|^p_{\Lp{1,p}(\mathcal{S})} \\
& \qquad \qquad < +\infty
\end{aligned}    
\end{equation*}
where $g_1$ is the first coordinate of vector $g$.
The last inequality follows from Proposition \ref{prop:Sobolev_Space_1}, and assumptions 
\ref{ass:prob_measures_Lambda}, \ref{ass:weight_four}. 
On the other hand, if we assume \eqref{eq:Back:Sobolev_Spaces:FinE}, 
by a change of variables it follows that,

\begin{align*}
\inf_{\theta \in \mathcal{S}}\rho_\mathcal{S}^2(\theta)\int_{\bbR^d}\eta(C|h|)|h_1|^p\, \dd h\|f \circ E\|^p_{\Lp{1,p}(\mathcal{S})} 
 & \leq \inf_{\theta \in\cS} \rho_\cS^2(\theta) \int_{\bbR^d} \eta(\|B_\theta(h)\|_{\Lp{2}(E(\theta))}) |h_1|^p \, \dd h \| f\circ E\|_{\Lp{1,p}(\cS)}^p \\
 & \leq \int_{\bbR^d} \int_{\cS} \eta(\|B_\theta(h)\|_{\Lp{2}(E(\theta))})|\nabla (f \circ E)(\theta) \cdot g|^p \rho_\cS^2(\theta) \, \dd\theta \, \dd g \\
 & = \mathcal{E}_\infty(f) \\
 & < + \infty.
\end{align*}

Hence, by \ref{ass:prob_measures_S} and \ref{ass:weight_two}, $f\circ E\in\Lp{1,p}(\cS)$ and by Proposition \ref{prop:Sobolev_Space_1} we have $f\in\Lp{1,p}(\Lambda,\dWp{2},\bbP_\Lambda)$.
We conclude that~\eqref{eq:Back:Sobolev_Spaces:FinE} holds if and only if $f \in \Wkp{1}{p}(\Lambda,\dWp{2},\bbP_\Lambda)$.
\end{proof}

\subsection{\texorpdfstring{$\Gamma$}{Gamma}- Convergence} \label{Section:Gamma_convergence}

The notion of $\Gamma$-convergence is particularly suited for handling sequences of minimization problems, where it can be interpreted as the ‘limiting lower semi-continuous envelope’. This concept is especially valuable when dealing with functionals that exhibit significant oscillations. In such cases, there may not be a strong limit, and any weak limit that does exist may not capture convergence of minima/minimisers. 
For a comprehensive introduction, see \cite{Dal_Maso_Gamma_Convergence,Braides_Gamma_Convergence}. 

\begin{mydef}[$\Gamma$-Convergence]
Let $(A,d)$ be a metric space, $\Lp{0}(A;\bbR \cup \{\pm \infty\})$ be the set of measurable functions from $A$ to $\bbR \cup \{\pm \infty\}$, and $(\mathcal{X},\bbP)$ be a probability space. The function $\mathcal{X} \ni \omega \mapsto \mathcal{E}^{(\omega)}_{n} \in \Lp{0}(A;\bbR \cup \{\pm \infty\})$ is a random variable. We say that $\mathcal{E}^{(\omega)}_{n}: A \to [0,\infty]$ $\Gamma$-converges almost surely on $A$ to $\mathcal{E}_\infty: A \to \bbR \cup \{\pm \infty\}$ with respect to the metric $d$, and write $\mathcal{E}_\infty=\Glim_{n \to \infty} \mathcal{E}^{(\omega)}_{n}$, if there exists a set $\mathcal{X}' \subset \mathcal{X}$ with $\bbP(\mathcal{X}')=1$, such that for all $\omega \in \mathcal{X}'$ and all $f \in A$:
\begin{enumerate}
\item{(\textnormal{lim inf-inequality})} for every sequence $\{f_n\}_{n=1}^\infty$ in $A$ converging to $f$
\begin{equation*}
 \mathcal{E}_\infty(f) \leq \liminf_{n \to \infty} \mathcal{E}^{(\omega)}_{n}(f_n), \quad \mbox{ and}
\end{equation*}
\item{(\textnormal{lim sup-inequality})} there exists a sequence $\{f_n\}_{n=1}^\infty$ in $A$ converging to $f$ such that
\begin{equation*}
\mathcal{E}_\infty(f) \geq \limsup_{n \to \infty} \mathcal{E}^{(\omega)}_{n}(f_n).    
\end{equation*}
\end{enumerate}
\end{mydef}
For simplicity we suppress the dependence of $\omega$ in writing the functionals. The randomness in our energy comes from the samples $\{\mu_i^{(m_n)}\}_{i=1}^n = \Lambda_n^{(m_n)}$.
A consequence of Theorem~\ref{thm:rates_convergence} is that, with probability one, $\bbP_{\Lambda_n^{(m_n)}}\weakstarto \bbP_\Lambda$.
We will work on the set of realisations of $\Lambda_n^{(m_n)}$ such that $\bbP_{\Lambda_n^{(m_n)}}\weakstarto \bbP_\Lambda$.

\begin{remark}
The fundamental property of $\Gamma$-convergence is that, together with a compactness property it implies the convergence of minimisers.  
The following theorem is also true if we replace minimisers with approximate minimisers.     
\end{remark} 

\begin{theorem}[Convergence of Minimisers] \cite[Theorem 1.21]{Braides_Gamma_Convergence} and \cite[Theorem 7.23]{Dal_Maso_Gamma_Convergence}.
Let $(A,d)$ be the metric space and $\mathcal{E}_{n}: A \to [0,\infty]$ be a sequence of functionals. Let $f_n$ be the minimising sequence for $\mathcal{E}_{n}$. If the set $\{f_n\}_{n=1}^\infty$ is pre-compact and $\mathcal{E}_\infty=\Glim_{n \to \infty}\mathcal{E}_{n}$ where $\mathcal{E}_\infty: A \to [0,\infty]$ is not identically $\infty$ then 
\begin{equation*}
 \min_{A} \mathcal{E}_\infty=\lim_{n \to \infty}\inf_{A}\mathcal{E}_n.   
\end{equation*}
Furthermore, any cluster point of $\{f_n\}_{n=1}^\infty$ is a minimiser of $\mathcal{E}_\infty$.
\end{theorem}
We recall that $\Gamma$-convergence is defined for functionals on general metric spaces. Subsection~\ref{section:TLp} provides an overview of the metric space framework employed in our setting to analyse the asymptotic behaviour of the energy functional. In particular, this framework provides the suitable notion of variational convergence from discrete to continuum energy functional.

\section{Proofs}\label{section:Proofs}

In this section, we first establish a compactness result in Subsection~\ref{section:compactness}, which ensures that any sequence of discrete minimizers admits a subsequence converging in the \(\TLp{p}\) topology. Together with this result, we prove in Subsection~\ref{section:gamma_convergence}, the \(\Gamma\)-convergence of the discrete energy functional 
\(\mathcal{E}_{\varepsilon_n,m_n,n}\) on \(\Lambda_n^{(m_n)}\) to its continuum counterpart 
\(\mathcal{E}_\infty\) on the Wasserstein submanifold \(\Lambda \subset \Probac{(\Omega)}\). Additionally, we also define the discrete graph Laplacian $\mathcal{L}_{\varepsilon_n,m_n,n}$ and its continuum counterpart Laplace-Beltrami operator $\Delta_{\Lambda}$ in Subsection \ref{section:graph_laplacians}.

\subsection{Compactness}\label{section:compactness}

We establish the almost sure convergence of the discrete transport maps  $D_n(\mu_i)=\mu^{(m_n)}_i$ to the identity map in the $\Lp{\infty}$ sense. The following lemma formalizes this result.

\begin{lemma}\label{lemma:transport_maps_Dn} Assume $\Omega$, $\Lambda$, $\bbP_\Lambda$, $\bbP_{\Lambda_n}$, $\bbP_{\Lambda^{(m)}_n}$ satisfy \ref{ass:domain_ll_one}-\ref{ass:domain_ll_two}, \ref{ass:prob_measures_Lambda}, \ref{ass:Lambda_n}-\ref{ass:Lambda_m_n}.
Furthermore, assume $\sup\limits_{n \in \bbN} \frac{m_n}{n} < +\infty$.
Define $D_n: \Lambda_n \to \Lambda^{(m)}_n$ by $D_n(\mu_i)=\mu^{(m_n)}_i$. 
Then, $D_n \to \Id$ with probability one in $\Lp{\infty}$. 
In particular, with probability one there exists $N <\infty$ such that for all $n \geq N$ we have
\[ \max_{i\in\{1,\dots,n\}} \dWp{2}(D_n(\mu_i),\mu_i) \leq \hat{C}\tilde{q}_k(m_n) \]
where  
$\hat{C}$ is a constant independent of $n$ and 
$\tilde{q}_k(m_n)$ is defined in~\eqref{eqn:rates_qd}.
\end{lemma}

Clearly, under the assumptions of the previous lemma we also have 
\[ \lp \int_{\Lambda_n} \dWp{2}(D_n(\mu),\mu)^p \, \dd \bbP_{\Lambda_n}(\mu) \rp^{\frac{1}{p}} \leq \hat{C}\tilde{q}_k(m_n) \]
for all $n\geq N$.

\begin{proof}
By Theorem~\ref{thm:rates_convergence} there exist constants $C$ and $\hat{C}$ such that 
\begin{equation*}\label{eqn:probability_bounds}
\begin{aligned}
 \bbP\lp\max_{i=1,\dots,n} \dWp{2}(\mu^{(m_n)}_i, \mu_i) \leq \hat{C} \tilde{q}_k(m_n)\rp
 & = \bbP(\dWp{2}(\mu^{(m_n)}_i, \mu_i) \leq \hat{C}\tilde{q}_k(m_n) \quad \mbox{ for all $i=1,\cdots n$})\\
 & = 1- \bbP(\dWp{2}(\mu^{(m_n)}_i,\mu_i) > \hat{C}\tilde{q}_k(m_n) \quad \mbox{ for some $i=1,\cdots,n$}) \\
 & \geq 1- \sum_{i=1}^n \bbP(\dWp{2}(\mu^{(m_n)}_i,\mu_i) > \hat{C}\tilde{q}_k(m_n)) \\
 & \geq 1- Cn {m_n}^{-\alpha/2}.
\end{aligned}    
\end{equation*}
From the assumption that $\frac{m_n}{n} \leq C$, there exists $\alpha$ such that $\sum_{n=1}^\infty n{m_n}^{-\alpha/2} < + \infty$ (i.e. $\alpha>4$). Hence, we have
\begin{equation*}\label{eqn:probability_bounds_2}
\sum_{n=1}^\infty \bbP\lp\max_{i=1,\cdots,n} \dWp{2}(\mu^{(m_n)}_i,\mu_i) > \hat{C}\tilde{q}_k(m_n)\rp \leq C \sum_{n=1}^\infty n {m_n}^{-\alpha/2} < + \infty.    
\end{equation*}
By the Borel Cantelli lemma the event $\max\limits_{i=1,\cdots,n} \dWp{2}(\mu^{(m_n)}_i,\mu_i) > \hat{C}\tilde{q}_k(m_n)$ happens finitely many times i.e., with probability one there exists $N < +\infty$ such that 
\begin{align}\label{eqn:Borel_Cantelli_lemma}
\dWp{2}(\mu^{(m_n)}_i,\mu_i) \leq \hat{C}\tilde{q}_k(m_n)
\end{align}
for all $i=1,\cdots,n$ and for all $n \geq N$.
\end{proof}

In the next lemma, we establish an upper bound on the non-local discrete energy functional $\hat{\mathcal{E}}_{\tilde{\varepsilon}_n,n}(f_n \circ D_n \circ E; \hat{\eta})$, defined on $\mathcal{S}_n$, in terms of the discrete energy functional $\mathcal{E}_{\varepsilon_n,m_n,n}(f_n;\eta)$, defined on $\Lambda^{(m_n)}_n$. 

\begin{lemma}\label{lemma:sup_energy_bounded}
 Assume $\Omega$, $\Lambda$, $\mathcal{S}$, $\bbP_{\mathcal{S}_n}$, $\bbP_{\Lambda_n}$, $\bbP_{\Lambda^{(m)}_n}$, $E$, $\eta$, and $\varepsilon_n$ satisfy~\ref{ass:domain_ll_one}-\ref{ass:domain_ll_three}, \ref{ass:S_n}-\ref{ass:Lambda_m_n}, \ref{ass:E_map}, \ref{ass:weight_one}-\ref{ass:weight_four}, and \ref{ass:eps_scale}-\ref{ass:eps_scale_2}. Define $\mathcal{E}_{\varepsilon_n,m_n,n}(\cdot;\eta) : \Lp{p}(\bbP_{\Lambda^{(m_n)}_n}) \to [0,+\infty)$ by~\eqref{eqn:discrete_energy_intro} and $\hat{\mathcal{E}}_{\varepsilon_n,n}(\cdot;\hat{\eta}) : \Lp{p}(\bbP_{\mathcal{S}_n}) \to [0,+\infty)$ by~\eqref{eqn:discrete_energy_parameter_intro}. Let $D_n: \Lambda_n \to \Lambda^{(m_n)}_n$ be defined by $D_n(\mu_i)=\mu^{(m_n)}_i$. 
Furthermore, assume $\sup\limits_{n \in \bbN} \frac{n}{m_n} < +\infty$. 
Then, with probability one, there exists $N<+\infty$, $C>0$, $\hat{\eta}:\bbR^d \to [0,\infty)$ 
and $\tilde{\varepsilon}_n$ 
satisfying, for all $n\geq N$, $0 < \liminf_{n \to \infty} \frac{\tilde{\varepsilon}_n}{\varepsilon_n} \leq \limsup_{n \to \infty} \frac{\tilde{\varepsilon}_n}{\varepsilon_n} < +\infty $ and 
\begin{equation*}
\hat{\mathcal{E}}_{\tilde{\varepsilon}_n,n}(f_n \circ D_n \circ E; \hat{\eta}) \leq C \mathcal{E}_{\varepsilon_n,m_n,n}(f_n;\eta)
\end{equation*}
for all $f_n : \Lambda^{(m_n)}_n \to \bbR$.
\end{lemma}

\begin{proof}
First we claim that there exists $\tilde{\eps}_n$ and $\hat{\eta}$ satisfying the conditions in the statement and such that
\begin{align} \label{eq:Proofs:Compact:etalowerbound}
\hat{\eta}\lp\frac{|\theta_i-\theta_j|}{\tilde{\eps}_n}\rp \leq \eta \lp \frac{\dWp{2}(\mu_i^{(m_n)},\mu_j^{(m_n)})}{\eps_n}\rp.
\end{align}

By Lemma~\ref{lemma:transport_maps_Dn}, $\dWp{2}(\mu_i^{(m_n)},\mu_i)\leq \hat{C}\tilde{q}_k(m_n)$.
Using the triangle inequality we can bound,
\begin{align}\label{eqn:rates_dwp}
\begin{aligned}
|\dWp{2}(\mu^{(m_n)}_i,\mu^{(m_n)}_j) - \dWp{2}(\mu_i,\mu_j) | & \leq |\dWp{2}(\mu^{(m_n)}_i,\mu^{(m_n)}_j) - \dWp{2}(\mu_i,\mu_j^{(m_n)})| \\
& \qquad + |\dWp{2}(\mu_i,\mu^{(m_n)}_j) - \dWp{2}(\mu_i,\mu_j)|  \\
&\leq \dWp{2}(\mu^{(m_n)}_i,\mu_i)+\dWp{2}(\mu^{(m_n)}_j,\mu_j)   \\
&\leq 2\max_{i=1,\cdots,n} \dWp{2}(\mu^{(m_n)}_i,\mu_i)  \\
& \leq \hat{C} \tilde{q}_k(m_n).
\end{aligned}
\end{align}
By \ref{ass:weight_two} and~\ref{ass:weight_three} there exists $a,r>0$ such that
\begin{align}\label{eqn:eta_hat}
\begin{aligned}
\hat{\eta}(t) &:=
\left\{\begin{array}{ll}
a & \text{if } |t| \leq r  \\ 
0 & \text{otherwise} 
\end{array}\right.\\
\end{aligned}    
\end{align}
satisfies $\hat{\eta}\leq \eta$.
Let $L$ be the Lipschitz constant for $E: \mathcal{S} \to \Lambda$ 
so for all $\theta_i,\theta_j \in \mathcal{S}_n$ and $\mu_i,\mu_j \in \Lambda_n$ we have,
\begin{align}\label{eqn:bounds_Wp_S}
\frac{1}{L} |\theta_i-\theta_j| \leq \dWp{2}(E(\theta_i),E(\theta_j)) \leq L |\theta_i-\theta_j|.
\end{align}
Assume $\frac{|\theta_i-\theta_j|}{\tilde{\varepsilon}_n} \leq r$ where $\tilde{\varepsilon}_n:=\frac{\varepsilon_n}{L}- \frac{\hat{C}\tilde{q}_k(m_n)}{Lr}$, then 
\begin{equation*}\label{eqn:tilde_varepsilon}
\begin{aligned}
\dWp{2}(\mu^{(m_n)}_i,\mu^{(m_n)}_j)    
 & \leq \hat{C}\tilde{q}_k(m_n) + \dWp{2}(\mu_i,\mu_j) \\
 & \leq \hat{C}\tilde{q}_k(m_n) + L|\theta_i-\theta_j| \\
& \leq \hat{C}\tilde{q}_k(m_n) + L\tilde{\varepsilon}_n r \\
& = \varepsilon_n r.
\end{aligned}    
\end{equation*}
Since, $\varepsilon_n \gg \tilde{q}_k(m_n)$ then $\tilde{\varepsilon}_n \geq 0$ and $\frac{\varepsilon_n}{\tilde{\varepsilon}_n} \to L$. 
Hence, $\frac{|\theta_i-\theta_j|}{\tilde{\varepsilon}_n} \leq r$ implies $\frac{\dWp{2}(\mu^{(m_n)}_i,\mu^{(m_n)}_j)}{\eps_n}\leq r$.
In particular, if $\frac{|\theta_i-\theta_j|}{\tilde{\varepsilon}_n} \leq r$ then
\begin{equation*}
   \eta\bigg(\frac{\dWp{2}(\mu^{(m_n)}_i,\mu^{(m_n)}_j)}{\varepsilon_n}\bigg)\geq a = \hat{\eta}\bigg(\frac{|\theta_i-\theta_j|}{\tilde{\varepsilon}_n}\bigg).
\end{equation*}
If $\frac{|\theta_i-\theta_j|}{\tilde{\varepsilon}_n} > r$ then $\eta\bigg(\frac{\dWp{2}(\mu^{(m_n)}_i,\mu^{(m_n)}_j)}{\varepsilon_n}\bigg) \geq 0 = \hat{\eta}\bigg(\frac{|\theta_i-\theta_j|}{\tilde{\varepsilon}_n}\bigg)$.
Hence~\eqref{eq:Proofs:Compact:etalowerbound} holds.

Using the claim in \eqref{eq:Proofs:Compact:etalowerbound}, for $n\geq N$ 
\begin{equation*}
\begin{aligned}
\hat{\mathcal{E}}_{\tilde{\varepsilon}_n,n}(\hat{f}_n;\hat{\eta}) &= \frac{1}{{\tilde{\varepsilon}}^{d+p}_n n^2} \sum_{i,j=1}^n \hat{\eta}\bigg(\frac{|\theta_i-\theta_j|}{\tilde{\varepsilon}_n}\bigg)\la \hat{f}_n(\theta_i) - \hat{f}_n(\theta_j) \ra^p \\
& \leq \frac{\varepsilon^{d+p}_n}{\tilde{\varepsilon}^{d+p}_n}\frac{1}{{\varepsilon}^{d+p}_n n^2} \sum_{i,j=1}^n \eta\bigg(\frac{\dWp{2}(\mu^{(m_n)}_i,\mu^{(m_n)}_j)}{{\varepsilon}_n}\bigg)\la f_n(\mu^{(m_n)}_i) - f_n(\mu^{(m_n)}_j)\ra^p \\
& = \bigg(\frac{{\varepsilon}_n}{\tilde{\varepsilon}_n}\bigg)^{d+p} \mathcal{E}_{\varepsilon_n,m_n,n}(f_n;\eta).
\end{aligned}    
\end{equation*}

We also note that $\big(\frac{{\varepsilon}_n}{\tilde{\varepsilon}_n}\big)^{d+p} \to L$ as $n \to \infty$ which completes the proof.
\end{proof}

We will construct a transport map $T_n: \Lambda\to \Lambda_n^{(m)}$ using a transport map $\hat{T}_n:\cS\to\cS_n$ and $E:\cS \to \Lambda$ and $D_n:\Lambda_n\to\Lambda_n^{(m)}$.
In particular, we can control (via Theorem~\ref{thm:rates_convergence}) $\hat{T}_n$, $E$ and $D_n$.

The next proposition contains some simple bounds and properties of $T_n$. 

\begin{proposition}\label{prop:TLp_distance_Wasserstein} 
 Assume $\Omega$, $\Lambda$, $\mathcal{S}$, $\bbP_\Lambda$, $\bbP_\mathcal{S}$, $\bbP_{\Lambda_n}$, $\bbP_{\Lambda^{(m)}_n}$, $\bbP_{\mathcal{S}_n}$, and $E$ satisfy 
\ref{ass:domain_ll_one}-\ref{ass:domain_ll_three}, \ref{ass:prob_measures_S}-\ref{ass:prob_measures_Lambda}, \ref{ass:S_n}-\ref{ass:Lambda_m_n}, and \ref{ass:E_map}. Let $p \geq 1$ and define $D_n: \Lambda_n \to \Lambda^{(m)}_n$ by $D_n(\mu_i)=\mu^{(m_n)}_i$. Assume $\hat{T}_n: \mathcal{S} \to \mathcal{S}_n$ satisfies $\hat{T}_{n\#} \bbP_\mathcal{S}=\bbP_{\mathcal{S}_n}$. 
Define the transport map $T_n:= D_n \circ E \circ \hat{T}_n \circ E^{-1}$. 
Let $f_n: \Lambda^{(m)}_n \to \bbR$ and $f:\Lambda\to \bbR$ and define $\hat{f}_n = f_n\circ D_n \circ E: \mathcal{S}_n \to \bbR$ and $\hat{f}=f\circ E:\cS\to\bbR$.
Then, the following conditions hold true:
\begin{enumerate}
\item\label{item:hat_T_and_Rn}   $T_{n\#}\bbP_\Lambda=\bbP_{\Lambda^{(m)}_n}$;
\item $\|\dWp{2}(T_n,\Id)\|_{\Lp{p}(\bbP_\Lambda)}  \leq  L \|\hat{T}_n-\Id\|_{\Lp{p}(\bbP_\mathcal{S})} + \|\dWp{2}(D_n,\Id) \|_{\Lp{p}(\bbP_{\Lambda_n})}$ where $L>0$ is the Lipschitz constant of $E$;
\item \label{item:equivalence norms}  $\|f_n \circ T_n - f\|_{\Lp{p}(\bbP_\Lambda)} = \|\hat{f}_n\circ \hat{T}_n - \hat{f} \|_{\Lp{p}(\bbP_\mathcal{S})}$.
\end{enumerate}
\end{proposition}

\begin{proof}
\begin{enumerate}
\item
A direct computation shows that

\begin{align*}
T_{n\#}\bbP_\Lambda 
&= \bigl( D_n \circ E \circ \hat{T}_n \circ E^{-1} \bigr)_{\#}\bbP_\Lambda \\
&= \bigl( D_n \circ E \circ \hat{T}_n \bigr)_{\#} \bbP_\mathcal{S} \\
&= \bigl( D_n \circ E \bigr)_{\#} \bbP_{\mathcal{S}_n} \\
&= D_{n\#} \bbP_{\Lambda_n} \\
&= \bbP_{\Lambda^{(m)}_n}.
\end{align*}

Hence, we have $T_{n\#}\bbP_\Lambda = \bbP_{\Lambda^{(m)}_n}$.

\item By the triangle inequality, where $L>0$ is the Lipschitz constant of $E$, we have, 
\begin{equation*}
\begin{aligned}
\|\dWp{2}(T_n,\Id)\|_{\Lp{p}(\bbP_\Lambda)} & = \bigg(\int_{\Lambda}\dWp{2}(T_n(\mu),\mu)^p\,\dd\bbP_\Lambda(\mu) \bigg)^{\frac{1}{p}}  \\  
& \leq \bigg(\int_\Lambda \dWp{2}(\mu,E(\hat{T}_n(E^{-1}(\mu))))^p \, \dd\bbP_{\Lambda}(\mu)\bigg)^{\frac{1}{p}}\\
& \qquad \qquad + \bigg(\int_{\Lambda}\dWp{2}(E(\hat{T}_n(E^{-1}(\mu))),D_n(E(\hat{T}_n(E^{-1}(\mu)))))^p \, \dd\bbP_{\Lambda}(\mu) \bigg)^{\frac{1}{p}}\\
& \leq L \bigg(\int_\Lambda |E^{-1}(\mu)-\hat{T}_n(E^{-1}(\mu))|^p \, \dd\bbP_{\Lambda}(\mu)\bigg)^{\frac{1}{p}}\\
& \qquad \qquad + \bigg(\int_{\cS}\dWp{2}(E(\hat{T}_n(\theta)),D_n(E(\hat{T}_n(\theta))))^p \, \dd\bbP_{\cS}(\theta) \bigg)^{\frac{1}{p}}\\
& = L \bigg(\int_{\mathcal{S}} |\theta - \hat{T}_n(\theta)|^p \, \dd\bbP_\mathcal{S}(\theta) \bigg)^\frac{1}{p} + \bigg(\int_{\mathcal{S}_n}\dWp{2}(E(\theta),D_n(E(\theta)))^p \,  \dd\bbP_{\mathcal{S}_n}(\theta) \bigg)^{\frac{1}{p}}\\
& = L \bigg(\int_{\mathcal{S}} |\theta-\hat{T}_n(\theta)|^p \, \dd\bbP_\mathcal{S}(\theta) \bigg)^\frac{1}{p} + \bigg(\int_{\Lambda_n}\dWp{2}(\mu,D_n(\mu))^p \, \dd\bbP_{\Lambda_n}(\mu)\bigg)^{\frac{1}{p}}.
\end{aligned}    
\end{equation*}
\item 
By a change of variables we have,
\begin{align*}
\|f_n \circ T_n - f \|^p_{\Lp{p}(\bbP_\Lambda)} = &\int_{\Lambda}|f_n \circ T_n(\mu) - f(\mu) |^p \, \dd\bbP_\Lambda(\mu) \\
=& \int_{\Lambda}|f_n \circ D_n \circ E \circ \hat{T}_n \circ E^{-1}(\mu) - f(\mu)|^p \, \dd\bbP_\Lambda(\mu)\\
=& \int_{\mathcal{S}}|f_n \circ D_n \circ E \circ \hat{T}_n(\theta) - f\circ E(\theta)|^p \, \dd\bbP_\mathcal{S}(\theta) \\
=& \int_{\mathcal{S}} |\hat{f}_n\circ \hat{T}_n(\theta) - \hat{f}(\theta) |^p \, \dd\bbP_\mathcal{S}(\theta) \\
=& \|\hat{f}_n \circ \hat{T}_n - \hat{f} \|^p_{\Lp{p}(\bbP_\mathcal{S})}. \qedhere
\end{align*}
\end{enumerate}
\end{proof}

Finally, we prove the main result of the subsection by showing that under the given assumptions, there exists a subsequence of $f_n$ that converges to a limit function $f$ in the topology of weak convergence of measures, with $f$ belonging to the Sobolev space $\Wkp{1}{p}(\Lambda, \dWp{2}, \bbP_\Lambda)$ defined in Subsection \ref{section:Sobolev_Spaces}.

\begin{theorem}[Compactness]\label{thm:compactness}
 Assume $\Omega$, $\Lambda$, $\bbP_\Lambda$, $\bbP_{\Lambda_n}$, $\bbP_{\Lambda^{(m)}_n}$, $\eta$ and $\varepsilon_n$ satisfy 
\ref{ass:domain_ll_one}-\ref{ass:domain_ll_two}, \ref{ass:prob_measures_Lambda}, \ref{ass:Lambda_n}-\ref{ass:Lambda_m_n}, \ref{ass:weight_one}-\ref{ass:weight_four} and \ref{ass:eps_scale}-\ref{ass:eps_scale_2}. Let $f_n: \Lambda^{(m)}_n \to \bbR$ satisfy $\sup_{n \in \bbN} \|f_n\|_{\Lp{p}(\bbP_{\Lambda^{(m)}_n})} < +\infty$ and $\sup_{n \in \bbN} \mathcal{E}_{\varepsilon_n,m_n,n}(f_n;\eta) < + \infty$. Then, with probability one there exists subsequence $\{f_{n_k}\}_{k=1}^\infty$ and $f \in \Wkp{1}{p}(\Lambda,\dWp{2},\bbP_\Lambda)$ such that $f_{n_k} \to f$ in $\TL_{\dWp{2}}(\Probac{(\Omega)})$ as $k \to \infty$.
\end{theorem}

\begin{proof}
Define $\hat{f}_n := f_n \circ D_n \circ E : \mathcal{S}_n \to \mathbb{R}$.  
By Lemma~\ref{lemma:sup_energy_bounded}, there exists a function $\hat{\eta}$ that is continuous, non-increasing, positive at zero, and compactly supported, such that $\hat{\mathcal{E}}_{\tilde{\varepsilon}_n,n}(\hat{f}_n;\hat{\eta}) < +\infty$.
A change of variables gives $\|\hat{f}_n\|_{\Lp{p}(\mathbb{P}_{\mathcal{S}_n})} = \|f_n\|_{\Lp{p}(\mathbb{P}_{\Lambda^{(m)}_n})}$. By \cite[Proposition 3.12]{Nicolas_pointcloud}, there exists a subsequence $\{\hat{f}_{n_k}\}_{k=1}^\infty$ and a limit $\hat{f} \in \Wkp{1}{p}(\mathcal{S})$ such that  
\[
\hat{f}_{n_k} \to \hat{f} \quad \text{in } \TLp{p}(\mathcal{S})
\]  
with probability one.
Equivalently, there exists a sequence of transport maps $\hat{T}_{n_k} : \mathcal{S} \to \mathcal{S}_{n_k}$ with $\hat{T}_{n_k\#} \mathbb{P}_{\mathcal{S}} = \mathbb{P}_{\mathcal{S}_{n_k}}$ and $\hat{T}_{n_k} \to \mathrm{Id}$ in $\Lp{p}(\mathbb{P}_{\mathcal{S}})$ such that
$\hat{f}_{n_k} \circ \hat{T}_{n_k} \to \hat{f} \quad \text{in } \Lp{p}(\mathbb{P}_{\mathcal{S}})$.  

Define $T_{n_k} := D_{n_k} \circ E \circ \hat{T}_{n_k} \circ E^{-1}$. 
Then $T_{n_k} \to \mathrm{Id}$ and, by Proposition~\ref{prop:TLp_distance_Wasserstein} we have,
\[
\|f_{n_k} \circ T_{n_k} - f\|_{\Lp{p}(\mathbb{P}_\Lambda)} 
= \|\hat{f}_{n_k} \circ \hat{T}_{n_k} - \hat{f}\|_{\Lp{p}(\mathbb{P}_{\mathcal{S}})} \to 0,
\]  
with probability one, where $f := \hat{f} \circ E^{-1}$.  
Hence, by Proposition~\ref{prop:TLp_distance} (see Remark~\ref{rem:Back:OT:TLpConv}), $f_{n_k} \to f$ in $\TLp{p}_{\dWp{2}}(\mathbb{P}_\Lambda)$, and by Proposition~\ref{prop:Sobolev_Space_1}, $f \in \Wkp{1}{p}(\Lambda, \dWp{2}, \mathbb{P}_\Lambda)$.
\end{proof}

\subsection{\texorpdfstring{$\Gamma$}{Gamma}- convergence}\label{section:gamma_convergence}

In the first lemma of this section we show that $\frac{\dWp{2}(E(\theta+\Delta h), E(\theta))}{\Delta}\approx \|B_\theta(h)\|_{\Lp{2}(E(\theta))}$.

\begin{lemma}\label{lemma:linear_operator_B}  
Let $\Omega, \Lambda, \mathcal{S}, E$ and $B_\theta(h)$ satisfy assumptions \ref{ass:domain_ll_one}-\ref{ass:domain_ll_three}, and \ref{ass:E_map}-\ref{ass:B_map}. 
Let $M>0$. There exists $C>0$ such that
\[ \la \frac{\dWp{2}(E(\theta+\Delta h), E(\theta))}{\Delta} -\|B_\theta(h)\|_{\Lp{2}(E(\theta))} \ra \leq C\Delta \|h\|^2 \]
for all $h \in B(0,s_\theta\wedge M)$, $\Delta\in (0,1]$ and $\theta\in\cS$.
\end{lemma}

\begin{proof}
Let $d_{\mathrm{W}^{2}_\Lambda}$ be the 2-Wasserstein metric on $\Lambda \subset \Probac(\Omega)$ defined by taking the shortest distance of curves in $\Lambda$ (see \cite[Eq. (2.8)]{Hamm} for more details). For $C > 0$ (independent of $\theta$), we have
\begin{equation*}\label{eqn:linear_operator_equivalence}
\begin{aligned}
\bigg|\frac{\dWp{2}(E(\theta+\Delta h), E(\theta))}{\Delta} - \|B_\theta(h)\|_{\Lp{2}(E(\theta))}\bigg| & \leq \frac{1}{\Delta} \left|\dWp{2}(E(\theta+\Delta h), E(\theta)) - d_{\mathrm{W}^{2}_\Lambda}(E(\theta+\Delta h), E(\theta))\right| \\
& + \left|\frac{d_{\mathrm{W}^{2}_\Lambda}(E(\theta+\Delta h),E(\theta))}{\Delta} -\|B_\theta(h)\|_{\Lp{2}(E(\theta))} \right| \\
& \stackrel{(a.)}{\leq} \frac{1}{\Delta} Cd_{\mathrm{W}^{2}_\Lambda}^2(E(\theta+\Delta h),E(\theta)) + C\Delta\|h\|^2 \\
& \stackrel{(b.)}\leq \frac{1}{\Delta}C\bigg(C\Delta^2\|h\|^2+\Delta\|B_\theta(h)\|_{\Lp{2}(E(\theta))}\bigg)^2 + C\Delta\|h\|^2 \\
& \leq C \lp \Delta \|h\|^2 + \Delta^2 \|h\|^3 + \Delta^3 \|h\|^4 \rp \\
\end{aligned}    
\end{equation*}
where the first term in (a.) follows from \cite[Proposition 2.13]{Hamm} and the second term follows from \cite[Theorem 2.14]{Hamm}. Similarly, the first term in (b) follows from \cite[Theorem 2.14]{Hamm} and the last inequality uses the Assumption \ref{item:B.2_a}.
\end{proof}

In the next lemma we define $\hat{\eta}_\theta(h) = \eta(\|B_\theta(h)\|_{\Lp{2}(E(\theta))})$ and show that $\hat{\eta}_\theta$ satisfies the conditions for the energy $\hat{\cE}_{\eps_n,n}$ defined by~\eqref{eqn:discrete_energy_parameter_intro} to $\Gamma$-converge to $\hat{\cE}_\infty$ defined by~\eqref{eqn:continuum_energy_functional_discrete_intro} (see Theorem~\ref{thm:gamma_convergence_parameter_manifold}). 

\begin{lemma}\label{lemma:assumptions_eta_hat}
Assume $\Omega$, $\Lambda$, $\mathcal{S}$, $E$, $B_\theta(h)$, and $\eta$  satisfy 
\ref{ass:domain_ll_one}-\ref{ass:domain_ll_three}, \ref{ass:E_map}-\ref{ass:B_map}, \ref{ass:weight_one}-\ref{ass:weight_four}. For all $\theta \in \mathcal{S}$, define $\hat{\eta}_\theta(h)=\eta(\|B_\theta(h)\|_{\Lp{2}(E(\theta))})$. Then, $\hat{\eta}_\theta$ satisfies the following conditions: 
\begin{enumerate}
\item \label{item:eta_hat_compact_support} there exists $R>0$ such that $\hat{\eta}_\theta$ has compact support in $B(0,R)$ (for all $\theta\in\cS)$;
 \item \label{item:eta_hat_even_function} $\hat{\eta}_\theta(h)=\hat{\eta}_\theta(-h)$ for all $h\in\bbR^d$; 
 \item \label{item:eta_hat_pointwise_equicontinuos} $\theta\mapsto \hat{\eta}_\theta(h)$ is pointwise equicontinuous; \item \label{item:eta_hat_existence_of_constants} there exist constants $a,r>0$ such that $\hat{\eta}_\theta(h) \geq a$ for $|h|<r$; and
  \item \label{item:eta_hat_continuity} for all $\xi, \psi > 0$ there exist constants $\alpha_{\xi,\psi} >0$, $c_{\xi,\psi} >0$ such that if $\|h_1-h_2\|\leq\xi$ and $\|\theta_1-\theta_2\|\leq\psi$ then $\hat{\eta}_{\theta_1}(h_1)\geq c_{\xi,\psi}\hat{\eta}_{\theta_2}(\alpha_\xi h_2)$ with $c_{\xi,\psi} \to 1$, $\alpha_{\xi,\psi} \to 1$ as $\xi \to 0$ and $\psi \to 0$.
\end{enumerate}
\end{lemma}

\begin{proof} We prove \ref{item:eta_hat_compact_support}-\ref{item:eta_hat_continuity} in turn.
\begin{enumerate}
\item\label{item:compact_support} Let $\tilde{R}>0$ be such that $\eta$ is compactly supported in $[0,\tilde{R}]$. From Assumption \ref{item:B.2_a}, there exist $C>0$ such that $\|B_\theta(h)\|_{\Lp{2}(E(\theta))} \geq \frac{1}{C}\|h\|$. Define $R=C\tilde{R}$. Now, if $\|h\| > R$ then
$\|B_\theta(h)\|_{\Lp{2}(E(\theta))} \geq \tilde{R}$. This implies that $\hat{\eta}_\theta(h)=\eta(\|B_\theta(h)\|_{\Lp{2}(E(\theta))}) = 0$. 

\item For all $h \in \bbR^d$,  by definition 
\[ \hat{\eta}_\theta(h)=\eta(\|B_\theta(h)\|_{\Lp{2}(E(\theta))})=\eta(\|-B_\theta(h)\|_{\Lp{2}(E(\theta))})=\eta(\|B_\theta(-h)\|_{\Lp{2}(E(\theta))})=\hat{\eta}_\theta(-h) \]
is an even function.

\item We choose $\theta \in \mathcal{S}$ and $\omega > 0$. 
Let $t= \|B_\theta(h)\|_{\Lp{2}(E(\theta))}$.
There exists $\delta>0$ such that for any $t' \in \bbR^d$ if $|t-t'| < \delta$, then 
\begin{align}\label{eqn:eta_difference}
 \left| \eta(t) - \eta(t') \right| < \omega.
\end{align}
Next, we set $t' = \|B_{\theta'}(h)\|_{\Lp{2}(E(\theta'))}$.  
By Assumption \ref{item:B.2_d} there exists $\psi > 0$ such that if $|\theta-\theta'|<\psi$, then 
\begin{align}\label{eqn:B_theta_difference}
\sup_{h\in B(0,R)}\left| \|B_\theta(h)\|_{\Lp{2}(E(\theta))} - \|B_{\theta'}(h)\|_{\Lp{2}(E(\theta'))} \right| < \delta.
\end{align}
Hence,
\begin{align}
 \left| \hat{\eta}_\theta(h) - \hat{\eta}_{\theta'}(h) \right|= \left| \eta(\|B_\theta(h)\|_{\Lp{2}(E(\theta))}) - \eta(\|B_{\theta'}(h)\|_{\Lp{2}(E(\theta'))}) \right| < \omega.
\end{align}
Therefore, $\theta \mapsto \hat{\eta}_\theta(h)$ is pointwise equicontinuous for $h\in B(0,R)$. For $\|h\| > R$, $\hat{\eta}_\theta(h)=0$, so we have pointwise equicontinuity for all $h$.

\item By \ref{ass:weight_one}-\ref{ass:weight_three} we assume that $\eta$ is decreasing and continuous with $\eta(0)>0$. There exists $\chi>0$  such that if $\|B_\theta(h)\|_{\Lp{2}(E(\theta))}<\chi$, then  $\eta(\|B_\theta(h)\|_{\Lp{2}(E(\theta))}) \geq \frac{\eta(0)}{2}=:a$. Let $r=\frac{\chi}{C}$.
If $\|h\| \leq r$, then using \ref{item:B.2_a} we have $\|B_\theta(h)\|_{\Lp{2}(E(\theta))} \leq C \|h\| < Cr=\chi$. Hence, if $\|h\|\leq r$, then  $\hat{\eta}_\theta(h) =\eta(\|B_\theta(h)\|_{\Lp{2}(E(\theta))}) \geq a$.

\item Fix $\xi > 0$, $\psi > 0$ and let $\|h_2\| \geq r$ for some $r>0$, which will be chosen appropriately in the sequel. If $\|h_1-h_2\| \leq \xi$ and $\|\theta_1-\theta_2\|\leq \psi$ then,
\begin{align}\label{eqn:proof_eta_bounds_greater_r}
\begin{aligned}
\|B_{\theta_1}(h_1)\|_{\Lp{2}(E(\theta_1))} & \leq \|B_{\theta_1}(h_2)\|_{\Lp{2}(E(\theta_1))}+\|B_{\theta_1}(h_1-h_2)\|_{\Lp{2}(E(\theta_1))} \\
& \leq \|B_{\theta_1}(h_2)\|_{\Lp{2}(E(\theta_1))}+C\|h_1-h_2\| \\
& \leq \|B_{\theta_1}(h_2)\|_{\Lp{2}(E(\theta_1))}+\frac{C\xi \|B_{\theta_1}(h_2)\|_{\Lp{2}(E(\theta_1))}}{\|B_{\theta_1}(h_2)\|_{\Lp{2}(E(\theta_1))}} \\
& \leq \|B_{\theta_1}(h_2)\|_{\Lp{2}(E(\theta_1))}\left(1+\frac{C^2\xi}{\|h_2\|} \right)  \\
& \leq \|B_{\theta_1}(h_2)\|_{\Lp{2}(E(\theta_1))}\underbrace{\left(1+\frac{C^2 \xi}{r}\right)}_{=:\tilde{\alpha}_\xi} \\
& \leq \tilde{\alpha}_\xi\left(\|B_{\theta_2}(h_2)\|_{\Lp{2}(E(\theta_2))}+\left(\|B_{\theta_1}(h_2)\|_{\Lp{2}(E(\theta_1))}-\|B_{\theta_2}(h_2)\|_{\Lp{2}(E(\theta_2))}\right)\right)\\
& = \tilde{\alpha}_\xi\|B_{\theta_2}(h_2)\|_{\Lp{2}(E(\theta_2))}\left(1+\frac{\|B_{\theta_1}(h_2)\|_{\Lp{2}(E(\theta_1))}-\|B_{\theta_2}(h_2)\|_{\Lp{2}(E(\theta_2))}}{\|B_{\theta_2}(h_2)\|_{\Lp{2}(E(\theta_2))}}\right)\\
&\leq \tilde{\alpha}_\xi\|B_{\theta_2}(h_2)\|_{\Lp{2}(E(\theta_2))}\left(1+C\omega_B(\|\theta_1-\theta_2\|)\right) \\
&\leq \underbrace{\tilde{\alpha}_\xi\left(1+C\omega_B(\psi)\right)}_{=:\alpha_{\xi,\psi}}\|B_{\theta_2}(h_2)\|_{\Lp{2}(E(\theta_2))} \\
\end{aligned}    
\end{align}
The first inequality follows from the triangle inequality;  
the second and fourth inequalities use Assumption~\ref{item:B.2_a}; 
the third inequality uses the fact that \(\|h_1 - h_2\| \leq \xi\); 
the fifth inequality relies on \(\|h_2\| \geq r\);  
the sixth inequality defines \(\tilde{\alpha}_\xi := \left(1 + \frac{C^2 \xi}{r}\right)\);  
the seventh inequality uses Assumption~\ref{item:B.2_d} with the modulus of continuity $\omega_B$.

Using \eqref{eqn:proof_eta_bounds_greater_r} and Assumption \ref{ass:weight_one}, implies that for $\|h_2\| \geq r$ we have, for any $c_{\xi,\psi}\in (0,1]$,
\begin{align}\label{eqn:eta_bounds_greater_r}
\hat{\eta}_{\theta_1}(h_1)=\eta(\|B_{\theta_1}(h_1)\|_{\Lp{2}(E(\theta_1))}) \geq c_{\xi,\psi} \eta\bigg(\alpha_{\xi,\psi}\|B_{\theta_2}(h_2)\|_{\Lp{2}(E(\theta_2))}\bigg)= c_{\xi,\psi}\hat{\eta}_{\theta_2}\left(\alpha_{\xi,\psi} h_2\right). 
\end{align}

Next 
we consider the case $\|h_2\|< r$. Let $\omega(\cdot)$ be a modulus of continuity for $\eta$ since it is continuous by Assumption \ref{ass:weight_two}. Then,
\begin{align}\label{eqn:modulus_continuity}
\begin{aligned}
|\hat{\eta}_{\theta_1}(h_1)-\hat{\eta}_{\theta_2}(h_2)| & = |\eta(\|B_{\theta_1}(h_1)\|_{\Lp{2}(E(\theta_1))})-\eta(\|B_{\theta_2}(h_2)\|_{\Lp{2}(E(\theta_2))})| \\
 & \leq \omega\Big(\|B_{\theta_1}(h_1)\|_{\Lp{2}(E(\theta_1))}-\|B_{\theta_1}(h_2)\|_{\Lp{2}(E(\theta_1))}\\
 & \qquad\qquad +\|B_{\theta_1}(h_2)\|_{\Lp{2}(E(\theta_1))}-\|B_{\theta_2}(h_2)\|_{\Lp{2}(E(\theta_2))}\Big)\\
 & \leq \omega\Big(\|B_{\theta_1}(h_1-h_2)\|_{\Lp{2}(E(\theta_1))}+\|h_2\|\omega_B(\|\theta_1-\theta_2\|)\Big)\\
 & \leq \omega\Big(C\|h_1-h_2\|+\|h_2\|\omega_B(\|\theta_1-\theta_2\|)\Big) 
\end{aligned}    
\end{align}
where the second inequality comes from reverse triangle inequality and the final inequality comes from Assumption \ref{item:B.2_a}.

Now, if $\|h_1 - h_2\| \leq \xi$ and $\|\theta_1-\theta_2\| \leq \psi$, then using \eqref{eqn:modulus_continuity} we have
\begin{align}\label{eqn:bounds_eta_continuity}
\begin{aligned}
\hat{\eta}_{\theta_1}(h_1) &\geq \hat{\eta}_{\theta_2}(h_2)-\omega(C\|h_1-h_2\|+\|h_2\|\omega_B(\|\theta_1-\theta_2\|)) \\
& = \hat{\eta}_{\theta_2}(h_2)\bigg(1-\frac{\omega(C\|h_1-h_2\|+\|h_2\|\omega_B(\|\theta_1-\theta_2\|))}{\hat{\eta}_{\theta_2}(h_2)}\bigg)\\
& \geq \hat{\eta}_{\theta_2}(h_2)\bigg(1-\frac{\omega(C\xi+r\omega_B(\psi))}{a}\bigg) \\
\end{aligned}
\end{align}
where the last inequality uses Part \ref{item:eta_hat_existence_of_constants} by choosing $r$ such that if $\|h_2\| < r$ then $\hat{\eta}_{\theta_2}(h_2)> a$. Define $c_{\xi,\psi}=:\bigg(1-\frac{\omega(C\xi+r\omega_B(\psi))}{a}\bigg)\leq 1$. 
Hence, using \eqref{eqn:bounds_eta_continuity} and Assumption \ref{ass:weight_one} we have for $\|h_2\| < r$, 
\begin{align}\label{eqn:eta_bounds_lesser_r}
\hat{\eta}_{\theta_1}(h_1) \geq c_{\xi,\psi} \hat{\eta}_{\theta_2}( h_2)   \geq c_{\xi,\psi} \hat{\eta}_{\theta_2}(\alpha_{\xi,\psi} h_2).   
\end{align}

Finally, using \eqref{eqn:eta_bounds_greater_r} and \eqref{eqn:eta_bounds_lesser_r} we have 
$\hat{\eta}_{\theta_1}(h_1) \geq c_{\xi,\psi} \hat{\eta}_{\theta_2}(\alpha_{\xi,\psi} h_2)$ 
where $c_{\xi,\psi} \to 1$ and $\alpha_{\xi,\psi} \to 1$ as $\xi,\psi \to 0$. \qedhere
\end{enumerate}
\end{proof}

Finally, we now prove the $\Gamma$-convergence of the discrete energy functional defined on $\Lambda^{(m)}_n$ to the continuum energy functional on $\Lambda$ i.e., $\mathcal{E}_\infty=\Glim_{n \to \infty}\mathcal{E}_{\varepsilon_n,m_n,n}$.

\begin{theorem}[$\Gamma$-Convergence]\label{thm:gamma_convergence}
 Assume $\Omega$, $\Lambda$, $\mathcal{S}$, $\bbP_\Lambda$, $\bbP_\mathcal{S}$, $\bbP_{\Lambda_n}$, $\bbP_{\Lambda^{(m)}_n}$, $\bbP_{\mathcal{S}_n}$, $E$, $B_\theta(h)$, $\eta$ and $\varepsilon_n$ satisfy
\ref{ass:domain_ll_one}-\ref{ass:domain_ll_three}, \ref{ass:prob_measures_S}-\ref{ass:prob_measures_Lambda}, \ref{ass:S_n}-\ref{ass:Lambda_m_n}, \ref{ass:E_map}-\ref{ass:B_map}, \ref{ass:weight_one}-\ref{ass:weight_four} and \ref{ass:eps_scale}-\ref{ass:eps_scale_2}. Define $\mathcal{E}_{\varepsilon_n,m_n,n}$ by  \eqref{eqn:discrete_energy_pertubed_intro} and $\cE_\infty$ by \eqref{eqn:continuum_energy_functional_intro}. Then, for $p \geq 1$ and any $M>0$, we have that $\mathcal{E}_{\varepsilon_n,m_n,n}$  $\Gamma-$ converges to $\mathcal{E}_\infty$ as $n \to \infty$ with probability one on the set $\lb (f,\bbP_\Lambda) \in\TLp{p}(\Lambda) \,:\, \|f\|_{\Lp{\infty}(\bbP_\Lambda)}\leq M\rb$.
\end{theorem}

We divide the proofs for $\Gamma$-convergence into lim-sup and lim-inf inequalities below.

\begin{lemma}[lim sup-inequality]\label{lemma:lim-sup} 
 Under the same conditions as Theorem \ref{thm:gamma_convergence} with probability one, for any function $f \in \Lp{p}(\bbP_\Lambda)$ with $\|f\|_{\Lp{\infty}(\bbP_\Lambda)} < + \infty$, there exists a sequence $f_n$ satisfying $f_n \to f$ in $\TL_{\dWp{2}}(\Probac{(\Omega)})$ with $\|f_n\|_{\Lp{\infty}(\bbP_{\Lambda^{(m)}_n})} < + \infty$ and 
\begin{align}\label{eqn:lim_sup_inequality}
 \limsup_{n \to \infty} \mathcal{E}_{\varepsilon_n,m_n,n}(f_n) \leq \mathcal{{E}}_\infty({f})   
\end{align}
\end{lemma}

\begin{proof}
Firstly, let us fix $f \in \Lp{p}(\bbP_\Lambda)$ and we assume $\mathcal{{E}}_\infty({f}) < + \infty$; otherwise the inequality in  \eqref{eqn:lim_sup_inequality} is trivial. Now  $\mathcal{{E}}_\infty({f}) < + \infty$ if and only if, $f \in \Wkp{1}{p}(\Lambda,\dWp{2},\bbP_\Lambda)$ and we define $\hat{f}=f \circ E \in \Wkp{1}{p}(\mathcal{S})$ using Proposition~\ref{prop:Sobolev_Space_1}. Assume that the result holds for $\hat{f} \in \Ck{2}(\cS)$; that is, there exists a sequence $f_n \to f=\hat{f} \circ E^{-1}$ such that \eqref{eqn:lim_sup_inequality} is satisfied. Then, for $\hat{f}\notin \Ck{2}(\cS)$ we can take a sequence $\hat{f}^{(k)}\in\Ck{2}(\cS)$ such that $\hat{f}^{(k)}\to \hat{f}$ in $\Lp{p}(\cS)$ and $f^{(k)}=\hat{f}^{(k)}\circ E^{-1} \to f$ in $\Lp{p}(\Lambda,\dWp{2},\bbP_\Lambda)$.
For each $\hat{f}^{(k)}$ we can find a recovery sequence for $f$ as $k \to \infty$ by a diagonalisation argument. 

We assume that $\hat{f}\in\Ck{2}(\cS)$.
By the previous lemma, the assumptions of Lemma~\ref{lemma:lim-sup_parameter_manifold} are satisfied for $\hat{\cE}_{\eps,n}$ defined by~\eqref{eqn:discrete_energy_parameter_intro} with $\hat{W}_{ij} = \frac{1}{\eps^d_n} \hat{\eta}_{\theta_i}\lp\frac{\theta_j-\theta_i}{\eps_n}\rp$ and $\hat{\eta}_\theta(h) = \eta(\|B_\theta(h)\|_{\Lp{2}(E(\theta))})$.
Hence, $\hat{f}_n:= \hat{f}\lfloor_{\cS_n}\in\Lp{p}(\bbP_{\cS_n})$ satisfies $\hat{f}_n\to \hat{f}$ in $\TLp{p}(\cS)$ and $\limsup\limits_{n\to\infty} \hat{\cE}_{{\eps}_n,n}(\hat{f}_n) \leq \hat{\cE}_\infty(\hat{f}) = \cE_\infty(f)$ for any sequence ${\eps}_n\to 0$ satisfying Assumption \ref{ass:eps_scale}.

Define $D_n$ as in Lemma~\ref{lemma:transport_maps_Dn} and let $f_n = \hat{f}_n \circ E^{-1} \circ D_n^{-1}$. Let $\hat{T}_n$ be a sequence of transport maps satisfying $\hat{T}_{n\#} \bbP_{\cS} = \bbP_{\cS_n}$ and $\|\hat{T}_n -\Id\|_{\Lp{p}(\bbP_\cS)}\to 0$ (see Theorem~\ref{thm:rates_convergence}).
Define $T_n=D_n \circ E \circ \hat{T}_n\circ E^{-1}$. 
By Proposition~\ref{prop:TLp_distance_Wasserstein} we have $f_n\to f$ in  $\TL_{\dWp{2}}(\Probac{(\Omega)})$.

We claim that the inequality in~\eqref{eqn:lim_sup_inequality} holds for this choice of $f_n$. We have,
\begin{align}\label{eqn:energy_functional_bounds_gamma_convergence_lsup}
\begin{aligned}
\mathcal{E}_{\varepsilon_n,m_n,n}(f_n) & = \frac{1}{n^2\varepsilon^{p+d}_n} \sum_{i,j=1}^n \eta\bigg(\frac{\dWp{2}(\mu^{(m_n)}_i,\mu^{(m_n)}_j)}{\varepsilon_n}\bigg) |f_n(\mu^{(m_n)}_i)-f_n(\mu^{(m_n)}_j)|^p  \\
& = \frac{1}{\varepsilon^{p+d}_n} \int_{\Lambda^{(m)}_n} \int_{\Lambda^{(m)}_n} \eta\bigg(\frac{\dWp{2}(\mu,\nu)}{\varepsilon_n} \bigg)|f_n(\mu) - f_n(\nu)|^p \, \dd\bbP_{\Lambda^{(m)}_n}(\mu) \,\dd\bbP_{\Lambda^{(m)}_n}(\nu) \\
& = \frac{1}{\varepsilon^{p+d}_n} \int_{\Lambda^{(m)}_n} \int_{\Lambda^{(m)}_n} \eta\bigg(\frac{\dWp{2}(\mu,\nu)}{\varepsilon_n} \bigg)|\hat{f}_n\circ E^{-1} \circ D^{-1}_n(\mu) -\hat{f}_n\circ E^{-1} \circ D^{-1}_n(\nu)|^p \\
& \qquad \qquad  \dd(D_n \circ E)_\#\bbP_{\mathcal{S}_n}(\mu)\,\dd (D_n \circ E)_\#\bbP_{\mathcal{S}_n}(\nu)\\
& = \frac{1}{\varepsilon^{p+d}_n} \int_{\mathcal{S}_n} \int_{\mathcal{S}_n} \eta\bigg(\frac{\dWp{2}(D_n \circ E(\theta), D_n \circ E(\tau))}{\varepsilon_n} \bigg)|\hat{f}_n(\theta) -\hat{f}_n(\tau)|^p \,\dd\bbP_{\mathcal{S}_n}(\theta)\,\dd\bbP_{\mathcal{S}_n}(\tau).
\end{aligned}    
\end{align}
We complete the proof in three steps.
In the first step we assume that $\eta$ is an indicator function.
In the second step we assume that $\eta$ is piecewise constant.
And in the third step we treat the general case.

\paragraph{Step 1.}

We will show that $\eta\bigg(\frac{\dWp{2}(D_n \circ E(\theta), D_n \circ E(\tau))}{\varepsilon_n}\bigg) \leq \eta\bigg(\frac{\dWp{2}(E(\theta),E(\tau))}{\varepsilon'_n}\bigg) \leq \eta\bigg(\frac{\|B_\theta(\tau-\theta)\|_{\Lp{2}(E(\theta))}}{\tilde{\varepsilon}_n}\bigg)$ for an appropriate choice of  $\tilde{\varepsilon}_n, \varepsilon'_n$ with $\frac{\tilde{\varepsilon}_n}{\varepsilon_n} \to 1$, $\frac{\varepsilon'_n}{\varepsilon_n} \to 1$ as $n \to \infty$.

Using~\eqref{eqn:rates_dwp}
and Theorem~\ref{thm:rates_convergence} there exists $\hat{C}>0$
\begin{align}\label{eqn:wasserstein_distance_bounds2}
\begin{aligned}
 \dWp{2}(E(\theta),E(\tau)) & \leq \dWp{2}(D_n \circ E(\theta), D_n \circ E(\tau)) +\hat{C}\tilde{q}_k(m_n)
\end{aligned}
\end{align}
for all $m_n$ sufficiently large with probability one. 

Next, for a suitable choice of $\varepsilon'_n$ (appropriately selected later), we have by Lemma~\ref{lemma:linear_operator_B},
\begin{align}\label{eqn:wasserstein_distance_bounds3}
\la \frac{\dWp{2}(E(\tau),E(\theta))}{\eps_n^\prime} - \lda B_\theta\lp\frac{\tau-\theta}{\eps_n^\prime}\rp \rda_{\Lp{2}(E(\theta))} \ra \leq \frac{C\|\tau-\theta\|^2}{\eps_n^\prime} \leq \frac{CL^2 \dWp{2}^2(E(\tau),E(\theta))}{\eps_n^\prime}.
\end{align}

Let us consider the case when $\eta$ is defined as follows,
\begin{align}\label{eqn:eta_indicator_function}
\begin{aligned}
\quad \eta(t) &:=
\left\{\begin{array}{rcll}
a &\quad\text{if } t \leq r,  \\
0 &\quad\text{else } .  \\[1.2ex]
\end{array}\right.\\
\end{aligned} 
\end{align}
for any two positive constants $a,r$. 

To show $\eta\bigg(\frac{\dWp{2}(D_n \circ E(\theta), D_n \circ E(\tau))}{\varepsilon_n} \bigg) \leq \eta\bigg(\frac{\dWp{2}(E(\theta),E(\tau))}{\varepsilon'_n}\bigg)$, it is sufficient to show that
\begin{align}\label{eqn:eta_equality}
\eta\bigg(\frac{\dWp{2}(D_n \circ E(\theta), D_n \circ E(\tau))}{\varepsilon_n}\bigg)=a \implies \eta\bigg(\frac{\dWp{2}(E(\theta),E(\tau))}{\varepsilon'_n}\bigg)=a.
\end{align}
which is equivalent to $\frac{\dWp{2}(D_n \circ E(\theta), D_n \circ E(\tau))}{\varepsilon_n} \leq r$ implies $\frac{\dWp{2}(E(\theta_i),E(\theta_j))}{\varepsilon'_n} \leq r$. 

Assume that $\frac{\dWp{2}(D_n \circ E(\theta), D_n \circ E(\tau))}{\varepsilon_n} \leq r$. Now, by \eqref{eqn:wasserstein_distance_bounds2} for all $m_n$ sufficiently large 
\begin{align}
\begin{aligned}
\dWp{2}(E(\theta),E(\tau)) & \leq \dWp{2}(D_n \circ E(\theta), D_n \circ E(\tau)))+\hat{C}\tilde{q}_k(m_n)\\
& \leq r\varepsilon_n+\hat{C}\tilde{q}_k(m_n) \\
& = r \varepsilon'_n\\
\end{aligned}    
\end{align}
by choosing  $\varepsilon'_n=\varepsilon_n+\frac{\hat{C}\tilde{q}_k(m_n)}{r}$.
Hence, 
\begin{align}\label{eqn:w2_n_w2}
\eta\bigg(\frac{\dWp{2}(D_n \circ E(\theta), D_n \circ E(\tau))}{\varepsilon_n}\bigg) \leq \eta\bigg(\frac{\dWp{2}(E(\theta),E(\tau))}{\varepsilon'_n}\bigg)    
\end{align}

Similarly, to show $\eta\bigg(\frac{\dWp{2}(E(\theta),E(\tau))}{\varepsilon'_n}\bigg) \leq  \eta\bigg(\frac{\|B_{\theta}({\tau-\theta})\|_{\Lp{2}(E(\theta))}}{\tilde{\varepsilon}_n}\bigg)$, it is sufficient to show that
\begin{align}\label{eqn:eta_prime_equality}
\eta\bigg(\frac{\dWp{2}(E(\theta),E(\tau))}{\varepsilon'_n}\bigg)=a \implies \eta\bigg(\frac{\|B_{\theta}({\tau-\theta})\|_{\Lp{2}(E(\theta))}}{\tilde{\varepsilon}_n}\bigg)=a.
\end{align}
which is equivalent to $\frac{\dWp{2}(E(\theta),E(\tau))}{\varepsilon'_n} \leq r$ implies $\frac{\|B_{\theta}({\tau-\theta})\|_{\Lp{2}(E(\theta))}}{\tilde{\varepsilon}_n} \leq r$. 

Assume that $\frac{\dWp{2}(E(\theta),E(\tau))}{\varepsilon'_n} \leq r$. Now, using \eqref{eqn:wasserstein_distance_bounds3},
\begin{align}\label{eqn:linear_operator_1}
 \begin{aligned}
\|B_{\theta}(\tau-\theta)\|_{\Lp{2}(E(\theta))} & \leq \dWp{2}(E(\theta),E(\tau)) + CL^2 \dWp{2}^2(E(\theta),E(\tau)) \\
& \leq r\varepsilon'_n + r^2CL^2\varepsilon'^2_n \\
& =r\tilde{\varepsilon}_n\\
 \end{aligned}   
\end{align}
by choosing $\tilde{\varepsilon}_n=\varepsilon'_n+rCL^2\varepsilon'^2_n$.
Hence, 
\begin{align}\label{eqn:w2_b_theta}
\eta\bigg(\frac{\dWp{2}(E(\theta),E(\tau))}{\varepsilon'_n}\bigg) \leq \eta\bigg(\frac{\|B_{\theta}({\tau-\theta})\|_{\Lp{2}(E(\theta))}}{\tilde{\varepsilon}_n}\bigg).
\end{align}

Substituting the bound in~\eqref{eqn:energy_functional_bounds_gamma_convergence_lsup}, and using \eqref{eqn:w2_n_w2}, \eqref{eqn:w2_b_theta} with $\left(\frac{\tilde{\varepsilon}_n}{\varepsilon'_n} \right)\to 1$, $\left(\frac{\varepsilon'_n}{\varepsilon_n} \right) \to 1$ as $n \to \infty$ we have
\begin{equation*}
\begin{aligned}
 \limsup_{n \to \infty}\mathcal{E}_{\varepsilon_n,m_n,n}(f_n) 
& 
 \leq \limsup_{n \to \infty}\frac{1}{\varepsilon^{p+d}_n} \int_{\mathcal{S}_n \times \mathcal{S}_n}\eta\bigg(\frac{\|B_{\theta}(\tau-\theta)\|_{\Lp{2}(E(\theta))}}{\tilde{\varepsilon}_n}\bigg)|\hat{f}_n(\theta) - \hat{f}_n(\tau)|^p \,\dd\bbP_{\mathcal{S}_n}(\theta)\,\dd\bbP_{\mathcal{S}_n}(\tau)   \\
 & =  \limsup_{n \to \infty}\frac{1}{\varepsilon^{p+d}_n} \int_{\mathcal{S}_n \times\mathcal{S}_n}\hat{\eta}_\theta\bigg(\frac{\tau-\theta}{\tilde{\varepsilon}_n}\bigg)|\hat{f}_n(\theta) - \hat{f}_n(\tau)|^p \,\dd\bbP_{\mathcal{S}_n}(\theta)\,\dd\bbP_{\mathcal{S}_n}(\tau)   \\
 & = \limsup_{n\to\infty} \lp\frac{\tilde{\eps}_n}{\eps_n}\rp^{p+d} \hat{\cE}_{\tilde{\eps}_n,n} (\hat{f}_n) \\
 & \leq \cE_\infty(f).
\end{aligned}    
\end{equation*}
where 
the final inequality follows from Lemma \ref{lemma:lim-sup_parameter_manifold}.

\paragraph{Step 2.}
Let us now consider the case when $\eta$ is piecewise constant, decreasing and with compact support.  
In this case, we let $\eta=\sum_{k=1}^l\eta_k$ for some $l$ and functions $\eta_k$ as defined in \eqref{eqn:eta_indicator_function}. 
For clarity let us denote the dependence of $\eta$ on $\cE_{\eps,m,n}$ and $\cE_\infty$ by $\cE_{\eps,m,n}(\cdot;\eta)$ and $\cE_\infty(\cdot;\eta)$.
Note, $\cE_{\eps,m,n}(\cdot;\eta) = \sum_{k=1}^l \cE_{\eps,m,n}(\cdot;\eta_k)$ and $\cE_\infty(\cdot;\eta) = \sum_{k=1}^l \cE_\infty(\cdot;\eta_k)$.

Therefore,
\begin{equation*}
\begin{aligned}
\limsup_{n\to \infty}\mathcal{E}_{\varepsilon_n,m_n,n}(f_n;\eta)& =\limsup_{n \to \infty} \sum_{k=1}^l \mathcal{E}_{\varepsilon_n,m,n}(f_n;\eta_k) \leq \sum_{k=1}^l \limsup_{n \to \infty} \mathcal{E}_{\varepsilon_n,m,n}(f_n;\eta_k) \\
 & \leq \sum_{k=1}^l\mathcal{{E}}_\infty({f};\eta_k) =\mathcal{E}_\infty(f).
\end{aligned} 
\end{equation*}
where the first inequality holds by subadditivity of $\limsup$.

\paragraph{Step 3.}
We assume that $\eta$ satisfies the assumptions in the statement of the theorem. We can find a decreasing sequence of piecewise constant functions $\eta_k:[0,\infty) \to [0,\infty)$ with $\eta_k \searrow \eta$ as $k \to \infty$ a.e. 
By the previous step $\limsup\limits_{n\to\infty} \cE_{\eps_n,m_n,n}(f_n;\eta) \leq \limsup\limits_{n\to\infty} \cE_{\eps_n,m_n,n}(f_n;\eta_k) \leq \cE_\infty(f;\eta_k)$.
By the monotone convergence theorem $\cE_\infty(f;\eta_k)\to \cE_\infty(f;\eta)$ as $k\to\infty$.
\end{proof}

\begin{lemma}[lim inf-inequality]\label{lemma:lim-inf} 
 Under the same conditions as Theorem \ref{thm:gamma_convergence} with probability one, for any function $f \in \Lp{p}(\bbP_\Lambda)$ with $\|f\|_{\Lp{\infty}(\bbP_\Lambda)} < + \infty$, and any sequences $f_n$ satisfying $f_n \to f$ in $\TL_{\dWp{2}}(\Probac{(\Omega)})$ with $\|f_n\|_{\Lp{\infty}(\bbP_{\Lambda^{(m)}_n})} < + \infty$ we have, 
\begin{align}\label{eqn:lim_inf_inequality}
 \liminf_{n \to \infty} \mathcal{E}_{\varepsilon_n,m_n,n}(f_n) \geq \mathcal{{E}}_\infty({f})   
\end{align}
\end{lemma}
\begin{proof}

We assume that $f_n \to f \in \TL_{\dWp{2}}(\Probac{(\Omega)})$. 

We complete the proof in three steps. In the first step we assume that $\eta$ is an indicator function. In the second step we assume that $\eta$ is piecewise constant. And in the third step we treat the general case.
 
\paragraph{Step 1.} 
We will show that $\eta\bigg(\frac{\dWp{2}(D_n \circ E(\theta), D_n \circ E(\tau))}{\varepsilon_n}\bigg) \geq \eta\bigg(\frac{\dWp{2}(E(\theta),E(\tau))}{\varepsilon'_n}\bigg) \geq \eta\bigg(\frac{\|B_\theta(\tau-\theta)\|_{\Lp{2}(E(\theta))}}{\tilde{\varepsilon}_n}\bigg)$ for an appropriate choice of $\varepsilon'_n,\tilde{\varepsilon}_n$ such that $\frac{\varepsilon'_n}{\varepsilon_n} \to 1$, $\frac{\tilde{\varepsilon}_n}{\varepsilon'_n} \to 1$ as $n \to \infty$.

Using~\eqref{eqn:rates_dwp} 
and Theorem~\ref{thm:rates_convergence} there exists $\hat{C}>0$
\begin{align}\label{eqn:wasserstein_distance_bounds4}
\begin{aligned}
\dWp{2}(D_n \circ E(\theta), D_n \circ E(\tau))  & \leq  \dWp{2}(E(\theta),E(\tau))+\hat{C}\tilde{q}_k(m_n)
\end{aligned}
\end{align}
for all $m_n$ sufficiently large with probability one. 

Next, for a suitable choice of $\varepsilon'_n$ (appropriately selected later), we have by Lemma~\ref{lemma:linear_operator_B},
\begin{align}\label{eqn:wasserstein_distance_bounds5}   
\la \frac{\dWp{2}(E(\tau),E(\theta))}{\eps_n^\prime} - \lda B_\theta\lp\frac{\tau-\theta}{\eps_n^\prime}\rp \rda_{\Lp{2}(E(\theta))} \ra \leq \frac{C\|\tau-\theta\|^2}{\eps_n^\prime} \leq \frac{C^3 \|B_\theta(\tau-\theta)\|^2_{\Lp{2}(E(\theta))}}{\eps_n^\prime}.
\end{align}

Let us first consider the case when $\eta$ is defined as in \eqref{eqn:eta_indicator_function}. Now, to show $\eta\bigg(\frac{\dWp{2}(D_n \circ E(\theta), D_n \circ E(\tau))}{\varepsilon_n}\bigg) \geq \eta\bigg(\frac{\dWp{2}(E(\theta),E(\tau))}{\varepsilon'_n}\bigg)$,
it is sufficient to show that
\begin{equation*}\label{eqn:eta_hat_eta}
   {\eta}\bigg(\frac{\dWp{2}(E(\theta),E(\tau))}{\varepsilon'_n}\bigg)=a \implies  \eta\bigg(\frac{\dWp{2}(D_n \circ E(\theta), D_n \circ E(\tau))}{\varepsilon_n}\bigg) = a 
\end{equation*}
which is equivalent to $\frac{\dWp{2}(E(\theta),E(\tau))}{\varepsilon'_n} \leq r$ implies $\frac{\dWp{2}(D_n \circ E(\theta), D_n \circ E(\tau))}{\varepsilon_n} \leq r$.

Assume that $\frac{\dWp{2}(E(\theta),E(\tau))}{\varepsilon'_n} \leq r$. Now, by \eqref{eqn:wasserstein_distance_bounds4} for $m_n$ sufficiently large
\begin{equation*}
\begin{aligned}
\dWp{2}(D_n\circ E(\theta),D_n\circ E(\tau)) & \leq \dWp{2}(E(\theta),E(\tau))+\hat{C}\tilde{q}_k(m_n) \\    
& \leq r\varepsilon'_n + \hat{C}\tilde{q}_k(m_n)\\
&  = r{\varepsilon}_n \\
\end{aligned}    
\end{equation*}
by choosing $ \varepsilon'_n=\varepsilon_n-\frac{\hat{C}\tilde{q}_k(m_n)}{r}$. Hence,
\begin{align}\label{eqn:eta_m_n_eta_liminf}
\eta\bigg(\frac{\dWp{2}(D_n \circ E(\theta), D_n \circ E(\tau))}{\varepsilon_n}\bigg) \geq \eta\bigg(\frac{\dWp{2}(E(\theta),E(\tau))}{\varepsilon'_n}\bigg)    
\end{align}

Similarly, to show $ \eta\bigg(\frac{\|B_\theta(\tau-\theta)\|_{\Lp{2}(E(\theta))}}{\tilde{\varepsilon}_n}\bigg) \leq \eta\bigg(\frac{\dWp{2}(E(\theta),E(\tau))}{\varepsilon'_n}\bigg)$ it is sufficient to show that
\begin{align}\label{eqn:eta_liminf_equality}
 \eta\bigg(\frac{\|B_\theta(\tau-\theta)\|_{\Lp{2}(E(\theta))}}{\tilde{\varepsilon}_n}\bigg)  = a \implies \eta\bigg(\frac{\dWp{2}(E(\theta),E(\tau))}{\varepsilon'_n}\bigg) = a   
\end{align}
which is equivalent to $\frac{\|B_\theta(\tau-\theta)\|_{\Lp{2}(E(\theta))}}{\tilde{\varepsilon}_n} \leq r$ implies $\frac{\dWp{2}(E(\theta),E(\tau))}{\varepsilon'_n}\leq r$.

Assume that $\frac{\|B_\theta(\tau-\theta)\|_{\Lp{2}(E(\theta))}}{\tilde{\varepsilon}_n} \leq r$. Now, using  \eqref{eqn:wasserstein_distance_bounds5}, 
\begin{equation*}
\begin{aligned}
\dWp{2}(E(\theta),E(\tau)) & \leq \|B_\theta(\tau-\theta)\|_{\Lp{2}(E(\theta))}+ C^3\|B_\theta(\tau-\theta)\|^2_{\Lp{2}(E(\theta))}\\
& \leq r\tilde{\varepsilon}_n + r^2C^3\tilde{\varepsilon}^{2}_n \\
& = r\varepsilon'_n\\
\end{aligned}    
\end{equation*}
by choosing $\tilde{\varepsilon}_n = \frac{-1 + \sqrt{1 + 4 C^3 r \varepsilon'_n}}{2 C^3 r}$ ($>0$). Hence,
\begin{align}\label{eqn:eta_B_theta_liminf}
\eta\bigg(\frac{\dWp{2}(E(\theta),E(\tau))}{\varepsilon'_n}\bigg) \geq \eta\bigg(\frac{\|B_\theta(\tau-\theta)\|_{\Lp{2}(E(\theta))}}{\tilde{\varepsilon}_n}\bigg)    
\end{align}

Continuing the bound from \eqref{eqn:energy_functional_bounds_gamma_convergence_lsup} using \eqref{eqn:eta_m_n_eta_liminf}, \eqref{eqn:eta_B_theta_liminf} with $\frac{\varepsilon'_n}{\varepsilon_n} \to 1$, $\frac{\tilde{\varepsilon}_n}{\varepsilon'_n} \to 1$ as $n \to \infty$ we have 
\begin{equation*}
\begin{aligned}
\liminf_{n \to \infty} \mathcal{E}_{\varepsilon_n,m_n,n}(f_n)    \\ 
& \geq \liminf_{n \to \infty} \frac{1}{\varepsilon^{p+d}_n} \int_{\mathcal{S}_n} \int_{\mathcal{S}_n} \eta\bigg(\frac{\|B_\theta(\tau-\theta)\|_{\Lp{2}(E(\theta))}}{\tilde{\varepsilon}_n}\bigg) |\hat{f}_n(\theta) -\hat{f}_n(\tau)|^p\, \dd\bbP_{\mathcal{S}_n}(\theta)\,\dd\bbP_{\mathcal{S}_n}(\tau)   \\ 
& =  \liminf_{n \to \infty} \frac{1}{\varepsilon^{p+d}_n} \int_{\mathcal{S}_n} \int_{\mathcal{S}_n} \hat{\eta}_\theta\bigg(\frac{\tau-\theta}{\tilde{\varepsilon}_n}\bigg)|\hat{f}_n(\theta) -\hat{f}_n(\tau)|^p \,\dd\bbP_{\mathcal{S}_n}(\theta)\,\dd\bbP_{\mathcal{S}_n}(\tau)\\
& = \liminf_{n \to \infty} \bigg(\frac{\tilde{\varepsilon}_n} {\varepsilon'_n}\bigg)^{p+d} \hat{\mathcal{E}}_{\tilde{\varepsilon}_n,n}(\hat{f}_n) \\
& \geq \hat{\mathcal{E}}_\infty(\hat{f})=\mathcal{E}_\infty(f).\\
\end{aligned}    
\end{equation*}
where the last inequality follows from Lemma \ref{lemma:lim-inf_parameter_manifold} which we can apply by Lemma
\ref{lemma:assumptions_eta_hat}.

\paragraph{Step 2.}
Let us consider the case when $\eta$ is piecewise constant, decreasing and with compact support. In this case, we  let $\eta=\sum_{k=1}^l\eta_k$ for some $l$ and functions $\eta_k$ as defined in \eqref{eqn:eta_indicator_function}. For clarity let us denote the dependence of $\eta$ on $\cE_{\eps,m,n}$ and $\cE_\infty$ by $\cE_{\eps,m,n}(\cdot;\eta)$ and $\cE_\infty(\cdot;\eta)$.
Note, $\cE_{\eps,m,n}(\cdot;\eta) = \sum_{k=1}^l \cE_{\eps,m,n}(\cdot;\eta_k)$ and $\cE_\infty(\cdot;\eta) = \sum_{k=1}^l \cE_\infty(\cdot;\eta_k)$. Therefore,
\begin{equation*}
\begin{aligned}
\liminf_{n\to \infty}\mathcal{E}_{\varepsilon_n,m_n,n}(f_n;\eta)& =\liminf_{n \to \infty} \sum_{k=1}^l \mathcal{E}_{\varepsilon_n,m,n}(f_n;\eta_k) \geq \sum_{k=1}^l \liminf_{n \to \infty} \mathcal{E}_{\varepsilon_n,m,n}(f_n;\eta_k) \\
& \geq \sum_{k=1}^l\mathcal{{E}}_\infty({f};\eta_k) = \mathcal{E}_\infty(f). 
\end{aligned}    
\end{equation*}
\paragraph{Step 3.} 
We assume that $\eta$ satisfies the assumptions in the statement of the theorem. We can find an increasing sequence of piecewise constant function $\eta_k:[0,\infty) \to [0,\infty)$ with $\eta_k \nearrow \eta$ as $k \to \infty$ a.e. By the previous step $\liminf\limits_{n \to \infty} \mathcal{E}_{\varepsilon_n,m_n,n}(f_n) \geq  \liminf\limits_{n \to \infty} \mathcal{E}_{\varepsilon_n,m,n}(f_n;\eta_k) \geq \mathcal{E}_\infty(f;\eta_k)$. By the monotone convergence theorem $\mathcal{E}_\infty(f;\eta_k) \to \mathcal{E}_\infty(f;\eta)$ as $k \to \infty$. 
\end{proof}

\subsection{Laplace--Beltrami Operators on Wasserstein Submanifolds}\label{section:graph_laplacians}

Graph Laplacians offer a rigorous framework for capturing both local and global geometric structure in data and form the basis of numerous learning methods \cite{Ronald_2006,Hein_2007}. Their connections to diffusion dynamics and density normalization via degree-based and partial Laplacian operators are well established \cite{Ronald_2006, Chow_2019, Nadler_2006}. Convergence analyses in the large-sample limit include results on pointwise consistency \cite{Amir_2006, Hein_2007, calder2022lipschitz} and spectral convergence~\cite{Matthias_2005, Belkin_2004}.  

The Laplace Learning problem in \eqref{eqn:discrete_energy_intro} can be posed variationally or via the discrete Euler–Lagrange equation 
$\mathcal{L}_{\varepsilon_n, m_n, n} f_n = 0$,
where $\mathcal{L}_{\varepsilon_n, m_n, n}$ is defined below in \eqref{eqn:discrete_graph_laplacian}. Specifically, for $p > 1$, the discrete graph Laplacian operator $\mathcal{L}_{\varepsilon_n,m_n,n}:\Lp{2}(\bbP_{\Lambda^{(m)}_n}) \to \Lp{2}(\bbP_{\Lambda^{(m)}_n})$ is given by, 
\begin{align}\label{eqn:discrete_graph_laplacian}
\mathcal{L}_{\varepsilon_n,m_n,n}f_n(\mu_i):=\frac{p}{2n\varepsilon^p_n}\sum_{j=1}^n W^{(m_n)}_{ij}\la f_n(\mu^{(m_n)}_i-f_n(\mu^{(m_n)}_j)\ra^{p-2}\lp f_n(\mu^{(m_n)}_i)-f_n(\mu^{(m_n)}_j)\rp.
\end{align}

Similarly, for $p > 1$, the Laplace-Beltrami operator $\Delta_{\Lambda}: \Lp{2}(\bbP_\Lambda)\to \Lp{2}(\bbP_\Lambda)$ can be defined as the operator associated with the first variation of $\mathcal{E}_\infty$ which is explicitly stated in the proposition below. 

\begin{proposition}\label{prop:LBO} 
Assume $\Omega$, $\Lambda$, $\mathcal{S}$, $\bbP_\Lambda$, $\bbP_\mathcal{S}$, $E$, $B_\theta(h)$, and $\eta$ satisfy 
\ref{ass:domain_ll_one}-\ref{ass:domain_ll_three}, \ref{ass:prob_measures_S}-\ref{ass:prob_measures_Lambda}, \ref{ass:E_map}-\ref{ass:B_map}, \ref{ass:weight_one}-\ref{ass:weight_four}. For $p \geq 2$ define $\cE_\infty$ by~\eqref{eqn:continuum_energy_functional_intro}.
For $\mu\in\Lambda$ and $f\in\Wkp{1}{p}(\Lambda,\dWp{2},\bbP_{\Lambda})$ we define 
\begin{align*}
\Delta_{\Lambda} f(\mu) &= 
- \frac{p}{\rho_\mathcal{S}(E^{-1}(\mu))} 
\int_{\mathbb{R}^d} 
h \cdot \nabla_\theta \Bigl(
\eta\bigl(\|B_\theta(h)\|_{\Lp{2}(E(\theta))}\bigr) 
\rho_\mathcal{S}^2(\theta) \\
&\qquad \times | \nabla(f \circ E)(\theta) \cdot h |^{p-2} 
\nabla(f \circ E)(\theta) \cdot h
\Bigr) \, \dd h \biggr\vert_{\theta = E^{-1}(\mu)}.
\end{align*}
Then $\partial\cE_\infty(f;g) = \langle\Delta_\Lambda f, g\rangle_{\Lp{2}(\bbP_\Lambda)}$ for all $f,g\in\Wkp{1}{p}(\Lambda,\dWp{2},\bbP_{\Lambda})$ with $g=0$ on $\partial\Lambda$.

\end{proposition}

\begin{proof} 
We let $\theta = (\theta_i, \theta_{-i}) \in \mathcal{S} \subset \mathbb{R}^d$, where 
$\theta_i \in \mathbb{R}$ is the $i$-th coordinate, and $\theta_{-i} \in \mathbb{R}^{d-1}$ are the remaining coordinates.
With an abuse, but simplification, of notation we ignore the order and write $(\theta_i,\theta_{-i})$.
Define 
\[ \mathcal{S}_{-i} := \big\{ \theta_{-i} \in \mathbb{R}^{d-1} \text{ such that there exists } \theta_i \text{ with } \theta=(\theta_i,\theta_{-i}) \in \mathcal{S} \big\}, \]
and 
\[ \mathcal{S}_i := \big\{ \theta_i \in \bbR \text{ such that there exists } \theta_{-i} \text{ with } \theta=(\theta_i,\theta_{-i}) \in \mathcal{S} \big\}. \]
For $f \in \Wkp{1}{p}(\Lambda,\dWp{2},\bbP_\Lambda)$ and $g \in \Wkp{1}{p}(\Lambda,\dWp{2},\bbP_\Lambda)$ when $g = 0$ in $\partial \Lambda$, we can compute 

\begin{align}
\partial \mathcal{E}_\infty(f;g) 
&= \lim_{\delta \to 0} \frac{1}{\delta} \big(\mathcal{E}_\infty(f + \delta g) - \mathcal{E}_\infty(f)\big) \notag \\
&= \lim_{\delta \to 0} \frac{1}{\delta} \int_{\mathbb{R}^d} \int_{\mathcal{S}} 
\lp
\big| \big(\nabla (f \circ E)(\theta) + \delta \nabla (g \circ E)(\theta) \big) \cdot h \big|^p 
- |\nabla (f \circ E)(\theta) \cdot h|^p \rp \notag \\
 & \qquad \qquad \times \eta\big(\|B_\theta(h)\|_{\Lp{2}(E(\theta))}\big)
 \rho^2_\mathcal{S}(\theta) \, \dd\theta \, \dd h \notag \\
&= p \int_{\mathbb{R}^d} \int_{\mathcal{S}} 
|\nabla (f \circ E)(\theta) \cdot h|^{p-2} (\nabla (f \circ E)(\theta) \cdot h) 
(h \cdot \nabla (g \circ E)(\theta)) \notag \\
& \qquad\quad \times \eta\big(\|B_\theta(h)\|_{\Lp{2}(E(\theta))}\big) \rho^2_\mathcal{S}(\theta) \, \dd\theta \, \dd h \label{eq:proofs:LB:Quad} \\
&= p \sum_{i=1}^d \int_{\mathbb{R}^d} \int_{\mathcal{S}} 
h_i \frac{\partial}{\partial \theta_i} (g \circ E)(\theta) 
|\nabla (f \circ E)(\theta) \cdot h|^{p-2} (\nabla (f \circ E)(\theta) \cdot h) \notag \\
& \qquad\quad \times \eta\big(\|B_\theta(h)\|_{\Lp{2}(E(\theta))}\big) \rho^2_\mathcal{S}(\theta) \, \dd\theta \, \dd h \notag \\
&= -p \sum_{i=1}^d \int_{\mathbb{R}^d} \int_{\mathcal{S}_{-i}} 
\bigg( \int_{\mathcal{S}_i} h_i (g \circ E)(\theta) \frac{\partial}{\partial \theta_i} \Big[
|\nabla (f \circ E)(\theta) \cdot h|^{p-2} (\nabla (f \circ E)(\theta) \cdot h) \notag \\
& \qquad\quad \times \eta\big(\|B_\theta(h)\|_{\Lp{2}(E(\theta))}\big) \rho^2_\mathcal{S}(\theta) \Big] \,\dd\theta_i \bigg) \,\dd\theta_{-i} \, \dd h \notag \\
&= -p \sum_{i=1}^d \int_{\mathbb{R}^d} \int_{\mathcal{S}} 
(g \circ E)(\theta) \frac{h_i}{\rho_\mathcal{S}(\theta)} 
\frac{\partial}{\partial \theta_i} \Big(
|\nabla (f \circ E)(\theta) \cdot h|^{p-2} \nabla (f \circ E)(\theta) \cdot h \notag \\
& \qquad\quad \times \eta\big(\|B_\theta(h)\|_{\Lp{2}(E(\theta))}\big) \rho^2_\mathcal{S}(\theta) \Big) 
\rho_\mathcal{S}(\theta) \, \dd\theta \, \dd h \notag \\
 & = \langle \Delta_\Lambda f,g\rangle_{\Lp{2}(\bbP_\Lambda)}. \notag \qedhere
\end{align}
\end{proof}

\begin{remark}
Both the graph Laplacian $\mathcal{L}_{\varepsilon_n,m_n,n}$ and the Laplace--Beltrami operator $\Delta_\Lambda$ can be characterized as the first variations of their respective Dirichlet energies. 
In particular, Proposition~\ref{prop:LBO} shows that the directional derivative of $\mathcal{E}_\infty$ at $f$ in the direction $g$ satisfies
\[
\partial \mathcal{E}_\infty(f; g) = \langle \Delta_\Lambda f, g \rangle_{\Lp{2}(\bbP_\Lambda)}.
\]
Consequently, by setting $g=f$, we obtain the following identity  (see also ~\eqref{eq:proofs:LB:Quad}),
\[
\mathcal{E}_\infty(f) = \frac{1}{p}\, \langle \Delta_\Lambda f, f \rangle_{\Lp{2}(\bbP_\Lambda)}.
\]
An analogous relation holds in the discrete setting,
\[
\mathcal{E}_{\varepsilon_n,m_n,n}(f_n) = \frac{1}{p}\, 
\langle \mathcal{L}_{\varepsilon_n,m_n,n} f_n, f_n \rangle_{\Lp{2}(\bbP_{\Lambda^{(m)}_n})}.
\]
This provides a variational interpretation of the discrete operator. It also suggests that, formally, 
$\Delta_\Lambda$ is the continuum limit of $\mathcal{L}_{\varepsilon_n,m_n,n}$.
\end{remark}

\section{Numerical Experiments}\label{section:Numerical_Experiments}
In this section, we evaluate the performance of Laplace Learning based on the fraction of training labels. These numerical experiments were conducted on two datasets: (a) a 2D dataset of synthetic Gaussian samples (See Figure~\ref{fig:synthetic_gaussian}) and (b) a 3D dataset (images of different furnitures) from ModelNet10~\cite{Wu_2015} (See Figure~\ref{fig:3D_image}). 

\begin{enumerate}\label{algorithm}
\item \textbf{Data Representation:}  
Each sample in the dataset is represented as an empirical probability measure. Specifically, for a fixed number of points \( m \), the \( i \)-th sample is represented as,  
\[
\mu_i^{(m)} = \frac{1}{m} \sum_{j=1}^{m} \delta_{x_j^{(i)}}.
\]

\item \textbf{Construction of the Weight Matrix:}  

We define the weights based on the distance between the nodes $\{\mu_i^{(m)}\}_{i=1}^n$.
For computationally efficiency we use the linear Wasserstein distance (see ~\cite{Wang_2012}, ~\cite[Section 3.4]{Oliver_2024} for more details) rather than the 2-Wasserstein distance. This makes the following approximation
\begin{align}\label{eqn:linear_wasserstein_distance}
\dWp{2}(\mu_i^{(m)},\mu^{(m)}_j) \approx \dWlinpmu{2}{\mu_1^{(m)}}(\mu_i^{(m)},\mu_j^{(m)}) = \| T_i^{(m)} - T_j^{(m)}\|_{\Lp{2}(\mu_1^{(m)})} \end{align}
where $T_i^{(m)}$ is the optimal transport map from $\mu_1^{(m)}$ to $\mu_i^{(m)}$.

In the first experiment, we employ an $\varepsilon$-connected graph, while in the second experiment, we make use of a $k$-nearest neighbor (kNN) graph. Further details are provided in the corresponding sections.

\item \textbf{Solve the Variational Problem:} For a classification problem with $c$ classes, the label $y_i$
is represented as a one-hot vector in $\bbR^c$, where the $l$-th element is $1$ if the samples belong to the class $l$, and the other elements are all zeros. 
Given the weights $(W_{ij}^{(m)})_{i,j=1}^n$ we minimise the discrete graph $p$-Dirichlet energy problem, that is we find the minimiser $f_n$ subject to the known labels
by solving:
\begin{align}\label{eqn:formulation}
\argmin_{f_n:\Lambda^{(m)}_n \to \bbR^c}\mathcal{E}_{\varepsilon_n,m,n}(f_n) \mbox{ subject to $f_n(\mu^{(m)}_i)=y_i$ for $i \leq N$.}    
\end{align}
Here, we denote $f_{n,l}(\mu^{(m)}_i)$ to represent each component of $f_n(\mu^{(m)}_i)$ associated with class $l$.

\item \textbf{Estimate the Predicted Labels:}
We threshold the solution $f_n:\Lambda^{(m)}_n \to \bbR^c$ in \eqref{eqn:formulation} to give a predicted class. 
More precisely, we predict the missing labels by
\begin{equation*}
\hat{y}_i=\argmax_{l\in {1,2,\dots,c}}\{f_{n,l}(\mu^{(m)}_i)\}   \mbox{ for all $i=\{N+1, \cdots,n\}$.}  
\end{equation*}
Based on the predicted missing labels on the unlabelled vertices, we find the accuracy for different fractions of training labels.
\end{enumerate}

For two classes, Steps 3 and 4 simplify by treating $y_i\in\{0,1\}$ and therefore $f_n: \{\mu_i^{(m)}\}_{i=1}^n\to [0,1]$ and the threshold is at 0.5, i.e. $\hat{y}_i = 1$ if $f_n(\mu_i^{(m)}) \geq 0.5$ and $\hat{y}_i = 0$ otherwise.

\subsection{Experiment 1: Synthetic Gaussian Samples}\label{section:2_D}

\begin{figure}[htbp]
    \centering
    \begin{subfigure}[b]{0.45\linewidth}
        \centering
        \includegraphics[width=\textwidth]{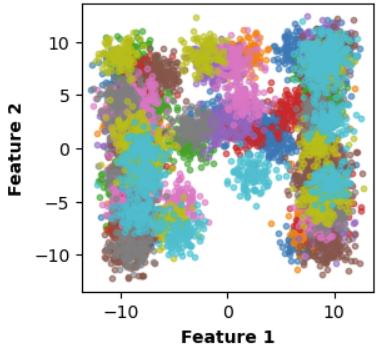}
        \caption{800 2-D synthetic empirical Gaussians (each represented as a feature vector).}
        \label{fig:gaussian-1}
    \end{subfigure}
    \hfill
    \begin{subfigure}[b]{0.45\linewidth}
        \centering
        \includegraphics[width=\textwidth]{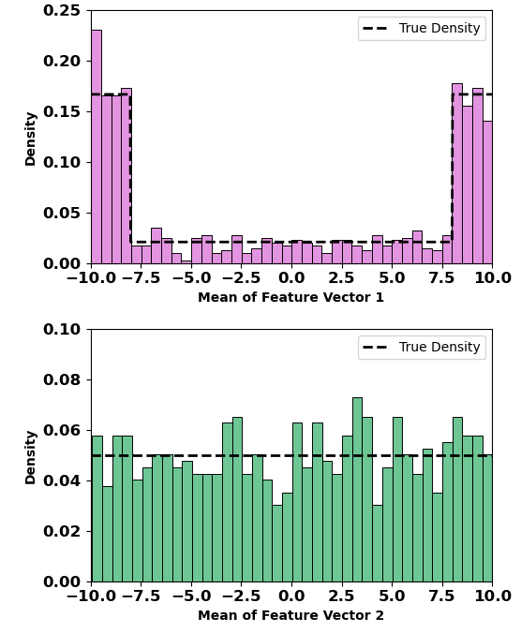}
        \caption{Density plots for 800 2-D synthetic empirical Gaussians.}
        \label{fig:density}
    \end{subfigure}

    \caption{The synthetic dataset comprising 800 Gaussian samples (left, each depicted by different colours), along with the corresponding density plots (right).}
    \label{fig:synthetic_gaussian}
\end{figure}

The synthetic dataset was generated by constructing \( n \) empirical Gaussian distributions, indexed by \( i = 1, \dots, n \) for $n=100,200,400,600,800,1600,3200,6400,12800$ with $m=100$ data points (see Figure~\ref{fig:synthetic_gaussian} for $n=800$). 
Each feature vector \( \mu_i^{(m)} = \frac{1}{m} \sum_{j=1}^m \delta_{x_j^{(i)}} \) is an empirical approximation of the Gaussian $\cN(c_i,\zeta^2 \Id_2)$ where \( \Id_2 \) is the \( 2 \times 2 \) identity matrix and we set \( \zeta^2 = 1 \) throughout the experiments. 
The mean $(c_{i1},c_{i2})$ is sampled for each feature vector.
The first coordinate $c_{i1}$ is a random variable with piecewise uniform density $\rho(c_{i1})$ given by
\[
\rho(c_{i1}) = 
\begin{cases} 
\frac{1}{6}, & \text{if } c_{i1} \in [-10, -8], \quad \text{high-density region I}  \\[4pt]
\frac{1}{48}, & \text{if } c_{i1} \in [-8, 8], \quad \hspace{1.25em} \text{low-density region} \\[4pt]
\frac{1}{6}, & \text{if } c_{i1} \in [8, 10], \quad \hspace{1.5em} \text{high-density region II} \\[2pt]
0, & \text{otherwise}.
\end{cases}
\]

 The second coordinate $c_{i2} \sim \mathrm{U}([-10,10])$ is an independent random variable with constant density. Each feature vector is assigned a binary label \( y_i \in \{0, 1\} \) based on the sign of the first coordinate of its true mean:
\[
y_i = \begin{cases} 
0 & \text{if } c_{i1} < 0, \\
1 & \text{if } c_{i1} \geq 0.
\end{cases}
\]

Whilst our theoretical results used the 2-Wasserstein distance to define the edge weights, this becomes computationally infeasible for large graphs. 
For computational ease we use the linear Wasserstein distance as in \eqref{eqn:linear_wasserstein_distance} and define the weights by

\[ W^{(m)}_{ij} = \lb \begin{array}{ll} 1 & \text{if } \dWlinpmu{2}{\mu_1^{(m)}}(\mu_i^{(m)},\mu^{(m)}_j)<\varepsilon_n \\ 0 & \text{else.} \end{array} \rd \]
We select $\varepsilon_n$ to be approximately the smallest value that ensures the graph is connected. 
For small $n$, using the 2-Wasserstein distance gives very similar accuracy; 
to maintain consistency, we use the linear Wasserstein distance instead for all $n$.

\begin{figure}
    \centering
    \includegraphics[width=0.5\linewidth]{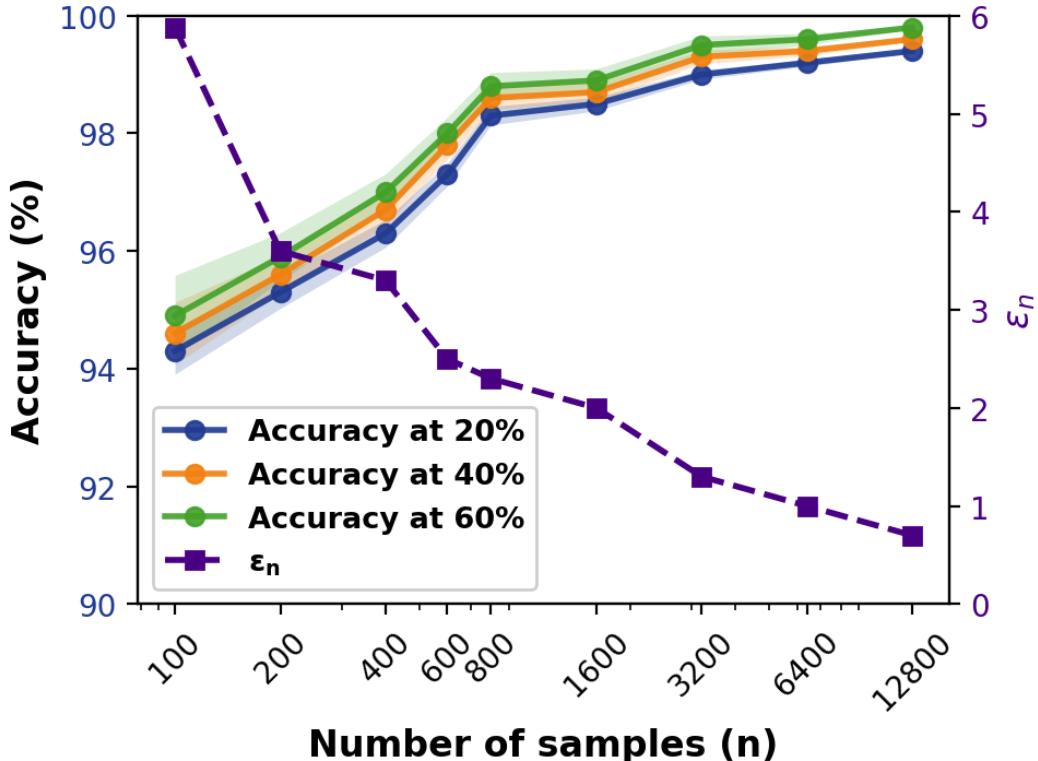}
    \caption{Accuracy (\%) plotted against the number of samples (log scale) in Experiment 1 for varying training label rates (20\%, 40\%, 60\%), represented by coloured lines with markers. The shaded regions around each line indicate the 95\% confidence intervals computed from the standard deviations across 100 runs. The secondary y-axis (purple) shows the corresponding values of the connectivity threshold ($\varepsilon_n$) across sample sizes. Accuracy increases with both the number of samples and training label rates, while $\varepsilon_n$ decreases as sample size grows.}
    \label{fig:accuracy}
\end{figure}

For the synthetic Gaussian dataset, Figure~\ref{fig:accuracy} shows that classification accuracy increases with the number of samples. At smaller sample sizes (100–800), accuracy starts around 92–97\% and rises steadily. At larger sample sizes (1600–12,800), performance begins at higher baselines 98–99\% and continues to improve modestly as more samples are added, demonstrating consistent gains in performance. The connectivity threshold parameter $\varepsilon_n$ was determined heuristically from local inter-cluster distances to guarantee adequate graph connectivity. Overall, these results highlight the model’s ability to exploit the underlying geometric structure of the data, achieving robust classification performance across a range of sample sizes.

\subsection{Experiment 2: ModelNet10 Dataset}\label{section:3_D}

\begin{figure}[htbp]
    \centering
    \begin{subfigure}[b]{0.5\linewidth}  
        \centering
        \includegraphics[width=6cm]{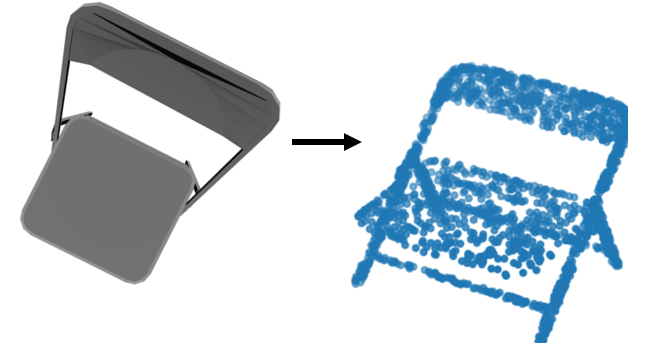}  
        \caption{Sample of ModelNet10 image.}
        \label{fig:sample_point_cloud}
    \end{subfigure}
    \hfill
    \begin{subfigure}[b]{0.8\linewidth}  
        \centering
        \includegraphics[scale=0.6]{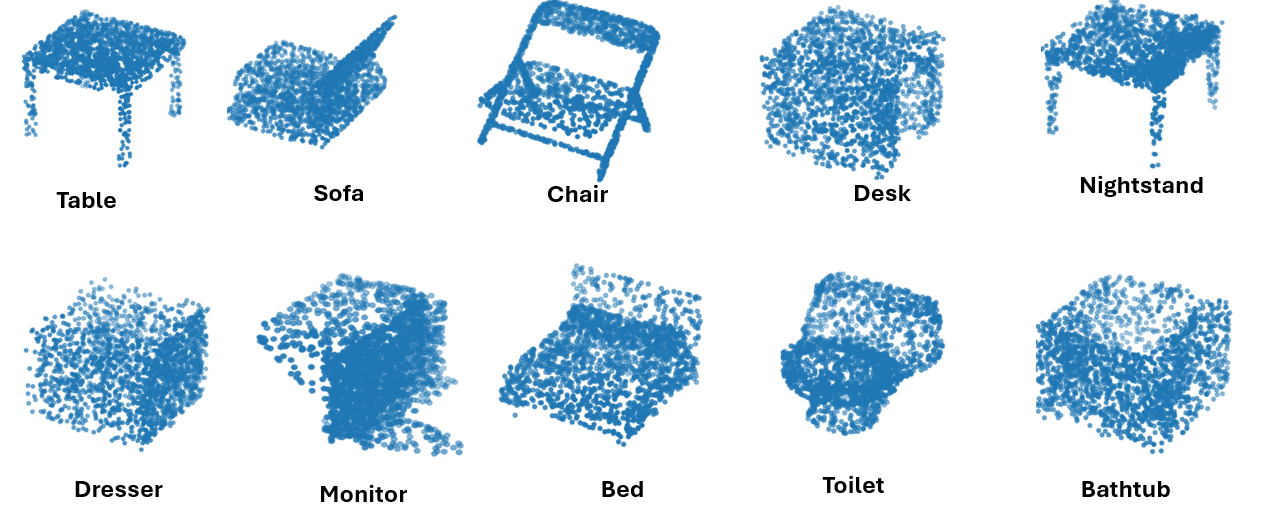}  
        \caption{Point cloud representation of the different classes of ModelNet10 dataset.}
        \label{fig:modelnet10_pc}
    \end{subfigure}
    \caption{ModelNet10 dataset.}
    \label{fig:3D_image}
\end{figure}

The ModelNet10 dataset \cite{Wu_2015} is a widely used benchmark for 3D object recognition and classification, containing 10 distinct categories of 3-D images of furnitures (number of images per class): "bathtub" (156), "bed" (615), "chair" (989), "night stand" (286), "sofa" (780), "table" (492), "dresser" (286), "monitor" (565), "desk" (286), and "toilet" (444). In total, we process 4899 3D objects, each represented by 2048 points. The learning problem consists in assigning each point cloud to its correct category, based on a labelled subset of point clouds

Each image (see Figure \ref{fig:sample_point_cloud}) is represented as a 3D mesh consisting of vertices, edges, and faces that define the object's geometric structure. By uniformly sampling points from the surface of each mesh, the images are converted into point clouds, where each point corresponds to a location on the object’s surface. A fixed number of points, $m = 2048$, are available for each point cloud, so that
\[
\mu_i^{(m)} = \frac{1}{m} \sum_{j=1}^{m} \delta_{x_j^{(i)}} \quad \text{for each } i = 1, \dots, n,
\]
up to relabelling of $x_j^{(i)} \in \mathbb{R}^k$.
 
Again, for computational efficiency we use the linear Wasserstein distance to compute the edge weights as in~\eqref{eqn:linear_wasserstein_distance}. Here, we choose the reference $\mu^{(m)}_1$ as the first point cloud, belonging to the class ‘table’. We also use a second approximation to reduce the number of edges by using a $k$ nearest neighbour ($k$-NN) graph. 
We define the weights 
by
\[ W^{(m)}_{ij} = 
\begin{cases} 
1 & \text{if } \mu_i^{(m)} \text{ is a $k$NN of } \mu_j^{(m)} \text{ or vice versa} \text{ based on } \dWlinpmu{2}{\mu_1^{(m)}}(\mu_i^{(m)},\mu_j^{(m)}) \\
0 & \text{otherwise.} 
\end{cases}
\]
 We choose $k=15, 20, 25$ in our experiment.

Figure \ref{fig:modelnet10} reports the mean classification accuracy over 100 iterations for varying proportions of labelled data (20\% to 80\%). As expected, increasing the fraction of labelled samples improves the model’s ability to capture the underlying data structure and propagate labels to unlabelled nodes. Smaller $k$ values were also associated with higher accuracy. For reference, we evaluated the fully supervised PointNet model \cite{Qi} under the same training ratios, observing a comparable trend, with accuracy increasing from 85\% (20:80 train:test ratio) to 89\% (80:20 train:test ratio). Notably, PointNet is a fully supervised method, whereas Laplace Learning operates in a semi-supervised setting, exploiting both labelled and unlabelled samples via the geometric structure of the data. Although direct quantitative comparisons are not strictly equivalent due to these differing paradigms, Laplace Learning achieves better performance while requiring fewer labelled instances, demonstrating its effectiveness in leveraging geometric relationships for label propagation under limited supervision.

A limitation of our framework is its restriction to discrete probability measures with equal total mass, each supported on the same number of points. While this simplifies the methodology, it limits the generality of the results and precludes application to broader scenarios. Extending the framework to unbalanced variants — such as the Hellinger–Kantorovich (Wasserstein–Fisher–Rao) distance \cite{liero2018optimal, cai2022linearized} to measure spaces, the Unbalanced Gromov–Wasserstein distance \cite{sejourne2021unbalanced} to measure spaces, or the Conic Gromov–Wasserstein and Conic Co-Optimal Transport distances \cite{MCAO_2025b} to measure networks and hyper-networks offers greater flexibility and constitutes a natural direction for future research.

\begin{figure}[ht]
    \centering
    \includegraphics[width=\textwidth]{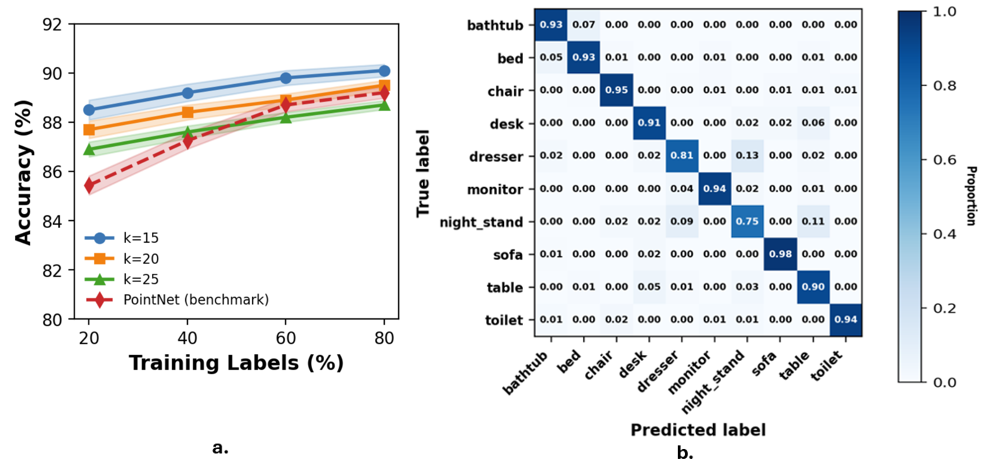}
    \caption{a. Mean classification accuracy on the ModelNet10 dataset using the linear Wasserstein distance with $k$-nearest-neighbour graphs for $k=15,20,25$. Results are averaged over 100 iterations for each training label rate (20–80\%), with 95\% confidence intervals shown as shaded regions around the lines. For comparison, the supervised PointNet model \cite{Qi} was trained for 100 epochs using Adam optimizer (learning rate 0.001), batch size 32, and cross-entropy loss. b. Confusion matrix showing classification accuracy (in \%) for different class labels at 80\% training label rate.}
    \label{fig:modelnet10}
\end{figure}

\section{Acknowledgments}
MCAO acknowledges the receipt of funding from the Health Data Research UK-The Alan Turing Institute Wellcome (Grant Ref: 218529/Z/19/Z) and the Cambridge Trust International scholarship from the Cambridge Commonwealth Europe and International Trust (CCEIT). 
C.-B.S, M.R. have received support
from the Trinity Challenge grant awarded to establish the BloodCounts! consortium, along with NIHR UCLH Biomedical Research Centre, the NIHR Cambridge Biomedical Research Centre. 
M.R. is additionally supported by the British Heart Foundation (TA/F/20/210001). 
C.-B.S. acknowledges support from the EPSRC programme grant in ‘The Mathematics of Deep Learning’ (EP/L015684), Cantab Capital Institute for the Mathematics of Information, the Philip Leverhulme Prize, the Royal Society Wolfson Fellowship, the EPSRC grants EP/S026045/1 and EP/T003553/1, EP/N014588/1, EP/T017961/1, the Wellcome Innovator Award RG98755 and the Alan Turing Institute. 
MT acknowledges the support of Leverhulme Trust Research through the Project Award ``Robust Learning: Uncertainty Quantification, Sensitivity and Stability'' (grant agreement RPG-2024-051) and the NHSBT award 177PATH25 ``Harnessing Computational Genomics to Optimise Blood Transfusion Safety and Efficacy''. 
MT and C.-B.S acknowledges support from the EPSRC Mathematical and Foundations of Artificial Intelligence Probabilistic AI Hub (grant agreement EP/Y007174/1). The authors would also like to thank Bernhard Schmitzer and Caroline Moosm\"{u}ller for insightful discussions related to this work, and Tudor Manole for suggesting an appropriate reference for Theorem \ref{thm:rates_convergence}.

\vspace{2cm}
\bibliographystyle{plain}
\bibliography{references}

@Article{Cedric,
author={Villani, C{\'e}dric},
title={Topics in {O}ptimal {T}ransportation},
journal={American Mathematical Society},
year={2021},
OPTkey = {•},
volume = {58},
number = {},
pages = {},
OPTmonth = {},
OPTnote = {},
OPTannote = {}
}

@Article{Santambrogio,
author={Santambrogio, Filippo},
title={Optimal {T}ransport for {A}pplied {M}athematicians},
journal={Birk{\"a}user, NY},
year={2015},
volume={55},
number={58-63},
pages={94},
publisher={Springer}
}

@Article{Piotr_1,
author={Haj{\l}asz, Piotr},
title={{S}obolev {S}paces on an {A}rbitrary {M}etric {S}pace},
journal={Potential Analysis},
year={1996},
volume={5},
pages={403--415},
publisher={Springer}
}

@incollection{Piotr_2,
  author    = {Haj{\l}asz, Piotr},
  title     = {{Sobolev Spaces on Metric-Measure Spaces}},
  booktitle = {Heat Kernels and Analysis on Manifolds, Graphs, and Metric Spaces},
  editor    = {Grigor'yan, Alexander and Saloff-Coste, Laurent},
  year      = {2003},
  pages     = {173--218},
  publisher = {American Mathematical Society},
  series    = {Contemporary Mathematics},
  volume    = {338}
}

@article{Piotr_3,
  author    = {Haj{\l}asz, Piotr and Koskela, Pekka},
  title     = {{Sobolev Met Poincar{\'e}}},
  journal   = {Memoirs of the American Mathematical Society},
  year      = {2000},
  volume    = {145},
  number    = {688}
}

@article{Flores_2019,
  author    = {Flores, Mauricio and Calder, Jeff and Lerman, Gilad},
  title     = {{Analysis and Algorithms for $\textrm{L}^p$-Based Semi-Supervised Learning on Graphs}},
  journal   = {Applied and Computational Harmonic Analysis},
  year      = {2022},
  volume    = {60},
  pages     = {77--122},
  publisher = {Elsevier}
}

@inproceedings{Wu_2015,
  author    = {Wu, Zhirong and Song, Shuran and Khosla, Aditya and Yu, Fisher and Zhang, Linguang and Tang, Xiaoou and Xiao, Jianxiong},
  title     = {{3D ShapeNets: A Deep Representation for Volumetric Shapes}},
  booktitle = {Proceedings of the IEEE Conference on Computer Vision and Pattern Recognition},
  pages     = {1912--1920},
  year      = {2015}
}

@article{Charles_2016,
  author    = {Fefferman, Charles and Mitter, Sanjoy and Narayanan, Hariharan},
  title     = {{Testing the Manifold Hypothesis}},
  journal   = {Journal of the American Mathematical Society},
  year      = {2016},
  volume    = {29},
  number    = {4},
  pages     = {983--1049}
}

@article{Fornasier_1,
  author    = {Fornasier, Massimo and Savar{\'e}, Giuseppe and Sodini, Giacomo Enrico},
  title     = {{Density of Subalgebras of {L}ipschitz Functions in Metric {S}obolev Spaces and Applications to {W}asserstein {S}obolev Spaces}},
  journal   = {Journal of Functional Analysis},
  year      = {2023},
  volume    = {285},
  number    = {11},
  pages     = {110153},
  publisher = {Elsevier}
}

@article{Tenenbaum,
  author    = {Tenenbaum, Joshua B. and Silva, Vin de and Langford, John C.},
  title     = {{A Global Geometric Framework for Nonlinear Dimensionality Reduction}},
  journal   = {Science},
  year      = {2000},
  volume    = {290},
  number    = {5500},
  pages     = {2319--2323},
  publisher = {American Association for the Advancement of Science}
}

@techreport{Bernstein,
  author      = {Bernstein, Mira and De Silva, Vin and Langford, John C. and Tenenbaum, Joshua B.},
  title       = {{Graph Approximations to Geodesics on Embedded Manifolds}},
  year        = {2000},
  institution = {Citeseer}
}

@article{Fan_2006,
  author    = {Fan, Jianqing and Ren, Yi},
  title     = {{Statistical Analysis of DNA Microarray Data in Cancer Research}},
  journal   = {Clinical Cancer Research},
  year      = {2006},
  volume    = {12},
  number    = {15},
  pages     = {4469--4473},
  publisher = {AACR}
}

@article{Clarke_2008,
  author    = {Clarke, Robert and Ressom, Habtom W. and Wang, Antai and Xuan, Jianhua and Liu, Minetta C. and Gehan, Edmund A. and Wang, Yue},
  title     = {{The Properties of High-Dimensional Data Spaces: Implications for Exploring Gene and Protein Expression Data}},
  journal   = {Nature Reviews Cancer},
  year      = {2008},
  volume    = {8},
  number    = {1},
  pages     = {37--49},
  publisher = {Nature Publishing Group}
}

@inproceedings{Koppen,
  author    = {K{\"o}ppen, Mario},
  title     = {{The Curse of Dimensionality}},
  booktitle = {5th Online World Conference on Soft Computing in Industrial Applications (WSC5)},
  year      = {2000},
  volume    = {1},
  pages     = {4--8}
}

@book{Duda,
  author    = {Duda, R. O. and Hart, P. E. and Stork, D. G.},
  title     = {{Pattern Recognition}},
  publisher = {John Wiley \& Sons},
  year      = {2001}
}

@article{Dudoit,
  author    = {Dudoit, Sandrine and Fridlyand, Jane and Speed, Terence P.},
  title     = {{Comparison of Discrimination Methods for the Classification of Tumors Using Gene Expression Data}},
  journal   = {Journal of the American Statistical Association},
  year      = {2002},
  volume    = {97},
  number    = {457},
  pages     = {77--87},
  publisher = {Taylor \& Francis}
}

@article{Hamm_2,
  author    = {Hamm, Keaton and Henscheid, Nick and Kang, Shujie},
  title     = {{WassMap: {W}asserstein Isometric Mapping for Image Manifold Learning}},
  journal   = {SIAM Journal on Mathematics of Data Science},
  year      = {2023},
  volume    = {5},
  number    = {2},
  pages     = {475--501},
  publisher = {SIAM}
}

@article{Fornasier_2,
  author    = {Fornasier, Massimo and Heid, Pascal and Sodini, Giacomo Enrico},
  title     = {{Approximation Theory, Computing, and Deep Learning on the Wasserstein Space}},
  journal   = {Mathematical Models and Methods in Applied Sciences},
  year      = {2025},
  doi       = {10.1142/S0218202525500113}
}

@article{Nicolas_pointcloud,
  author    = {Garc{\'\i}a Trillos, Nicol{\'a}s and Slep{\v{c}}ev, Dejan},
  title     = {{Continuum Limit of Total Variation on Point Clouds}},
  journal   = {Archive for Rational Mechanics and Analysis},
  year      = {2016},
  volume    = {220},
  pages     = {193--241},
  publisher = {Springer}
}

@inproceedings{Qi,
  author    = {Qi, Charles R. and Su, Hao and Mo, Kaichun and Guibas, Leonidas J.},
  title     = {{PointNet: Deep Learning on Point Sets for 3D Classification and Segmentation}},
  booktitle = {2017 IEEE Conference on Computer Vision and Pattern Recognition (CVPR)},
  year      = {2017},
  pages     = {77--85},
  organization = {IEEE}
}

@article{Hamm,
  author    = {Hamm, Keaton and Moosm{\"u}ller, Caroline and Schmitzer, Bernhard and Thorpe, Matthew},
  title     = {{Manifold Learning in Wasserstein Space}},
  journal   = {SIAM Journal on Mathematical Analysis},
  volume    = {57},
  number    = {3},
  pages     = {2983--3029},
  year      = {2025},
  publisher = {SIAM}
}

@article{Nicolas_rate_convergence,
  author    = {Trillos, Nicol{\'a}s Garcia and Slep{\v{c}}ev, Dejan},
  title     = {{On the Rate of Convergence of Empirical Measures in $\infty$-Transportation Distance}},
  journal   = {Canadian Journal of Mathematics},
  volume    = {67},
  number    = {6},
  pages     = {1358--1383},
  year      = {2015},
  publisher = {Cambridge University Press}
}

@article{Fonseca,
  author    = {Fonseca, Irene and Kreusser, Lisa Maria and Sch{\"o}nlieb, Carola-Bibiane and Thorpe, Matthew},
  title     = {{$\Gamma$-Convergence of an Ambrosio-Tortorelli Approximation Scheme for Image Segmentation}},
  journal   = {Indiana University Mathematics Journal},
  year      = {2024},
  publisher = {Indiana University Mathematics Journal}
}

@book{Villani_Old,
  author    = {Villani, C{\'e}dric},
  title     = {{Optimal Transport: Old and New}},
  volume    = {338},
  year      = {2009},
  publisher = {Springer}
}

@article{Monge,
  author    = {Monge, Gaspard},
  title     = {{M{\'e}moire sur la Th{\'e}orie des D{\'e}blais et des Remblais}},
  journal   = {Mem. Math. Phys. Acad. Royale Sci.},
  pages     = {666--704},
  year      = {1781}
}

@article{Matt_p_Laplacian,
  author    = {Slep{\v{c}}ev, Dejan and Thorpe, Matthew},
  title     = {{Analysis of p-Laplacian Regularization in Semi-Supervised Learning}},
  journal   = {SIAM Journal on Mathematical Analysis},
  year      = {2019},
  volume    = {51},
  number    = {3},
  pages     = {2085--2120},
  publisher = {SIAM}
}

@inproceedings{PTLp,
  author    = {Liu, Xinran and Bai, Yikun and Tran, Huy and Zhu, Zhanqi and Thorpe, Matthew and Kolouri, Soheil},
  title     = {{PTLP}: Partial Transport $\textrm{L}^p$ Distances},
  booktitle = {NeurIPS 2023 Workshop Optimal Transport and Machine Learning},
  year      = {2023},
  url       = {https://openreview.net/forum?id=RXYURNZzfs}
}

@inproceedings{Sliced_PTLp,
  author    = {Bai, Yikun and Schmitzer, Bernhard and Thorpe, Matthew and Kolouri, Soheil},
  title     = {{Sliced Optimal Partial Transport}},
  booktitle = {Proceedings of the IEEE/CVF Conference on Computer Vision and Pattern Recognition},
  pages     = {13681--13690},
  year      = {2023}
}

@book{Garling,
  author    = {Garling, David J. H.},
  title     = {{Analysis on {P}olish Spaces and an Introduction to Optimal Transportation}},
  volume    = {89},
  year      = {2018},
  publisher = {Cambridge University Press}
}

@article{Rachev,
  author    = {Rachev, S. T.},
  title     = {{L. Ruschendorf. Mass Transportation Problems}},
  journal   = {Probab. Appl. Springer-Verlag, New York},
  year      = {1998}
}

@book{Penrose,
  author    = {Penrose, Mathew},
  title     = {{Random Geometric Graphs}},
  volume    = {5},
  year      = {2003},
  publisher = {OUP Oxford}
}

@article{Belkin_2004,
  author    = {Belkin, Mikhail and Niyogi, Partha},
  title     = {{Semi-Supervised Learning on Riemannian Manifolds}},
  journal   = {Machine Learning},
  year      = {2004},
  volume    = {56},
  pages     = {209--239},
  publisher = {Springer}
}

@article{Wang_2012,
  author    = {Wang, Wei and Slep{\v{c}}ev, Dejan and Basu, Saurav and Ozolek, John A. and Rohde, Gustavo K.},
  title     = {{A Linear Optimal Transportation Framework for Quantifying and Visualizing Variations in Sets of Images}},
  journal   = {International Journal of Computer Vision},
  year      = {2013},
  volume    = {101},
  pages     = {254--269},
  publisher = {Springer}
}

@article{Chow_2019,
  author    = {Chow, Yat Tin and Gangbo, Wilfrid},
  title     = {{A Partial Laplacian as an Infinitesimal Generator on the Wasserstein Space}},
  journal   = {Journal of Differential Equations},
  year      = {2019},
  volume    = {267},
  number    = {10},
  pages     = {6065--6117},
  publisher = {Elsevier}
}

@article{Hein_2007,
  author    = {Hein, Matthias and Audibert, Jean-Yves and Luxburg, Ulrike von},
  title     = {{Graph Laplacians and Their Convergence on Random Neighborhood Graphs}},
  journal   = {Journal of Machine Learning Research},
  year      = {2007},
  volume    = {8},
  number    = {6}
}

@book{Braides_Gamma_Convergence,
  author    = {Braides, Andrea},
  title     = {{Gamma-Convergence for Beginners}},
  volume    = {22},
  year      = {2002},
  publisher = {Clarendon Press}
}

@book{Dal_Maso_Gamma_Convergence,
  author    = {Dal Maso, Gianni},
  title     = {{An Introduction to $\Gamma$-Convergence}},
  volume    = {8},
  year      = {2012},
  publisher = {Springer Science \& Business Media}
}

@article{Mikhail_2007,
  author    = {Belkin, Mikhail and Niyogi, Partha},
  title     = {{Convergence of Laplacian Eigenmaps}},
  journal   = {Advances in Neural Information Processing Systems},
  year      = {2006},
  volume    = {19}
}

@article{Nadler_2006,
  author    = {Nadler, Boaz and Lafon, Stéphane and Coifman, Ronald R. and Kevrekidis, Ioannis G.},
  title     = {{Diffusion Maps, Spectral Clustering and Reaction Coordinates of Dynamical Systems}},
  journal   = {Applied and Computational Harmonic Analysis},
  year      = {2006},
  volume    = {21},
  number    = {1},
  pages     = {113--127},
  note      = {Special Issue: Diffusion Maps and Wavelets},
  issn      = {1063-5203},
  doi       = {https://doi.org/10.1016/j.acha.2005.07.004},
  url       = {https://www.sciencedirect.com/science/article/pii/S1063520306000534}
}

@article{Ronald_2006,
  author    = {Coifman, Ronald R. and Lafon, Stéphane},
  title     = {{Diffusion Maps}},
  journal   = {Applied and Computational Harmonic Analysis},
  volume    = {21},
  number    = {1},
  pages     = {5--30},
  year      = {2006}
}

@article{Evarist_2006,
  author    = {Giné, Evarist and Koltchinskii, Vladimir},
  title     = {{Empirical Graph Laplacian Approximation of Laplace–Beltrami Operators: Large Sample Results}},
  journal   = {IMS Lecture Notes Monographs Series},
  volume    = {51},
  pages     = {238--259},
  publisher = {Institute of Mathematical Statistics},
  year      = {2006}
}

@article{Amir_2006,
  author    = {Singer, Amir},
  title     = {{From Graph to Manifold Laplacian: The Convergence Rate}},
  journal   = {Applied and Computational Harmonic Analysis},
  volume    = {21},
  pages     = {128--134},
  year      = {2006}
}

@inproceedings{Rasmus_2015,
  author    = {Kyng, Rasmus and Rao, Anup and Sachdeva, Sushant and Spielman, Daniel A.},
  title     = {{Algorithms for Lipschitz Learning on Graphs}},
  booktitle = {Proceedings of the Conference on Learning Theory},
  pages     = {1190--1223},
  year      = {2015}
}

@article{Jeff_2018,
  author    = {Calder, Jeff},
  title     = {{The Game Theoretic p-Laplacian and Semi-Supervised Learning with Few Labels}},
  journal   = {Nonlinearity},
  volume    = {32},
  number    = {1},
  pages     = {301--330},
  month     = {December},
  year      = {2018}
}

@inproceedings{Zhu_2003,
  author    = {Zhu, Xianjin and Ghahramani, Zoubin and Lafferty, John},
  title     = {{Semi-Supervised Learning Using Gaussian Fields and Harmonic Functions}},
  booktitle = {Proceedings of the International Conference on Machine Learning},
  year      = {2003}
}

@article{Oliver_2024,
  author    = {Antony Oliver, Mary Chriselda and Graham, Matthew and Manolopoulou, Ioanna and Medley, Graham F. and Pellis, Lorenzo and Pouwels, Koen B. and Thorpe, Matthew and Hollingsworth, T. Deirdre},
  title     = {{Uncertainty Quantification in Cost-Effectiveness Analysis for Stochastic-Based Infectious Disease Models: Insights from Surveillance on Lymphatic Filariasis}},
  journal   = {Journal of Theoretical Biology},
  pages     = {112197},
  year      = {2025},
  publisher = {Elsevier}
}

@article{MCAO_2025b,
  author    = {Antony Oliver, Mary Chriselda and Hartman, Emmanuel and Needham, Tom},
  title     = {{Conic Formulations of Transport Metrics for Unbalanced Measure Networks and Hyper-Networks}},
  journal   = {arXiv preprint arXiv:2508.10888},
  year      = {2025}
}

@article{Divol2021Short,
  author  = {Divol, Vincent},
   title   = {{A Short Proof on the Rate of Convergence of the Empirical Measure for the Wasserstein Distance}},
  journal = {arXiv preprint arXiv:2101.08126},
  year    = {2021}
}

@article{liero2018optimal,
  author    = {Liero, Matthias and Mielke, Alexander and Savar{\'e}, Giuseppe},
  title     = {{Optimal Entropy-Transport Problems and a New Hellinger--Kantorovich Distance Between Positive Measures}},
  journal   = {Inventiones Mathematicae},
  volume    = {211},
  number    = {3},
  pages     = {969--1117},
  year      = {2018},
  publisher = {Springer}
}

@article{cai2022linearized,
  author    = {Cai, Tianji and Cheng, Junyi and Schmitzer, Bernhard and Thorpe, Matthew},
  title     = {{The Linearized Hellinger--Kantorovich Distance}},
  journal   = {SIAM Journal on Imaging Sciences},
  volume    = {15},
  number    = {1},
  pages     = {45--83},
  year      = {2022},
  publisher = {SIAM}
}

@article{sejourne2021unbalanced,
  author    = {S{\'e}journ{\'e}, Thibault and Vialard, Fran{\c{c}}ois-Xavier and Peyr{\'e}, Gabriel},
  title     = {{The Unbalanced Gromov Wasserstein Distance: Conic Formulation and Relaxation}},
  journal   = {Advances in Neural Information Processing Systems},
  volume    = {34},
  pages     = {8766--8779},
  year      = {2021}
}

@article{otto2001geometry,
  author    = {Otto, Felix},
  title     = {{The Geometry of Dissipative Evolution Equations: The Porous Medium Equation}},
  journal   = {Communications in Partial Differential Equations},
  volume    = {26},
  pages     = {101--174},
  year      = {2001}
}

@inproceedings{Ting_2011,
  author    = {Ting, Daniel and Huang, Ling and Jordan, Michael I.},
  title     = {{An Analysis of the Convergence of Graph Laplacians}},
  booktitle = {Proceedings of the 27th International Conference on Machine Learning (ICML'10)},
  pages     = {1079--1086},
  year      = {2010}
}

@article{Jeff_2022,
  author    = {Calder, Jeff and García Trillos, Nicolás},
  title     = {{Improved Spectral Convergence Rates for Graph Laplacians on $\varepsilon$-Graphs and k-NN Graphs}},
  journal   = {Applied and Computational Harmonic Analysis},
  volume    = {60},
  pages     = {123--175},
  year      = {2022}
}

@article{Nicolas_2020,
  author    = {Garc{\'\i}a Trillos, Nicolás and Gerlach, Moritz and Hein, Matthias and Slep{\v{c}}ev, Dejan},
  title     = {{Error Estimates for Spectral Convergence of the Graph Laplacian on Random Geometric Graphs Toward the Laplace--Beltrami Operator}},
  journal   = {Foundations of Computational Mathematics},
  volume    = {20},
  number    = {4},
  pages     = {827--887},
  year      = {2020},
  publisher = {Springer}
}

@inproceedings{Matthias_2006,
  author    = {Hein, Matthias},
  title     = {{Uniform Convergence of Adaptive Graph-Based Regularization}},
  booktitle = {International Conference on Computational Learning Theory},
  publisher = {Springer Berlin Heidelberg},
  year      = {2006}
}

@inproceedings{Matthias_2005,
  author    = {Hein, Matthias and Audibert, Jean-Yves and von Luxburg, Ulrike},
  title     = {{From Graphs to Manifolds – Weak and Strong Pointwise Consistency of Graph Laplacians}},
  booktitle = {Proceedings of the Conference on Learning Theory},
  pages     = {470--485},
  year      = {2005}
}

@article{Bruggner_2014,
  author    = {Bruggner, Robert V. and Bodenmiller, Bernd and Dill, David L. and Tibshirani, Robert J. and Nolan, Garry P.},
  title     = {{Automated Identification of Stratifying Signatures in Cellular Subpopulations}},
  journal   = {Proceedings of the National Academy of Sciences},
  volume    = {111},
  number    = {26},
  pages     = {E2770--E2777},
  year      = {2014},
  publisher = {National Acad Sciences}
}

@article{calder2022lipschitz,
  author    = {Calder, Jeff and García Trillos, Nicolás and Lewicka, Marta},
  title     = {{Lipschitz Regularity of Graph Laplacians on Random Data Clouds}},
  journal   = {SIAM Journal on Mathematical Analysis},
  volume    = {54},
  number    = {1},
  pages     = {1169--1222},
  year      = {2022},
  publisher = {SIAM}
}

@article{Jun_2021,
  author    = {Zhao, Jun and Jaffe, Ariel and Li, Henry and Lindenbaum, Ofir and Sefik, Esen and Jackson, Ruaidhrí and Cheng, Xiuyuan and Flavell, Richard A. and Kluger, Yuval},
  title     = {{Detection of Differentially Abundant Cell Subpopulations in scRNA-Seq Data}},
  journal   = {Proceedings of the National Academy of Sciences},
  volume    = {118},
  number    = {22},
  pages     = {e2100293118},
  year      = {2021},
  publisher = {National Acad Sciences}
}

@article{Nicolas_bayesian_2020,
  author    = {Trillos, Nicol\'{a}s Garc\'{i}a and Kaplan, Zachary and Samakhoana, Thabo and Sanz-Alonso, Daniel},
  title     = {{On the Consistency of Graph-Based Bayesian Semi-Supervised Learning and the Scalability of Sampling Algorithms}},
  journal   = {Journal of Machine Learning Research},
  volume    = {21},
  number    = {28},
  pages     = {1--47},
  year      = {2020}
}

@article{ajtai1984optimal,
  author    = {Ajtai, Mikl\'{o}s and Koml\'{o}s, J\'{a}nos and Tusn\'{a}dy, G\'{a}bor},
  title     = {{On Optimal Matchings}},
  journal   = {Combinatorica},
  volume    = {4},
  number    = {4},
  pages     = {259--264},
  year      = {1984},
  publisher = {Springer-Verlag Berlin/Heidelberg}
}

@article{fournier2015rate,
  title={On the rate of convergence in {W}asserstein distance of the empirical measure},
  author={Fournier, Nicolas and Guillin, Arnaud},
  journal={Probability theory and related fields},
  volume={162},
  number={3},
  pages={707--738},
  year={2015},
  publisher={Springer}
}

@article{weed2019sharp,
  title={Sharp asymptotic and finite-sample rates of convergence of empirical measures in {W}asserstein distance},
  author={Weed, Jonathan and Bach, Francis},
  journal={Bernoulli},
  volume={25},
  number={4A},
  pages={2620--2648},
  year={2019},
  publisher={JSTOR}
}

@book{bobkov2019one,
  title={{One-dimensional empirical measures, order statistics, and Kantorovich transport distances}},
  author={Bobkov, Sergey and Ledoux, Michel},
  volume={261},
  year={2019},
  publisher={American Mathematical Society}
}

@article{manole2024plugin,
  title={{Plugin estimation of smooth optimal transport maps}},
  author={Manole, Tudor and Balakrishnan, Sivaraman and Niles-Weed, Jonathan and Wasserman, Larry},
  journal={The Annals of Statistics},
  volume={52},
  number={3},
  pages={966--998},
  year={2024},
  publisher={Institute of Mathematical Statistics}
}

@article{ambrosio2022quadratic,
  title={{On the quadratic random matching problem in two-dimensional domains}},
  author={Ambrosio, Luigi and Goldman, Michael and Trevisan, Dario},
  journal={Electronic Journal of Probability},
  volume={27},
  pages={1--35},
  year={2022},
  publisher={The Institute of Mathematical Statistics and the Bernoulli Society}
}

@article{bobkov2021simple,
  title={{A simple Fourier analytic proof of the AKT optimal matching theorem}},
  author={Bobkov, Sergey G and Ledoux, Michel},
  journal={The Annals of Applied Probability},
  volume={31},
  number={6},
  pages={2567--2584},
  year={2021},
  publisher={Institute of Mathematical Statistics}
}

@inproceedings{hundrieser2024empirical,
  title={{Empirical optimal transport between different measures adapts to lower complexity}},
  author={Hundrieser, Shayan and Staudt, Thomas and Munk, Axel},
  booktitle={Annales de l'Institut Henri Poincare (B) Probabilites et statistiques},
  pages={824--846},
  year={2024},
  organization={Institut Henri Poincar{\'e}}
}

\section{Appendix}

In this section, we prove the compactness and $\Gamma$-convergence of the discrete energy functional $\hat{\mathcal{E}}_{\varepsilon_n,n}$ to its continuum counterpart $\hat{\mathcal{E}}_\infty$ in the parameter space $\mathcal{S} \subset \bbR^d$.

\begin{proposition}[Compactness]\label{prop:compactness_parameter_manifold}
Assume $\mathcal{S}$, $\bbP_\mathcal{S}$, $\bbP_{\mathcal{S}_n}$, and $\varepsilon_n$ satisfy \ref{ass:domain_ll_three}, \ref{ass:prob_measures_S}, \ref{ass:S_n}, and \ref{ass:eps_scale}. 
Let $\hat{\eta}_\theta:\bbR^d \to [0,\infty)$ and assume there exist constants $a,r > 0$ such that $\hat{\eta}_\theta(h) \geq a$ for $|h|<r$ and for every $\theta\in\cS$. Define $\hat{\mathcal{E}}_{\varepsilon_n,n}$ by  \eqref{eqn:discrete_energy_parameter_intro}. Then, with probability one, any sequence $\hat{f}_n :\mathcal{S}_n \to \bbR$ with $\sup_{n \in \bbN}\hat{\mathcal{E}}_{\varepsilon_n,n}(\hat{f}_n) < +\infty$ and $\sup_{n \in \bbN}\|\hat{f}_n\|_{\Lp{\infty}(\bbP_{\mathcal{S}_n})} < \infty$ has a subsequence $\hat{f}_{n_k}$ such that $(\hat{f}_{n_k},\bbP_{\mathcal{S}_{n_k}})$ converges in $\TLp{p}(\mathcal{S})$ to $(\hat{f},\bbP_\mathcal{S})$ for some $\hat{f} \in \Wkp{1}{p}(\cS)$.
\end{proposition}

\begin{proof}

We define
\[ \tilde{\eta}(h) = \lb \begin{array}{ll} a & \text{if } |h|\leq r \\ 0 & \text{else} \end{array} \rd \]

and
\begin{align}
\tilde{\mathcal{E}}_{\varepsilon_n,n}(\hat{f}_n)=\frac{1}{n^2\varepsilon^{p}_n}\sum_{i,j=1}^n\tilde{W}_{ij}|\hat{f}_n(\theta_i)-\hat{f}_n(\theta_j)|^p    
\end{align}
where $\tilde{W}_{ij}=\frac{1}{\varepsilon^d_n}\tilde{\eta}\bigg(\frac{|\theta_i-\theta_j|}{\varepsilon_n}\bigg)$. 
Since $\tilde{\eta}\leq \eta$ we have $\tilde{\mathcal{E}}_{\varepsilon_n,n}(\hat{f}_n)\leq \hat{\mathcal{E}}_{\varepsilon_n,n}(\hat{f}_n)$, and therefore $\sup_{n \in \bbN}\tilde{\mathcal{E}}_{\varepsilon_n,n}(\hat{f}_n) < +\infty$.
By~\cite[Proposition 4.6]{Nicolas_pointcloud} $\{\hat{f}_n\}_{n=1}^\infty$ is precompact in $\TLp{p}(\cS)$ and any limit point $\hat{f}$ is in $\Wkp{1}{p}(\cS)$.  
\end{proof}

\begin{theorem}[$\Gamma$-Convergence]\label{thm:gamma_convergence_parameter_manifold}
 Assume $\mathcal{S}$, $\bbP_\mathcal{S}$,  $\bbP_{\mathcal{S}_n}$, and $\varepsilon_n$ satisfy \ref{ass:domain_ll_three}, \ref{ass:prob_measures_S}, \ref{ass:S_n}, and \ref{ass:eps_scale} with $p>1$. 
Assume $\hat{\eta}_\theta:\bbR^d \to [0,\infty)$ for $\theta\in\cS$ satisfies the following conditions for all $\theta \in \mathcal{S}$,
\begin{enumerate}
\item\label{item:compact_support_eta_hat}  $\hat{\eta}_\theta$ has compact support in $B(0,R)$;
 \item  $\hat{\eta}_\theta(h)=\hat{\eta}_\theta(-h)$ for all $h\in\bbR^d$; 
 \item  $\theta\mapsto \hat{\eta}_\theta(h)$ is pointwise equicontinuous; 
 \item  there exists constants $a,r>0$ such that $\hat{\eta}_\theta(h) \geq a$ for $|h|<r$; and
\item\label{item:continuity_assumption_eta_hat}  for all $\xi> 0$, $\psi > 0$ there exists constants $\alpha_{\xi,\psi} >0$, $c_{\xi,\psi} >0$ such that if $\|h_1-h_2\|\leq\xi$ and $\|\theta_1-\theta_2\|\leq\psi$ then $\hat{\eta}_{\theta_1}(h_1)\geq c_{\xi,\psi}\hat{\eta}_{\theta_2}(\alpha_{\xi,\psi} h_2)$ with $c_{\xi,\psi} \to 1$, $\alpha_{\xi,\psi} \to 1$ as $\xi,\psi \to 0$.
\end{enumerate}
Define $\hat{\mathcal{E}}_{\varepsilon_n,n}$ by~\eqref{eqn:discrete_energy_parameter_intro} and $\hat{\cE}_\infty$ by \eqref{eqn:continuum_energy_functional_discrete_intro}. Then, for any $M>0$, with probability one, we have that $\hat{\mathcal{E}}_{\varepsilon_n,n}$  $\Gamma-$ converges to $\hat{\mathcal{E}}_\infty$ as $n \to \infty$ on the set $\{(\hat{f},\bbP_\mathcal{S})\in\TLp{p}(\cS)\,:\,\|\hat{f}\|_{\Lp{\infty}(\bbP_\mathcal{S})}\leq M\}$ with probability one.
\end{theorem}

We briefly justify the assumptions on $\hat{\eta}_\theta(h)$ as listed in Lemma \ref{lemma:assumptions_eta_hat}. Assumption \ref{item:eta_hat_compact_support} ensures that $\hat{\eta}_\theta$ is localized, meaning it vanishes outside a bounded region. This is useful in analysis, preventing unnecessary complications from unbounded tails. Assumption \ref{item:eta_hat_even_function} is supported by the fact that $\hat{\eta}_\theta$ is used as an interaction potential which we assume to be symmetric. Assumption \ref{item:eta_hat_pointwise_equicontinuos} guarantees that small changes in $\theta$ do not lead to arbitrarily large changes in $\hat{\eta}_\theta$. Equicontinuity prevents instability in parameter dependence, which is crucial when analysing the convergence of $\hat{\mathcal{E}}_{\varepsilon_n,n}$ to $\hat{\mathcal{E}}_\infty$ as $n \to \infty$. Assumption~\ref{item:eta_hat_existence_of_constants} ensures that $\hat{\eta}_\theta$ is uniformly bounded below by a positive constant in a neighbourhood of the origin.  
This ensures that we connect all feature vectors that are ``close'' (i.e. there are no unseen directions).
Assumption~\ref{item:eta_hat_continuity} ensures that $\hat{\eta}_\theta$ does not change too abruptly under small perturbations in $h$ or the parameter $\theta$. In particular, it provides local control that allows $\hat{\eta}_\theta$ to be lower- and upper-bounded in sufficiently small neighbourhoods. By choosing the neighbourhood radius appropriately, these bounds can be made arbitrarily close to the values at nearby points, even if $\hat{\eta}_\theta$ is discontinuous.

A consequence of Assumptions~\ref{item:eta_hat_compact_support} and~\ref{item:eta_hat_continuity} is that $\|\hat{\eta}_\theta\|_{\Lp{\infty}}<\infty$ and $\int_{\bbR^d}\hat{\eta}_\theta(h)|h|\,\dd h<\infty$.

The proofs for $\Gamma$-convergence are divided into the lim-inf and lim-sup inequalities.

\begin{lemma}[lim inf-inequality]\label{lemma:lim-inf_parameter_manifold}
Under the same assumptions as Theorem \ref{thm:gamma_convergence_parameter_manifold}, with probability one, for any function $\hat{f}\in \Lp{p}(\bbP_\mathcal{S})$ with $\|\hat{f}\|_{\Lp{\infty}(\bbP_\mathcal{S})}<\infty$ and any sequence $\hat{f}_n \to \hat{f}$ in $\TLp{p}(\cS)$ with $\sup\limits_{n\in\bbN}\|\hat{f}_n\|_{\Lp{\infty}(\bbP_{\mathcal{S}_n})} < \infty$ we have, 

\begin{align}\label{eqn:lim_inf_appendix_proof}
\begin{aligned}
\liminf_{n \to \infty} \hat{\mathcal{E}}_{\varepsilon_n,n}(\hat{f}_n) {\geq}\hat{\mathcal{E}}_\infty(\hat{f}). \\
\end{aligned}
\end{align}
\end{lemma}

\begin{proof} Assume $\hat{f}_n \xrightarrow[]{\TLp{p}} \hat{f}$ as $n \to \infty$. 
Let $\hat{T}_n$ be the transport map satisfying $\hat{T}_{n\#}\bbP_{\mathcal{S}}=\bbP_{\mathcal{S}_n}$ and $\|\hat{T}_n-\Id\|_{\Lp{\infty}} \leq \hat{C}q_d(n)$ (as in Theorem~\ref{thm:rates_convergence}). We claim that such a sequence of map exist with probability one and we continue the proof under the assumption that such a sequence exists.

The proof (adapted from \cite{Nicolas_pointcloud}) has four main steps: 
\begin{align}
\begin{aligned}\label{eqn:lim_inf_appendix_proof_proof}
\liminf_{n \to \infty} \hat{\mathcal{E}}_{\varepsilon_n,n}(\hat{f}_n) & \stackrel{(a.)}{\geq} \liminf_{n \to \infty} \hat{\mathcal{E}}_{\tilde{\varepsilon}_n,\infty}(\hat{f}_n\circ \hat{T}_n;\cS) \stackrel{(b.)}{\geq} \liminf_{n \to \infty}\hat{\mathcal{E}}_{\hat{\varepsilon}_n,\infty}(\hat{f}^{(\delta_n)}_n;\mathcal{S}')\\
&\stackrel{(c.)}{\geq} \liminf_{n \to \infty}\hat{\mathcal{E}}_{\infty}(\hat{f}^{(\delta_n)};\mathcal{S}^{\prime\prime\prime}) \stackrel{(d.)}{\geq}\hat{\mathcal{E}}_\infty(\hat{f};\mathcal{S}^{\prime\prime\prime}) \\
\end{aligned}
\end{align}
where $\mathcal{S}^{\prime\prime\prime}\subset\cS^\prime\subset\cS$ will be defined later,
\begin{align} \label{eq:hatEinftyS'}
\hat{\mathcal{E}}_{{\varepsilon},\infty}(g;\mathcal{S}^{\prime\prime\prime})=\frac{1}{{\varepsilon}^{p+d}}\int_{\mathcal{S}^{\prime\prime\prime}\times \mathcal{S}^{\prime\prime\prime}}\hat{\eta}_\theta\bigg(\frac{\theta-\tau}{{\varepsilon}}\bigg)|g(\theta)-g(\tau)|^p\rho_\mathcal{S}(\theta)\rho_\mathcal{S}(\tau)\,\textnormal{d}\theta\textnormal{d}\tau.
\end{align}
Here, we denote $\tilde{\eps}_n=\frac{\eps_n}{\alpha_n}$, $\hat{\eps}_n=\frac{\tilde{\eps}_n}{\tilde{\alpha}_n}$ (with $\alpha_n$ and $\tilde{\alpha}_n$ defined later), and
\[ \hat{f}^{(\delta_n)}_n(\theta)=\int_{\bbR^d}J_{\delta_n}(\theta-z)\hat{f}_n\circ\hat{T}_n(z)\, \dd z=\int_{\bbR^d} J_{\delta_n}(z)\hat{f}_n\circ\hat{T}_n(\theta-z) \, \dd z \]
where $J_\delta$ is a standard mollifier and we have extended $\hat{f}_n \circ \hat{T}_n$ to be zero outside of $\mathcal{S}$.

\paragraph{Step 1.}
To prove part (a.) 
we use a change of variables to infer
\begin{equation*}
\begin{aligned}
\hat{\mathcal{E}}_{\varepsilon_n,n}(\hat{f}_n)&= \frac{1}{n^2\varepsilon^{p+d}_n}\sum_{i=1}^n\sum_{j=1}^n\hat{\eta}_{\theta_i}\bigg(\frac{\theta_j-\theta_i}{\varepsilon_n}\bigg)|\hat{f}_n(\theta_i)-\hat{f}_n(\theta_j)|^p\\     
& =\frac{1}{\varepsilon^{p+d}_n}\int_{\mathcal{S}_n}\int_{\mathcal{S}_n} \hat{\eta}_\theta\bigg(\frac{\tau-\theta}{\varepsilon_n}\bigg)|\hat{f}_n(\theta)-\hat{f}_n(\tau)|^p\,\dd\bbP_{\mathcal{S}_n}(\theta)\,\dd\bbP_{\mathcal{S}_n}(\tau) \\
& = \frac{1}{\varepsilon^{p+d}_n}\int_{\mathcal{S}}\int_{\mathcal{S}} \hat{\eta}_{\hat{T}_n(\theta)}\bigg(\frac{\hat{T}_n(\tau)-\hat{T}_n(\theta)}{\varepsilon_n}\bigg)|\hat{f}_n \circ \hat{T}_n (\theta)-\hat{f}_n \circ\hat{T}_n(\tau)|^p\,\dd\bbP_{\mathcal{S}}(\theta)\,\dd\bbP_{\mathcal{S}}(\tau) \\
\end{aligned}
\end{equation*}
Let $\xi_n=\frac{2\|\hat{T}_n-\Id \|_{\Lp{\infty}}}{\varepsilon_n}$, $\psi_n=\|\hat{T}_n-\Id\|_{\Lp{\infty}}$ so $\xi_n,\psi_n \to 0$.
By Assumption~\ref{item:continuity_assumption_eta_hat} in Theorem~\ref{thm:gamma_convergence_parameter_manifold} there exists $\alpha_{n}:=\alpha_{\xi_n,\psi_n} \to 1$, $c_{n}:=c_{\xi_n,\psi_n} \to 1$ as $n \to \infty$ such that if $\|a-b\|\leq\xi_n$ and $\|\theta'-\theta\|\leq \psi_n$ then
\begin{align}\label{eqn:hat_comparison}
\hat{\eta}_{\theta'}(a) \geq c_{n}\hat{\eta}_{\theta}(\alpha_{n}b).   
\end{align}

If we let $a:=\frac{\hat{T}_n(\theta)-\hat{T}_n(\tau)}{\varepsilon_n}$, $b:=\frac{\theta-\tau}{\varepsilon_n}$ and $\theta'=\hat{T}_n(\theta)$, then we have
\begin{equation*}
\|a-b\|=\bigg\|\frac{\hat{T}_n(\theta)-\hat{T}_n(\tau)+\tau-\theta}{\varepsilon_n} \bigg\| \leq \frac{2\|\hat{T}_n - \Id\|_{\Lp{\infty}}}{\varepsilon_n} =\xi_n    
\end{equation*}
\begin{equation*}
\|\theta-\theta'\|=\|\hat{T}_n(\theta)-\theta\|\leq \|\hat{T}_n-\Id\|_{\Lp{\infty}}=\psi_n    
\end{equation*}
and therefore from ~\eqref{eqn:hat_comparison} we get $\hat{\eta}_{\hat{T}_n(\theta)}\left(\frac{\hat{T}_n(\theta)-\hat{T}_n(\tau)}{\varepsilon_n} \right) \geq c_{n} \hat{\eta}_\theta\left(\frac{\alpha_{n}(\theta-\tau)}{\varepsilon_n}\right)$. 
Then, 
\begin{align}\label{eqn:energy_tilde_energy}
\begin{aligned}
\liminf_{n \to \infty} \hat{\mathcal{E}}_{\varepsilon_n,n}(\hat{f}_n) & \geq \liminf_{n \to \infty} \frac{c_n}{\varepsilon^{p+d}_n}\int_{\mathcal{S}}\int_{\mathcal{S}}  \hat{\eta}_\theta\bigg(\frac{\alpha_{n}(\theta-\tau)}{\varepsilon_n}\bigg) |\hat{f}_n \circ \hat{T}_n (\theta)-\hat{f}_n \circ\hat{T}_n(\tau)|^p\,\dd\bbP_{\mathcal{S}}(\theta)\,\dd\bbP_{\mathcal{S}}(\tau) \\     
& = \liminf_{n \to \infty}\frac{c_n}{\alpha^{p+d}_{n}\tilde{\varepsilon}^{p+d}_n} \int_{\mathcal{S}}\int_{\mathcal{S}}  \hat{\eta}_\theta\bigg(\frac{\theta-\tau}{\tilde{\varepsilon}_n}\bigg) |\hat{f}_n \circ \hat{T}_n (\theta)-\hat{f}_n \circ\hat{T}_n(\tau)|^p\,\dd\bbP_{\mathcal{S}}(\theta)\,\dd\bbP_{\mathcal{S}}(\tau) \\
& = \liminf_{n \to \infty}\frac{c_n}{\alpha^{p+d}_{n}}\hat{\mathcal{E}}_{\tilde{\varepsilon}_n,\infty}(\hat{f}_n \circ \hat{T}_n;\cS)
\end{aligned}    
\end{align}
where $\frac{{\varepsilon}_n}{\tilde{\varepsilon}_n}=\alpha_{n}$. 

\paragraph{Step 2.}
Fix $\mathcal{S}'$ to be an open set compactly contained in $\mathcal{S}$. There exists $\delta'>0$ such that $\mathcal{S}''=\bigcup_{\theta \in \mathcal{S}'}B(\theta,\delta')$ is contained in $\mathcal{S}$.
We assume that $\delta_n\leq \delta^\prime$.

Let us define by $a_n$ the approximation error in the energy $\hat{\mathcal{E}}_{\tilde{\eps}_n,\infty}(\hat{f}_n \circ \hat{T}_n)$ as follows, 
\begin{align}\label{eqn:approximation_error}
\begin{aligned}
a_n&=\frac{1}{\tilde{\varepsilon}^{p+d}_n}\int_{\mathcal{S}''} \int_{\mathcal{S}''}\int_{\bbR^d}J_{\delta_n}(z)\hat{\eta}_\theta\bigg(\frac{\theta-\tau}{\tilde{\varepsilon}_n}\bigg)|\hat{f}_n\circ\hat{T}_n(\theta)-\hat{f}_n\circ\hat{T}_n(\tau)|^p\\
&\qquad\qquad \times \big(\rho_\mathcal{S}(\theta)\rho_\mathcal{S}(\tau)-\rho_{\mathcal{S}}(\theta+z)\rho_\mathcal{S}(\tau+z)\big)\,\dd z\dd\theta\dd\tau
\end{aligned}
\end{align}
Now, 

\begin{align}
\hat{\mathcal{E}}_{\tilde{\varepsilon}_n,\infty}(\hat{f}_n\circ\hat{T}_n;\cS) & \geq \frac{1}{\tilde{\varepsilon}^{p+d}_n} \int_{\mathcal{S}''}\int_{\mathcal{S}''}\hat{\eta}_\theta \bigg(\frac{\theta-\tau}{\tilde{\varepsilon}_n}\bigg)|\hat{f}_n\circ\hat{T}_n(\theta)-\hat{f}_n\circ\hat{T}_n(\tau)|^p\rho_\mathcal{S}(\theta)\rho_\mathcal{S}(\tau)\,\dd\theta\,\dd\tau \notag \\   
& = \frac{1}{\tilde{\varepsilon}^{p+d}_n} \int_{\mathcal{S}''}\int_{\mathcal{S}''}\int_{\bbR^d}J_{\delta_n}(z)\hat{\eta}_\theta\bigg(\frac{\theta-\tau}{\tilde{\varepsilon}_n}\bigg)|\hat{f}_n\circ\hat{T}_n(\theta)-\hat{f}_n\circ\hat{T}_n(\tau)|^p\rho_\mathcal{S}(\theta)\rho_\mathcal{S}(\tau)\,\dd z\dd\theta\,\dd\tau \notag \\
& =a_n+\frac{1}{\tilde{\varepsilon}^{p+d}_n} \int_{\mathcal{S}''}\int_{\mathcal{S}''}\int_{\bbR^d}J_{\delta_n}(z)\hat{\eta}_\theta\bigg(\frac{\theta-\tau}{\tilde{\varepsilon}_n}\bigg)|\hat{f}_n\circ\hat{T}_n(\theta)-\hat{f}_n\circ\hat{T}_n(\tau)|^p \label{eqn:energy_energy_mollified} \\
& \qquad\qquad \times \rho_\mathcal{S}(\theta+z)\rho_\mathcal{S}(\tau+z)\,\dd z\,\dd \theta\,\dd\tau. \notag
\end{align}  

By Assumption \ref{item:continuity_assumption_eta_hat} there exists $\tilde{\alpha}_n:=\alpha_{\delta_n,0}$ and $\tilde{c}_n:=c_{\delta_n,0}$ with $\tilde{\alpha}_n\to 1$, $\tilde{c}_n\to 1$ and $\hat{\eta}_\theta(h)\geq \tilde{c}_n\hat{\eta}_{\theta^\prime}(\tilde{\alpha}_nh)$ for all $\|\theta-\theta^\prime\|<\delta_n$.
In particular, $\hat{\eta}_\theta(h)\geq \tilde{c}_n\hat{\eta}_{\theta+z}(\tilde{\alpha}_nh)$ for all $z\in \spt(J_{\delta_n})$.
Hence,
\begin{align}
\hat{\mathcal{E}}_{\tilde{\varepsilon}_n,\infty}(\hat{f}_n\circ\hat{T}_n;\cS) 
&\geq a_n + \frac{\tilde{c}_n}{\tilde{\varepsilon}^{p+d}_n} 
\int_{\mathcal{S}^{\prime\prime}} \int_{\mathcal{S}^{\prime\prime}} 
\int_{\bbR^d} 
J_{\delta_n}(z) 
\hat{\eta}_{\theta+z}\Bigl(\frac{\tilde{\alpha}_n(\theta-\tau)}{\tilde{\varepsilon}_n}\Bigr) 
\notag \\
&\qquad \times \bigl| \hat{f}_n\circ\hat{T}_n(\theta) - \hat{f}_n\circ\hat{T}_n(\tau) \bigr|^p 
\rho_\mathcal{S}(\theta+z) \rho_\mathcal{S}(\tau+z) \, \dd z \, \dd\theta \, \dd\tau 
\notag\\
&\geq a_n + \frac{\tilde{c}_n}{\tilde{\varepsilon}^{p+d}_n} 
\int_{\mathcal{S}^{\prime}} \int_{\mathcal{S}^{\prime}} 
\int_{\bbR^d} 
J_{\delta_n}(z) 
\hat{\eta}_{\tilde{\theta}}\Bigl(\frac{\tilde{\alpha}_n(\tilde{\theta}-\tilde{\tau})}{\tilde{\varepsilon}_n}\Bigr) 
\notag \\
&\qquad \times \bigl| \hat{f}_n\circ\hat{T}_n(\tilde{\theta}-z) - \hat{f}_n\circ\hat{T}_n(\tilde{\tau}-z) \bigr|^p 
\rho_\mathcal{S}(\tilde{\theta}) \rho_\mathcal{S}(\tilde{\tau}) \, \dd z \, \dd\tilde{\theta} \, \dd\tilde{\tau} 
\notag\\
&\geq a_n + \frac{\tilde{c}_n}{\tilde{\varepsilon}^{p+d}_n} 
\int_{\mathcal{S}^{\prime}} \int_{\mathcal{S}^{\prime}} 
\hat{\eta}_{\tilde{\theta}}\Bigl(\frac{\tilde{\alpha}_n(\tilde{\theta}-\tilde{\tau})}{\tilde{\varepsilon}_n}\Bigr)
\notag \\
&\qquad \times \biggl| \int_{\bbR^d} J_{\delta_n}(z) 
\bigl(\hat{f}_n\circ\hat{T}_n(\tilde{\theta}-z) - \hat{f}_n\circ\hat{T}_n(\tilde{\tau}-z)\bigr) \, \dd z \biggr|^p
\rho_\mathcal{S}(\tilde{\theta}) \rho_\mathcal{S}(\tilde{\tau}) \, \dd\tilde{\theta} \, \dd\tilde{\tau} 
\notag\\
&= a_n + \frac{\tilde{c}_n}{\tilde{\varepsilon}^{p+d}_n} 
\int_{\mathcal{S}^{\prime}} \int_{\mathcal{S}^{\prime}} 
\hat{\eta}_{\tilde{\theta}}\Bigl(\frac{\tilde{\alpha}_n(\tilde{\theta}-\tilde{\tau})}{\tilde{\varepsilon}_n}\Bigr)
|\hat{f}^{(\delta_n)}_n(\tilde{\theta}) - \hat{f}^{(\delta_n)}_n(\tilde{\tau})|^p 
\rho_\mathcal{S}(\tilde{\theta}) \rho_\mathcal{S}(\tilde{\tau}) \, \dd\tilde{\theta} \, \dd\tilde{\tau}
\label{eqn:energy_energy_mollified_two}
\end{align}

where the third inequality is from Jensen's inequality. 

We now estimate $a_n$ as follows:

\begin{align}
& \leq \frac{2 \delta_n\|\rho_\mathcal{S}\|_{\Lp{\infty}}\Lip(\rho_\mathcal{S})}{\tilde{\varepsilon}^{p+d}_n\inf_{\theta \in \mathcal{S}''}\rho^2_\mathcal{S}(\theta)}\int_{\mathcal{S}''}\int_{\mathcal{S}''}\int_{\bbR^d}J_{\delta_n}(z)\hat{\eta}_\theta\bigg(\frac{\theta-\tau}{\tilde{\varepsilon}_n}\bigg)|\hat{f}_n\circ\hat{T}_n(\theta)-\hat{f}_n\circ\hat{T}_n(\tau)|^p\rho_\mathcal{S}(\theta)\rho_\mathcal{S}(\tau)\,\dd z\,\dd\theta\,\dd\tau \notag \\
& = \frac{2\delta_n\|\rho_\mathcal{S}\|_{\Lp{\infty}}\Lip(\rho_\mathcal{S})}{\inf_{\theta \in \mathcal{S}''}\rho^2_\mathcal{S}(\theta)}\hat{\mathcal{E}}_{\tilde{\varepsilon}_n,\infty}(\hat{f}_n\circ \hat{T}_n;\mathcal{S}''). \label{eqn:error_approximation} 
\end{align}  

Continuing from~\eqref{eqn:energy_energy_mollified_two} using ~\eqref{eqn:error_approximation} we have,
\begin{align}
\begin{aligned}
& \liminf_{n \to \infty}\hat{\mathcal{E}}_{\tilde{\varepsilon}_n,\infty}(\hat{f}_n\circ\hat{T}_n;\mathcal{S}') \\
& \geq \liminf_{n \to \infty} 
\Biggl(
\frac{\tilde{c}_n}{\tilde{\varepsilon}^{p+d}_n} 
\int_{\mathcal{S}'} \int_{\mathcal{S}'} 
\hat{\eta}_{\tilde{\theta}}\Bigl(\frac{\tilde{\alpha}_n(\tilde{\theta}-\tilde{\tau})}{\tilde{\varepsilon}_n}\Bigr)
|\hat{f}^{(\delta_n)}_n(\tilde{\theta})-\hat{f}^{(\delta_n)}_n(\tilde{\tau})|^p 
\rho_\mathcal{S}(\tilde{\theta}) \rho_\mathcal{S} (\tilde{\tau}) 
\,\dd\tilde{\theta}\, \dd\tilde{\tau} 
\Biggr) \\
& \qquad \times 
\Biggl(
\frac{1}{1 + \frac{2 \delta_n \|\rho_\mathcal{S}\|_{\Lp{\infty}} \Lip(\rho_\mathcal{S})}
{\inf_{\theta \in \mathcal{S}''} \rho^2_\mathcal{S}(\theta)} }
\Biggr) \\
& = \liminf_{n \to \infty} 
\frac{\tilde{c}_n}{\tilde{\alpha}^{p+d}_n \hat{\varepsilon}^{p+d}_n} 
\int_{\mathcal{S}'} \int_{\mathcal{S}'} 
\hat{\eta}_{\tilde{\theta}}\Bigl(\frac{\tilde{\theta}-\tilde{\tau}}{\hat{\varepsilon}_n}\Bigr)
|\hat{f}^{(\delta_n)}_n(\tilde{\theta})-\hat{f}^{(\delta_n)}_n(\tilde{\tau})|^p 
\rho_\mathcal{S}(\tilde{\theta}) \rho_\mathcal{S}(\tilde{\tau}) 
\,\dd\tilde{\theta}\, \dd\tilde{\tau} \\
& = \liminf_{n \to \infty} \hat{\mathcal{E}}_{\hat{\varepsilon}_n,\infty}(\hat{f}^{(\delta_n)}_n,\mathcal{S}')
\end{aligned}
\end{align}

where $\hat{\varepsilon}_n=\frac{\tilde{\varepsilon}_n}{\tilde{\alpha}_n}$.

\paragraph{Step 3.}
To prove part (c.) of the lim-inf inequality in  \eqref{eqn:lim_inf_appendix_proof_proof} we define functionals $\hat{g}^{(\delta)}_n:\mathcal{S}' \to [0,+\infty)$ and $\hat{g}^{(\delta)}_{\infty}: \mathcal{S}' \to [0,+\infty)$ as follows,
\begin{align}\label{eqn:hat_g_discrete_mollified}
\hat{g}^{(\delta)}_{n}(\tilde{\theta})=\frac{1}{\hat{\varepsilon}^{p+d}_n}\int_\mathcal{S'}\hat{\eta}_{\tilde{\theta}}\bigg(\frac{\tilde{\theta}-\tilde{\tau}}{\hat{\varepsilon}_n}\bigg)|\hat{f}^{(\delta)}_n(\tilde{\theta})-\hat{f}^{(\delta)}_n(\tilde{\tau})|^p\rho_\mathcal{S}(\tilde{\tau})\textnormal{d}\tilde{\tau}
\end{align}
\begin{align}\label{eqn:hat_g_continnum_mollified}
\hat{g}^{(\delta)}_\infty(\tilde{\theta})=\rho_\mathcal{S}(\tilde{\theta})\int_{\bbR^d}\hat{\eta}_{\tilde{\theta}}(z)|\nabla \hat{f}^{(\delta)}(\tilde{\theta})\cdot z|^p\,\textnormal{d}z.   
\end{align}

Let $\tilde{\theta} \in \mathcal{S}^{\prime\prime\prime}$ where $\cS^{\prime\prime\prime} = \{\theta\in\cS^\prime\,:\, \dist(\theta,\partial \cS^\prime)>\delta_0\}$, then
\begin{align}\label{eqn:g_terms_bounds}
\begin{aligned}
&\big|\hat{g}^{(\delta_n)}_{n}(\tilde{\theta}) - \hat{g}^{(\delta_n)}_\infty(\tilde{\theta})\big| \\
& = \bigg|\int_{\frac{\tilde{\theta}-\mathcal{S}'}{\hat{\varepsilon}_n}}\hat{\eta}_{\tilde{\theta}}(z) \bigg|\frac{\hat{f}^{(\delta_n)}_n(\tilde{\theta}) -\hat{f}^{(\delta_n)}_n(\tilde{\theta}-\hat{\varepsilon}_nz)}{\hat{\varepsilon}_n} \bigg|^p \rho_\mathcal{S}(\tilde{\theta}-\hat{\varepsilon}_nz)\,\textnormal{d}z - \rho_\mathcal{S}(\tilde{\theta})\int_{\bbR^d}\hat{\eta}_{\tilde{\theta}}(z)|\nabla \hat{f}^{(\delta_n)}(\tilde{\theta})\cdot z|^p\,\textnormal{d}z \bigg| \\ 
& = \bigg|\int_{\bbR^d}\hat{\eta}_{\tilde{\theta}}(z) \bigg|\frac{\hat{f}^{(\delta_n)}_n(\tilde{\theta})-\hat{f}^{(\delta_n)}_n(\tilde{\theta}-\hat{\varepsilon}_nz)}{\hat{\varepsilon}_n} \bigg|^p \rho_\mathcal{S}(\tilde{\theta}-\hat{\varepsilon}_nz)\,\textnormal{d}z- \rho_\mathcal{S}(\tilde{\theta})\int_{\bbR^d}\hat{\eta}_{\tilde{\theta}}(z)|\nabla \hat{f}^{(\delta_n)}(\tilde{\theta})\cdot z|^p\,\textnormal{d}z \bigg| \\
& \leq \bigg|\int_{\bbR^d}\hat{\eta}_{\tilde{\theta}}(z) \bigg|\frac{\hat{f}^{(\delta_n)}_n(\tilde{\theta})-\hat{f}^{(\delta_n)}_n(\tilde{\theta}-\hat{\varepsilon}_nz)}{\hat{\varepsilon}_n} \bigg|^p \big(\rho_\mathcal{S}(\tilde{\theta}-\hat{\varepsilon}_nz)-\rho_\mathcal{S}(\tilde{\theta})\big)\,\textnormal{d}z \bigg| \\
& \qquad \qquad + \bigg| \rho_\mathcal{S}(\tilde{\theta})\int_{\bbR^d}\hat{\eta}_{\tilde{\theta}}(z)\bigg(\bigg|\frac{\hat{f}^{(\delta_n)}_n(\tilde{\theta})-\hat{f}^{(\delta_n)}_n(\tilde{\theta}-\hat{\varepsilon}_nz)}{\hat{\varepsilon}_n}\bigg|^p-|\nabla \hat{f}^{(\delta_n)}(\tilde{\theta})\cdot z|^p \bigg)\,\dd z\bigg| \\
& \leq \underbrace{\Lip(\rho_\mathcal{S})\hat{\varepsilon}_n \int_{\bbR^d}|z|\hat{\eta}_{\tilde{\theta}}(z)\bigg|\frac{\hat{f}^{(\delta_n)}_n(\tilde{\theta})-\hat{f}^{(\delta_n)}_n(\tilde{\theta}-\hat{\varepsilon}_nz)}{\tilde{\varepsilon}_n}\bigg|^p\,\dd z}_{=:A_n} \\ 
& \qquad \qquad + \underbrace{\rho_\mathcal{S}(\tilde{\theta})\int_{\bbR^d}\hat{\eta}_{\tilde{\theta}}(z)\bigg|\bigg|\frac{\hat{f}^{(\delta_n)}_n(\tilde{\theta})-\hat{f}^{(\delta_n)}_n(\tilde{\theta}-\hat{\varepsilon}_nz)}{\hat{\varepsilon}_n} \bigg|^p - |\nabla \hat{f}^{(\delta_n)}(\tilde{\theta})\cdot z|^p \bigg|\,\dd z}_{=:B_n}\\
\end{aligned}    
\end{align}
where the first equality follows from substitution $z=\frac{\tilde{\theta}-\tilde{\tau}}{\hat{\varepsilon}_n}$ and the second equality follows from the fact that if $z\not\in \frac{\tilde{\theta}-\cS^\prime}{\hat{\eps}_n}$ then $|z|>\frac{\delta_0}{\hat{\eps}_n} > R$, for $n$ sufficiently large, and so $\hat{\eta}_{\tilde{\theta}}(z) = 0$.

By Assumption \ref{item:compact_support_eta_hat}, there exists $R>0$ such that,
\begin{align}\label{eqn:a_n}
\begin{aligned}
A_n& \leq\frac{\Lip(\rho_\mathcal{S})\hat{\varepsilon}_n R}{\hat{\varepsilon}^d_n} \int_{\bbR^d}\hat{\eta}_{\tilde{\theta}}\bigg(\frac{\tilde{\theta}-\tilde{\tau}}{\hat{\varepsilon}_n}\bigg)\bigg|\frac{\hat{f}^{(\delta_n)}_n(\tilde{\theta})-\hat{f}^{(\delta_n)}_n(\tilde{\tau})}{\tilde{\varepsilon}_n} \bigg|^p \,\dd\tilde{\tau}\\
& \leq\frac{R}{\inf_{\tilde{\theta} \in \mathcal{S}'}\rho_\mathcal{S}(\tilde{\theta})}\Lip(\rho_\mathcal{S})\frac{\hat{\varepsilon}_n}{\hat{\varepsilon}^{p+d}_n}\int_{\mathcal{S}'}\hat{\eta}_{\tilde{\theta}}\bigg(\frac{\tilde{\theta}-\tilde{\tau}}{\hat{\varepsilon}_n}\bigg)|\hat{f}^{(\delta_n)}_n(\tilde{\theta})-\hat{f}^{(\delta_n)}_n(\tilde{\tau})|^p\rho_\mathcal{S}(\tilde{\tau})\,\dd\tau \\
& = \frac{R\Lip(\rho_\mathcal{S}) \hat{\varepsilon}_n}{\inf_{\tilde{\theta} \in \mathcal{S}'}\rho_\mathcal{S}(\tilde{\theta})}\hat{g}^{(\delta_n)}_{n}(\tilde{\theta})
\end{aligned}    
\end{align}
and for any $\kappa>0$ there exists $C_\kappa>0$ such that

\begin{align}
\begin{aligned}
B_n 
&\leq \rho_\mathcal{S}(\tilde{\theta}) \int_{\bbR^d} 
\hat{\eta}_{\tilde{\theta}}(z) \Biggl(
\kappa |\nabla \hat{f}^{(\delta_n)}(\tilde{\theta})\cdot z|^p \\
&\qquad + C_\kappa 
\Biggl| \frac{\hat{f}^{(\delta_n)}_n(\tilde{\theta})-\hat{f}^{(\delta_n)}_n(\tilde{\theta}-\hat{\varepsilon}_n z)}{\hat{\varepsilon}_n} 
- \nabla \hat{f}^{(\delta_n)}(\tilde{\theta})\cdot z \Biggr|^p
\Biggr)\, \dd z \\
&= \underbrace{\kappa \rho_\mathcal{S}(\tilde{\theta}) \int_{\bbR^d} 
\hat{\eta}_{\tilde{\theta}}(z) |\nabla \hat{f}^{(\delta_n)}(\tilde{\theta})\cdot z|^p \, \dd z}_{=: B_n'} \\
&\qquad + \underbrace{C_\kappa \rho_\mathcal{S}(\tilde{\theta}) \int_{\bbR^d} 
\hat{\eta}_{\tilde{\theta}}(z) 
\Biggl| \frac{\hat{f}^{(\delta_n)}_n(\tilde{\theta})-\hat{f}^{(\delta_n)}_n(\tilde{\theta}-\hat{\varepsilon}_n z)}{\hat{\varepsilon}_n} 
- \nabla \hat{f}^{(\delta_n)}(\tilde{\theta})\cdot z \Biggr|^p \, \dd z}_{=: B_n''}
\end{aligned}
\end{align}
where the first inequality follows from the fact: for all $\kappa > 0$ there exists $C_\kappa:=\frac{1}{\kappa^{p-1}}>0$ such that for any $a,b\in \bbR^d$ we have
\begin{align}\label{eqn:inequality}
|a|^p-|b|^p\leq \kappa|b|^p+C_\kappa|a-b|^p    
\end{align}
with $C_\kappa$ is independent of $a$ and $b$.

For $B_n^\prime$ we use
\begin{align*}
\|\nabla \hat{f}^{(\delta_n)}(\tilde{\theta}) \| & = \lda \nabla \int_{\bbR^d}J_{\delta_n}(\tilde{\theta}-\tilde{\tau})\hat{f}(\tilde{\tau})\,\dd \tilde{\tau} \rda \\
 & = \frac{1}{\delta^{d+1}_n} \lda \int_{\bbR^d}(\nabla J)\lp\frac{\tilde{\theta}-\tilde{\tau}}{\delta_n}\rp\hat{f}(\tilde{\tau})\,\dd\tilde{\tau}\rda \\
 & \leq \| \hat{f}\|_{\Lp{\infty}} \frac{1}{\delta_n^{d+1}} \int_{\bbR^d} \lda (\nabla J)\lp\frac{\tilde{\theta}-\tilde{\tau}}{\delta_n}\rp \rda \, \dd \tilde{\tau} \\
 & \leq \frac{C\|\hat{f}\|_{\Lp{\infty}}}{\delta_n}.
\end{align*}
So that,
\begin{align}\label{eqn:b_prime}
\begin{aligned}
B'_n & \leq \kappa \rho_\mathcal{S}(\tilde{\theta}) \int_{\bbR^d} \hat{\eta}_{\tilde{\theta}}(z) |z|^p \,\dd z \|\nabla \hat{f}^{(\delta_n)}(\tilde{\theta})\|^p \\
& = C \kappa \| \nabla \hat{f}^{(\delta_n)}(\tilde{\theta})\|^p  \\
& \leq \frac{C\kappa}{\delta^p_n}\|\hat{f}\|^p_{\Lp{\infty}}
\end{aligned}
\end{align}
where the first inequality follows H\"{o}lder's inequality. 
For $B_n^{\prime\prime}$ we can similarly bound 
\begin{align}\label{eqn:norm_b_1}
\begin{aligned}
\|D^2\hat{f}^{(\delta_n)}_n(\tilde{\theta})\| & \leq \frac{\|\hat{f}_n\|_{\Lp{\infty}}}{\delta^{d+2}_n} \int_{\bbR^d}\bigg\|(D^2J)\bigg(\frac{\tilde{\theta}-\tilde{\tau}}{\delta_n}\bigg) \bigg\| \, \dd\tilde{\tau} \\
& = \frac{\|\hat{f}_n\|_{\Lp{\infty}}}{\delta^{2}_n} \int_{\bbR^d}\big\|(D^2J)(\tilde{\theta}) \big\| \, \dd\tilde{\tau} \\
& =\frac{C\|\hat{f}_n\|_{\Lp{\infty}}}{\delta^2_n}
\end{aligned}    
\end{align}
and

\begin{align}\label{eqn:norm_b_2}
\begin{aligned}
\|\nabla\hat{f}^{(\delta_n)}_n(\tilde{\theta})-\nabla\hat{f}^{(\delta_n)}(\tilde{\theta}) \| & \leq \frac{1}{\delta^{d+1}_n}\int_{\bbR^d}\bigg\|\nabla J\bigg(\frac{\tilde{\theta}-\tilde{\tau}}{\delta_n}\bigg)\bigg\||\hat{f}_n\circ\hat{T}_n(\tilde{\tau})-\hat{f}(\tilde{\tau})|\,\dd\tilde{\tau}\\   
& \leq \frac{1}{\delta^{d+1}_n}\bigg(\int_{\bbR^d}\bigg\|\nabla J\bigg(\frac{\tilde{\theta}-\tilde{\tau}}{\delta_n}\bigg)\bigg\|^2\,\dd z \bigg)^{\frac{1}{2}}\|\hat{f}_n\circ\hat{T}_n-\hat{f}\|_{\Lp{2}(\mathcal{S})} \\
& = \frac{1}{\delta^{d+1}_n}\bigg(\int_{\bbR^d}\|\nabla J(z)\|^2\delta^d_n\,\dd z\bigg)^{\frac{1}{2}}\|\hat{f}_n\circ\hat{T}_n-\hat{f}\|_{\Lp{2}(\mathcal{S})} \\
& = \frac{1}{\delta^{d/2+1}_n}\|\nabla J\|_{\Lp{2}(\bbR^d)}\|\hat{f}_n\circ\hat{T}_n-\hat{f}\|_{\Lp{2}(\mathcal{S})}
\end{aligned}    
\end{align}
where the second inequality follows from Cauchy--Schwarz inequality and the third equality follows from substitution $z=\frac{\tilde{\theta}-\tilde{\tau}}{\delta_n}$.

Hence,

\begin{align}\label{eqn:b_double_prime}
\begin{aligned}
& \frac{B''_n}{C_\kappa \rho_\cS(\tilde{\theta})} \\ 
&= \int_{\bbR^d} 
\hat{\eta}_{\tilde{\theta}}(z) 
\Biggl| 
\frac{
\hat{f}^{(\delta_n)}_n(\tilde{\theta}) 
- \big(\hat{f}^{(\delta_n)}_n(\tilde{\theta}) 
- \hat{\varepsilon}_n z \cdot \nabla \hat{f}^{(\delta_n)}_n(\tilde{\theta}) 
+ \hat{\varepsilon}_n^2 z^T D^2 \hat{f}^{(\delta_n)}_n(\tilde{\theta}_z) z \big)
}{\hat{\varepsilon}_n}  - \nabla \hat{f}^{(\delta_n)}(\tilde{\theta}) \cdot z 
\Biggr|^p \, \dd z \\
&\leq 2^{p-1} 
\Biggl(
\|\nabla \hat{f}^{(\delta_n)}_n(\tilde{\theta}) - \nabla \hat{f}^{(\delta_n)}(\tilde{\theta})\|^p
\int_{\bbR^d} \hat{\eta}_{\tilde{\theta}}(z) |z|^p \, \dd z
+ \hat{\varepsilon}_n^p \|D^2 \hat{f}^{(\delta_n)}_n\|^p_{\Lp{\infty}} 
\int_{\bbR^d} \hat{\eta}_{\tilde{\theta}}(z) |z|^{2p} \, \dd z
\Biggr) \\
&\leq C \Biggl(
\|\nabla \hat{f}^{(\delta_n)}_n(\tilde{\theta}) - \nabla \hat{f}^{(\delta_n)}(\tilde{\theta})\|^p 
+ \hat{\varepsilon}_n^p \|D^2 \hat{f}^{(\delta_n)}_n\|^p_{\Lp{\infty}}
\Biggr) \\
&\leq C \Biggl(
\frac{1}{\delta_n^{(d/2+1)p}} \|\nabla J\|^p_{\Lp{2}(\bbR^d)} 
\|\hat{f}_n \circ \hat{T}_n - \hat{f}\|^p_{\Lp{2}(\mathcal{S})} 
+ C \frac{\hat{\varepsilon}_n^p}{\delta_n^{2p}} \|\hat{f}_n\|^p_{\Lp{\infty}}
\Biggr) \\
&= C \Biggl(
\frac{1}{\delta_n^{(d/2+1)p}} \|\hat{f}_n \circ \hat{T}_n - \hat{f}\|^p_{\Lp{2}(\mathcal{S})} 
+ \frac{\hat{\varepsilon}_n^p}{\delta_n^{2p}}
\Biggr)
\end{aligned}
\end{align}
where the first equality is obtained from Taylor's expansion of $\hat{f}^{(\delta_n)}_n(\tilde{\theta}-\hat{\varepsilon}_nz)$ upto the second order term, the first inequality follows from the fact: if $a,b \in \bbR$ and $p \geq 1$ then $|a+b|^p \leq 2^{p-1}(|a|^p+|b|^p)$ and last inequality uses  \eqref{eqn:norm_b_1} and  \eqref{eqn:norm_b_2}.

We can now bound \eqref{eqn:g_terms_bounds} by  \eqref{eqn:a_n},  \eqref{eqn:b_prime} and  \eqref{eqn:b_double_prime},
\begin{align}\label{eqn:final_bounds_g}
\big|\hat{g}^{(\delta_n)}_{n}(\tilde{\theta}) - \hat{g}^{(\delta_n)}_\infty(\tilde{\theta})\big| \leq C \hat{\varepsilon}_n\hat{g}^{(\delta_n)}_{n}(\tilde{\theta})+\frac{C\kappa}{\delta^p_n}\|\hat{f}\|^p_{\Lp{\infty}} + \frac{C_\kappa}{\delta^{(d/2+1)p}_n}\|\hat{f}_n\circ\hat{T}_n - \hat{f}\|^p_{\Lp{2}(\mathcal{S})}+\frac{C_\kappa\hat{\varepsilon}^p_n}{\delta^{2p}_n} 
\end{align}
By choosing $\kappa=\delta^{p+1}_n$ and $C_\kappa:=\frac{1}{\delta^{p^2-1}_n}$ as in \eqref{eqn:inequality} and letting $\delta_n \to 0$ sufficiently slowly i.e., $\frac{\|\hat{f}_n\circ\hat{T}_n-\hat{f}\|^p_{\Lp{2}(\mathcal{S})}}{\delta^{(d/2+1)p+p^2-1}_n} \to 0$, $\frac{\hat{\varepsilon}_n^p}{\delta^{p^2-1+2p}_n} \to 0$,
\begin{align}\label{eqn:g_n_infty}
\begin{aligned}
& \big|\hat{g}^{(\delta_n)}_{n}(\tilde{\theta}) - \hat{g}^{(\delta_n)}_\infty(\tilde{\theta})\big| \leq C\hat{\varepsilon}_n\hat{g}^{(\delta_n)}_{n}(\tilde{\theta})+C \\
& \implies \hat{g}^{(\delta_n)}_{n}(\tilde{\theta}) \leq \hat{g}^{(\delta_n)}_\infty(\tilde{\theta})+C\bigg(\hat{\varepsilon}_n\hat{g}^{(\delta_n)}_{n}(\tilde{\theta})+1\bigg) \\
\end{aligned}
\end{align}
where $C$ is independent of $\kappa$, $\tilde{\theta}$ and $n$.
Hence, we can rewrite~\eqref{eqn:final_bounds_g} with this choice of $\kappa$ and $\delta_n$ and using~\eqref{eqn:g_n_infty} to obtain 
\begin{equation*}
\begin{aligned}
|\hat{g}^{(\delta_n)}_{n}(\tilde{\theta})-\hat{g}^{(\delta_n)}_\infty(\tilde{\theta})| & \leq C\hat{\varepsilon}_n\bigg(\hat{g}^{(\delta_n)}_{\infty}(\tilde{\theta})+C+C\hat{\varepsilon}_n\hat{g}^{(\delta_n)}_{n}(\tilde{\theta})\bigg)+C\delta_n+\frac{\|\hat{f}_n\circ\hat{T}_n - \hat{f}\|^p_{\Lp{2}(\mathcal{S})}}{\delta^{(d/2+1)p+p^2-1}_n}+\frac{\hat{\varepsilon}^p_n}{\delta^{p^2-1+2p}_n}  \\    
& = C^2\hat{\varepsilon}^2_n\hat{g}^{(\delta_n)}_{n}(\tilde{\theta})+C\hat{\varepsilon}_n\hat{g}^{(\delta_n)}_\infty(\tilde{\theta})+C^2\hat{\varepsilon}_n+  C\delta_n +\frac{\|\hat{f}_n\circ\hat{T}_n - \hat{f}\|^p_{\Lp{2}(\mathcal{S})}}{\delta^{(d/2+1)p+p^2-1}_n}+\frac{\hat{\varepsilon}^p_n}{\delta^{p^2-1+2p}_n}  \\
& \leq C^2\hat{\varepsilon}^2_n\bigg(\hat{g}^{(\delta_n)}_{\infty}(\tilde{\theta})+C\hat{\varepsilon}_n\hat{g}^{(\delta_n)}_{n}(\tilde{\theta})+C\bigg)+C\hat{\varepsilon}_n\bigg(\hat{g}^{(\delta_n)}_\infty(\tilde{\theta})+C\bigg)+  C\delta_n\\
& \qquad \qquad +\frac{\|\hat{f}_n\circ\hat{T}_n - \hat{f}\|^p_{\Lp{2}(\mathcal{S})}}{\delta^{(d/2+1)p+p^2-1}_n}+\frac{\hat{\varepsilon}^p_n}{\delta^{p^2-1+2p}_n} \\
& = C^3\hat{\varepsilon}^3_n\hat{g}^{(\delta_n)}_{n}(\tilde{\theta})+C^2\hat{\varepsilon}^2_n\bigg(\hat{g}^{(\delta_n)}_\infty(\tilde{\theta})+C\bigg)+C\hat{\varepsilon}_n\bigg(\hat{g}^{(\delta_n)}_\infty(\tilde{\theta})+C\bigg)+  C\delta_n \\
& \qquad \qquad +\frac{\|\hat{f}_n\circ\hat{T}_n - \hat{f}\|^p_{\Lp{2}(\mathcal{S})}}{\delta^{(d/2+1)p+p^2-1}_n}+\frac{\hat{\varepsilon}^p_n}{\delta^{p^2-1+2p}_n} \\
& \leq (C\hat{\varepsilon}_n)^k\hat{g}^{(\delta_n)}_{n}(\tilde{\theta})+C\hat{\varepsilon}_n\sum_{i=0}^{k-2}(C\hat{\varepsilon}_n)^i(\hat{g}^{(\delta_n)}_\infty(\tilde{\theta})+C)+C\delta_n\\
& +\frac{\|\hat{f}_n\circ\hat{T}_n - \hat{f}\|^p_{\Lp{2}(\mathcal{S})}}{\delta^{(d/2+1)p+p^2-1}_n}+\frac{\hat{\varepsilon}^p_n}{\delta^{p^2-1+2p}_n} \\
& \leq C\hat{\varepsilon}_n(\hat{g}^{(\delta_n)}_\infty(\tilde{\theta})+C)+C\delta_n +\frac{\|\hat{f}_n\circ\hat{T}_n - \hat{f}\|^p_{\Lp{2}(\mathcal{S})}}{\delta^{(d/2+1)p+p^2-1}_n}+\frac{\hat{\varepsilon}^p_n}{\delta^{p^2-1+2p}_n}.
\end{aligned}    
\end{equation*}
By generalizing the third inequality with a telescoping series and assuming $C\hat{\varepsilon}_n \leq \frac{1}{2}$, and the fourth inequality follows from evaluating the limit as $k \to \infty$.

Hence, we have
\begin{align}\label{eqn:discrete_energy_continuum_mollified}
\begin{aligned}
\liminf_{n\to \infty}\hat{\mathcal{E}}_{\hat{\varepsilon}_n,\infty}(\hat{f}^{(\delta_n)}_n;\mathcal{S}') & \geq \liminf_{n \to \infty} \int_{\mathcal{S}^{\prime\prime\prime}}\hat{g}^{(\delta_n)}_{n}(\tilde{\theta})\rho_\mathcal{S}(\tilde{\theta})\,\dd\tilde{\theta} \\
& \geq\liminf_{n \to \infty} \frac{\tilde{c}_{n}}{\tilde{\alpha}^{p+d}_n}\bigg(\int_{\mathcal{S}^{\prime\prime\prime}} \hat{g}^{(\delta_n)}_\infty(\tilde{\theta})\rho_\mathcal{S}(\tilde{\theta})\,\dd\tilde{\theta} -C\hat{\varepsilon}_n\int_{\mathcal{S}^{\prime\prime\prime}}\hat{g}^{(\delta_n)}_\infty(\tilde{\theta})\rho_\mathcal{S}(\tilde{\theta})\,\dd\tilde{\theta} - C^2\hat{\varepsilon}_n \\
& \qquad \qquad -C\delta_n-\frac{1}{\delta^{(d/2+1)p+p^2-1}_n}\|\hat{f}_n\circ\hat{T}_n - \hat{f}\|^p_{\Lp{2}(\mathcal{S})}-\frac{\hat{\varepsilon}^p_n}{\delta^{2p+p^2-1}_n} \bigg)\\
& \geq \liminf_{n \to \infty}\int_{\mathcal{S}^{\prime\prime\prime}}\hat{g}^{(\delta_n)}_\infty(\tilde{\theta})\rho_\mathcal{S}(\tilde{\theta})\,\dd\tilde{\theta}\\
& = \liminf_{n \to \infty} \hat{\mathcal{E}}_{\infty}(\hat{f}^{(\delta_n)};\mathcal{S}^{\prime\prime\prime}) \\
\end{aligned}    
\end{align}
where the second inequality follows from the fact as $n \to \infty$, $\hat{\varepsilon}_n \to 0$ and $\delta_n \to 0$.

\paragraph{Step 4.}
Finally, to prove part (d.) of the lim-inf inequality in  \eqref{eqn:lim_inf_appendix_proof_proof} we use the inequality in  \eqref{eqn:inequality} to show that
\begin{align}\label{eqn:final_k_ck}
\begin{aligned}
& \bigg| \int_{\bbR^d}\hat{\eta}_{\tilde{\theta}}(z)\bigg(|\nabla\hat{f}^{(\delta_n)}(\tilde{\theta})\cdot z|^p - |\nabla\hat{f}(\tilde{\theta})\cdot z|^p \bigg)\,\dd z \bigg| \\
& \qquad \qquad \leq \kappa \int_{\bbR^d} \hat{\eta}_{\tilde{\theta}}(z)|\nabla \hat{f}(\tilde{\theta})\cdot z|^p\,\dd z
+ C_\kappa \int_{\bbR^d}\hat{\eta}_{\tilde{\theta}}(z) |\big(\nabla\hat{f}^{(\delta_n)}(\tilde{\theta})-\nabla\hat{f}(\tilde{\theta})\big)\cdot z|^p\,\dd z\\
& \qquad \qquad \leq \kappa \int_{\bbR^d}\hat{\eta}_{\tilde{\theta}}(z)\|z\|^p\,\dd z\|\nabla \hat{f}(\tilde{\theta})\|^p 
+C_\kappa \| \nabla \hat{f}^{(\delta_n)}(\tilde{\theta})-\nabla \hat{f}(\tilde{\theta})\|^p\int_{\bbR^d}\hat{\eta}_{\tilde{\theta}}(z)\|z\|^p\,\dd z.
\end{aligned}    
\end{align}

We obtain
\begin{align}\label{eqn:energy_continuum_mollified_continuum}
\begin{aligned}
\liminf_{n \to \infty}\hat{\mathcal{E}}_\infty(\hat{f}^{(\delta_n)};\mathcal{S}^{\prime\prime\prime})  
&= \liminf_{n \to \infty}\int_{\mathcal{S}^{\prime\prime\prime}}g^{(\delta_n)}_\infty(\tilde{\theta})\rho_\mathcal{S}(\tilde{\theta})\,\dd \tilde{\theta} \\
&= \liminf_{n \to \infty} \int_{\mathcal{S}^{\prime\prime\prime}}\int_{\mathbb{R}^d} \hat{\eta}_{\tilde{\theta}}(z)\left|\nabla \hat{f}^{(\delta_n)}(\tilde{\theta})\cdot z\right|^p\,\rho_\mathcal{S}^2(\tilde{\theta})\,\dd z\,\dd\tilde{\theta} \\
&\geq \int_{\mathcal{S}^{\prime\prime\prime}}\int_{\mathbb{R}^d} \hat{\eta}_{\tilde{\theta}}(z)\left|\nabla \hat{f}(\tilde{\theta})\cdot z\right|^p \rho_\mathcal{S}^2(\tilde{\theta})\,\dd z\,\dd\tilde{\theta} \\
&\quad - \kappa \int_{\mathcal{S}^{\prime\prime\prime}} \int_{\mathbb{R}^d} \hat{\eta}_{\tilde{\theta}}(z) \left|\nabla \hat{f}(\tilde{\theta})\cdot z\right|^p\, \rho_\mathcal{S}^2(\tilde{\theta})\,\dd z \,\dd\tilde{\theta} \\
&\quad - C_\kappa \limsup_{n \to \infty} \int_{\mathcal{S}^{\prime\prime\prime}} \int_{\mathbb{R}^d} \hat{\eta}_{\tilde{\theta}}(z) \left|\left(\nabla \hat{f}^{(\delta_n)}(\tilde{\theta}) - \nabla \hat{f}(\tilde{\theta})\right)\cdot z\right|^p \rho^2_\mathcal{S}(\tilde{\theta})\,\dd z\,\dd\tilde{\theta}.
\end{aligned}
\end{align}
 
We can write
\[
\limsup_{n \to \infty} \int_{\mathbb{R}^d} J_{\delta_n}(w) \left\|\nabla \hat{f}(\cdot - w) - \nabla \hat{f}(\cdot)\right\|^p_{\Lp{p}(\mathcal{S}')} \,\dd w \to 0 \quad \text{as } n \to \infty,
\]
where \(\nabla \hat{f} \) denotes the gradient of the mollified function. This follows from the strong continuity of mollification in \( \Lp{p} \), and the dominated convergence theorem, since \( \nabla \hat{f} \in \Lp{p}(\mathcal{S}') \).
So,
\begin{align*}
& \limsup_{n\to\infty} \int_{\cS^{\prime\prime\prime}} \int_{\bbR^d} \hat{\eta}_{\tilde{\theta}}(z) \la \lp \nabla \hat{f}^{(\delta_n)}(\tilde{\theta})-\nabla \hat{f}(\tilde{\theta})\rp \cdot z\ra^p \rho_\cS^2(\tilde{\theta}) \, \dd z \dd \tilde{\theta} \\
& \qquad \qquad \leq \limsup_{n\to\infty} \int_{\cS^{\prime\prime\prime}} \int_{\bbR^d} \hat{\eta}_{\tilde{\theta}}(z) \|z\|^p \,\dd z \lda \int_{\bbR^d} J_{\delta_n}(w)\lp \nabla \hat{f}(\tilde{\theta}-w) - \nabla \hat{f}(\tilde{\theta})\rp \, \dd w \rda^p \rho_\cS^2(\tilde{\theta}) \, \dd \tilde{\theta} \\
& \qquad \qquad \leq C \limsup_{n\to\infty} \int_{\cS^{\prime\prime\prime}} \int_{\bbR^d} J_{\delta_n}(w) \lda \nabla \hat{f}(\tilde{\theta}-w) - \nabla \tilde{f}(\tilde{\theta}) \rda^p \, \dd w \dd \tilde{\theta} \\
& \qquad \qquad =0.
\end{align*}
Therefore, 
\[
\liminf_{n \to \infty} \hat{\mathcal{E}}_\infty(\hat{f}^{(\delta_n)}; \mathcal{S}^{\prime\prime\prime}) \geq (1-\kappa)\hat{\mathcal{E}}_\infty(\hat{f}; \mathcal{S}^{\prime\prime\prime}).
\]

Taking the limit \( \kappa \to 0 \) 
we obtain 
\[
\liminf_{n \to \infty} \hat{\mathcal{E}}_\infty(\hat{f}^{(\delta_n)}; \mathcal{S}^{\prime\prime\prime}) \geq \hat{\mathcal{E}}_\infty(\hat{f}; \mathcal{S}^{\prime\prime\prime}).
\]

Finally, we let $\cS^{\prime\prime\prime}\nearrow \cS$ and applying the monotone convergence theorem implies~\eqref{eqn:lim_inf_appendix_proof}.
\end{proof}

\begin{lemma}[lim sup-inequality]\label{lemma:lim-sup_parameter_manifold}
Under the same conditions as Theorem \ref{thm:gamma_convergence_parameter_manifold} with probability one, for any function $\hat{f}\in \Lp{p}(\bbP_\mathcal{S})$ with $\|\hat{f}\|_{\Lp{\infty}(\bbP_\mathcal{S})}<\infty$ there exists a sequence $\hat{f}_n \to \hat{f}$ in $\TL_{d_\mathcal{S}}(\mathcal{S})$ with $\|\hat{f}_n\|_{\Lp{\infty}(\bbP_{\mathcal{S}_n})}<\infty$ such that 
\begin{align}\label{eqn:lim_sup}
\begin{aligned}
\limsup_{n \to \infty} \hat{\mathcal{E}}_{\varepsilon_n,n}(\hat{f}_n)  \leq \hat{\mathcal{E}}_{\infty}(\hat{f}). 
\end{aligned}
\end{align}
Moreover, if $\hat{f} \in \textnormal{C}^{2}(\mathcal{S})$, then $\hat{f}_n=\hat{f}\lfloor_{\mathcal{S}_n}$ satisfies the inequality in \eqref{eqn:lim_sup}.
\end{lemma}

\begin{proof}  We note that since \(\mathrm{C}^2(\mathcal{S})\) functions are dense in \(\Lp{p}(\mathcal{S})\), it suffices to prove the \(\limsup\) inequality in~\eqref{eqn:lim_sup} for \(\hat{f} \in \mathrm{C}^2(\mathcal{S})\). For such \(\hat{f}\), we define \(\hat{f}_n := \hat{f}\lfloor_{\mathcal{S}_n} \in \Lp{p}(\mathcal{S}_n)\), interpreted as the restriction of \(\hat{f}\) onto \(\mathcal{S}_n\). The proof (adapted from \cite{Nicolas_pointcloud}) has four main steps: we first show that $\hat{f}_n\to \hat{f}$ in $\TLp{p}$ and then
\begin{align}\label{eqn:lim_sup_inequality_parameter}
\begin{aligned}
\limsup_{n \to \infty} \hat{\mathcal{E}}_{\varepsilon_n,n}(\hat{f}_n) \stackrel{(a.)}{\leq} \limsup_{n \to \infty} \hat{\mathcal{E}}_{\tilde{\varepsilon}_n,\infty}(\hat{f}_n \circ \hat{T}_n) \stackrel{(b.)}{\leq} \limsup_{n \to \infty} \hat{\mathcal{E}}_{\tilde{\varepsilon}_n,\infty}(\hat{f}) \stackrel{(c.)}{\leq}  \hat{\mathcal{E}}_{\infty}(\hat{f}) 
\end{aligned}
\end{align}
where $\hat{\cE}_{\eps,\infty} = \hat{\cE}_{\eps,\infty}(\cdot;\cS)$ is defined by~\eqref{eq:hatEinftyS'} and $\tilde{\eps}_n=\frac{\eps_n}{\hat{\alpha}_n}$ (with $\hat{\alpha}_n$ defined later).

\paragraph{Step 1.} Let $\hat{T}_n$ be the transport map satisfying $\hat{T}_{n\#}\bbP_{\cS} = \bbP_{\cS_n}$ and $\|\hat{T}_n-\Id\|_{\Lp{\infty}}\leq \hat{C} q_d(n)$ (as in Theorem~\ref{thm:rates_convergence}).
Such a map exists with probability one, and we continue the proof under the assumption that such a sequence of maps exist.

Since $\hat{f}\in \Ck{2}(\cS)$, 
$\|\hat{f}_n \circ \hat{T}_n - \hat{f}\|_{\Lp{p}(\cS)}=\|\hat{f}\circ \hat{T}_n-\hat {f}\|_{\Lp{p}(\cS)}\leq \Lip(\hat{f})\|\hat{T}_n-\Id\|_{\Lp{\infty}} \to 0$, and therefore $\hat{f}_n\to \hat{f}$ in $\TLp{p}$. 

\paragraph{Step 2.}
In order to prove part $(a.)$ of the lim-sup inequality in  \eqref{eqn:lim_sup_inequality_parameter}. We show that there exists $\hat{\alpha}_n \to 1$, $\hat{c}_{n} \to 1$ such that, 
\begin{align}\label{eqn:hat_eta_comparison}
\hat{\eta}_{\hat{T}_n(\theta)}\bigg(\frac{\hat{T}_n(\theta)-\hat{T}_n(\tau)}{\varepsilon_n}\bigg) \leq \hat{c}_{n}\hat{\eta}_\theta\bigg(\frac{\hat{\alpha}_n(\theta-\tau)}{\varepsilon_n}\bigg) 
\end{align}
for every $(\theta,\tau)\in \mathcal{S}\times \mathcal{S}$.

In order to show~\eqref{eqn:hat_eta_comparison} we use the following subclaim which we prove below: for all $\hat{\xi},\hat{\psi}>0$ sufficiently small there exists $\hat{\alpha} >0$, $\hat{c}>0$ such that 
\begin{align}\label{eqn:a_hat_b_comparison}
\|\hat{a}-\hat{b}\|\leq \hat{\xi} \mbox{ and } \|\theta'-\theta\|\leq\hat{\psi}  \implies \hat{c}\hat{\eta}_{\theta}(\hat{\alpha}\hat{b}) \geq \hat{\eta}_{\theta'}(\hat{a}).    
\end{align}
and $\hat{\alpha} \to 1$, $\hat{c} \to 1$, as $\hat{\xi},\hat{\psi} \to 0$.

Then,~\eqref{eqn:hat_eta_comparison} can be obtained as follows: for any $n \in \bbN$ take
\begin{equation*}
\hat{a}:=\frac{\hat{T}_n(\theta)-\hat{T}_n(\tau)}{\varepsilon_n}, \quad \hat{b}:=\frac{\theta-\tau}{\varepsilon_n}, \quad \hat{\xi}:=\frac{2\| \hat{T}_n - \Id \|_{\Lp{\infty}}}{\varepsilon_n},    
\end{equation*}
\begin{equation*}
\theta'=\hat{T}_n(\theta) \quad \hat{\psi}:=\|\hat{T}_n-\Id\|_{\Lp{\infty}},    
\end{equation*}

To prove the subclaim, let $\hat{\xi},\hat{\psi} > 0$ and assume that $\|\hat{a}-\hat{b}\| \leq \hat{\xi}$, $\|\theta'-\theta\| \leq \hat{\psi}$. Let $\xi=2\hat{\xi},\psi=\hat{\psi}$ and $c=c_{\xi,\psi}$, $\alpha=\alpha_{\xi,\psi}$ be as in Assumption~\ref{item:continuity_assumption_eta_hat}. Then, $\hat{\eta}_\theta(h_1)\geq c \hat{\eta}_{\theta'}(\alpha h_2)$ for any $\|h_1-h_2\| \leq \xi$.

Now, set $h_1=\frac{\hat{b}}{\alpha}$, $h_2=\frac{\hat{a}}{\alpha}$ so that $\|h_1-h_2\|=\frac{1}{\alpha}\|\hat{b}-\hat{a}\|\leq \frac{\hat{\xi}}{\alpha}\leq 2 \hat{\xi}=\xi$, if $\alpha \geq \frac{1}{2}$. 
For $\hat{\xi},\hat{\psi}$ small then $\xi,\psi$ are also small which implies $c,\alpha$ are close to 1.
In particular, for $\hat{\xi},\hat{\psi}$ sufficiently small, we have $\alpha \geq \frac{1}{2}$ and $\hat{\eta}_\theta\left(\frac{\hat{b}}{\alpha}\right) \geq c\hat{\eta}_{\theta'}(\hat{a})$. Plugging $\hat{c}=\frac{1}{c}$ and $\hat{\alpha}=\frac{1}{\alpha}$ completes the claim. 

Hence, it follows for all $n \in \bbN$,
\begin{align}\label{eqn:first_inequality}
\begin{aligned}
& \frac{1}{{\varepsilon}^{p+d}_n}\int_{\mathcal{S}}\int_{\mathcal{S}}\hat{\eta}_{\hat{T}_n(\theta)}\bigg(\frac{\hat{T}_n(\theta)-\hat{T}_n(\tau)}{\varepsilon_n}\bigg)|\hat{f}_n\circ\hat{T}_n(\theta) - \hat{f}_n\circ\hat{T}_n(\tau)|^p\,\dd\bbP_\mathcal{S}(\theta)\,\dd\bbP_\mathcal{S}(\tau)\\
& \leq \frac{\hat{c}_n}{{\varepsilon}^{p+d}_n}\int_{\mathcal{S}}\int_{\mathcal{S}}\hat{\eta}_\theta\bigg(\frac{\hat{\alpha}_n(\theta-\tau)}{\varepsilon_n}\bigg)|\hat{f}_n\circ\hat{T}_n(\theta) - \hat{f}_n\circ\hat{T}_n(\tau)|^p\,\dd\bbP_\mathcal{S}(\theta)\,\dd\bbP_\mathcal{S}(\tau).
\end{aligned}
\end{align}

We conclude that
\begin{align}\label{eqn:part_a_lim_sup_parameter}
\begin{aligned}
\limsup_{n\to\infty}\hat{\mathcal{E}}_{\varepsilon_n,n}(\hat{f}_n)& =\limsup_{n\to\infty}\frac{1}{{\varepsilon}^{p+d}_n}\int_{\mathcal{S}\times\mathcal{S}}\hat{\eta}_{\hat{T}_n(\theta)}\bigg(\frac{\hat{T}_n(\tau)-\hat{T}_n(\theta)}{\varepsilon_n}\bigg)|\hat{f}_n\circ\hat{T}_n(\theta)-\hat{f}_n\circ\hat{T}_n(\tau)|^p\,\dd\bbP_\mathcal{S}(\theta)\,\dd\bbP_\mathcal{S}(\tau)\\    
& \leq \limsup_{n \to \infty}\frac{1}{\hat{\alpha}^{p+d}_{n}\tilde{\varepsilon}^{p+d}_n}\int_{\mathcal{S}\times \mathcal{S}}\hat{\eta}_\theta\bigg(\frac{\theta-\tau}{\tilde{\varepsilon}_n}\bigg)|\hat{f}_n\circ\hat{T}_n(\theta)-\hat{f}_n\circ\hat{T}_n(\tau)|^p\,\dd\bbP_\mathcal{S}(\theta)\,\dd\bbP_\mathcal{S}(\tau) \\
& = \limsup_{n \to \infty}  \hat{\mathcal{E}}_{\tilde{\varepsilon}_n,\infty}(\hat{f}_n\circ\hat{T}_n)\\
\end{aligned}    
\end{align}
where the first inequality is obtained from~\eqref{eqn:first_inequality} and $\frac{\varepsilon_n}{\tilde{\varepsilon}_n}=\hat{\alpha}_n$. 
Finally, using the fact $\hat{\alpha}_n \to 1$ as $n \to \infty$ we get the last equality.

\paragraph{Step 3.}

We have for any $\kappa>0$,
\begin{align}\label{eqn:part_b_lim_sup_proof}
\begin{aligned}
\hat{\mathcal{E}}_{\tilde{\varepsilon}_n,\infty}(\hat{f}_n\circ\hat{T}_n) & = \frac{1}{\tilde{\varepsilon}^{p+d}_n}\int_{\mathcal{S}\times\mathcal{S}}\hat{\eta}_\theta\bigg(\frac{\theta-\tau}{\tilde{\varepsilon}_n}\bigg)|\hat{f}_n\circ\hat{T}_n(\theta)-\hat{f}_n\circ\hat{T}_n(\tau)|^p\rho_\mathcal{S}(\theta)\rho_\mathcal{S}(\tau)\,\dd\theta\,\dd\tau \\
& =\hat{\mathcal{E}}_{\tilde{\varepsilon}_n,\infty}(\hat{f})+ \frac{1}{\tilde{\varepsilon}^{p+d}_n}\int_{\mathcal{S}\times\mathcal{S}}\hat{\eta}_\theta\bigg(\frac{\theta-\tau}{\tilde{\varepsilon}_n}\bigg)\bigg(|\hat{f}_n\circ\hat{T}_n(\theta)-\hat{f}_n\circ\hat{T}_n(\tau)|^p-|\hat{f}(\theta)-\hat{f}(\tau)|^p\bigg)\\
&\qquad \qquad \rho_\mathcal{S}(\theta)\rho_\mathcal{S}(\tau)\,\dd\theta\,\dd\tau \\
& \leq \hat{\mathcal{E}}_{\tilde{\varepsilon}_n,\infty}(\hat{f})+ \frac{1}{\tilde{\varepsilon}^{p+d}_n}\int_{\mathcal{S}\times\mathcal{S}}\hat{\eta}_\theta\bigg(\frac{\theta-\tau}{\tilde{\varepsilon}_n}\bigg)\bigg(C_\kappa\big|\hat{f}_n\circ\hat{T}_n(\theta)-\hat{f}_n\circ\hat{T}_n(\tau)-\hat{f}(\theta)+\hat{f}(\tau)\big|^p\\
&\qquad\qquad+\kappa|\hat{f}(\theta)-\hat{f}(\tau)|^p\bigg)\rho_\mathcal{S}(\theta)\rho_\mathcal{S}(\tau)\,\dd\theta\,\dd\tau \\
& \leq (1+\kappa)\hat{\mathcal{E}}_{\tilde{\varepsilon}_n,\infty}(\hat{f})+\frac{C_\kappa2^{p}\Lip(\hat{f})^p}{\tilde{\varepsilon}^{p+d}_n}\|\hat{T}_n-\Id\|^p_{\Lp{\infty}}\int_{\mathcal{S}\times\mathcal{S}}\hat{\eta}_\theta\bigg(\frac{\theta-\tau}{\tilde{\varepsilon}_n}\bigg)\rho_\mathcal{S}(\theta)\rho_\mathcal{S}(\tau)\,\dd\theta\,\dd\tau  \\
\end{aligned}    
\end{align}
where the inequality in the third line follows from   \eqref{eqn:inequality}.
Hence,
\begin{align}\label{eqn:part_b_lim_sup_paramter}
\limsup_{n\to \infty} \hat{\mathcal{E}}_{\tilde{\varepsilon}_n,\infty}(\hat{f}_n\circ\hat{T}_n)    \leq \limsup_{n \to \infty} (1+\kappa)\hat{\mathcal{E}}_{\tilde{\varepsilon}_n,\infty}(\hat{f})
\end{align}
where the second term in  \eqref{eqn:part_b_lim_sup_proof} tends to 0 as $n \to \infty$ since $\|\hat{T}_n-\Id\|^p_{\Lp{\infty}} \to 0$.
Since this is true for all $\kappa > 0$, we let $\kappa \to 0$.
\paragraph{Step 4.}
Finally, to prove part $(c.)$ of  \eqref{eqn:lim_sup_inequality_parameter} we apply a second-order Taylor expansion of the function \(\hat{f}\) around \(\theta\): for each fixed \(\theta\), there exists \(\theta_\tau\) on the segment joining \(\theta\) and \(\tau\) such that
\[
|\hat{f}(\theta) - \hat{f}(\tau)|^p = \left| (\tau - \theta) \cdot \nabla \hat{f}(\theta) + \frac{1}{2} (\tau - \theta)^T D^2 \hat{f}(\theta_\tau) (\tau - \theta) \right|^p.
\]
Substituting into \(\hat{\mathcal{E}}_{\tilde{\varepsilon}_n,\infty}(\hat{f})\), we get
\[
\hat{\mathcal{E}}_{\tilde{\varepsilon}_n,\infty}(\hat{f}) = \frac{1}{\tilde{\varepsilon}_n^{p+d}} \int_{\mathcal{S}} \int_{\mathcal{S}} \hat{\eta}_\theta \left( \frac{\theta - \tau}{\tilde{\varepsilon}_n} \right) \left| (\tau - \theta) \cdot \nabla \hat{f}(\theta) + \frac{1}{2} (\tau - \theta)^T D^2 \hat{f}(\theta_\tau) (\tau - \theta) \right|^p \rho_{\mathcal{S}}(\theta) \rho_{\mathcal{S}}(\tau) \, \dd\tau \,\dd\theta.
\]

Using the triangle inequality and binomial expansion for \(p \in \mathbb{N}\), we obtain the estimate
\[
|a + b|^p \leq (|a|+|b|)^p = \sum_{k=0}^p \binom{p}{k} |a|^k |b|^{p-k} = |a|^p + \sum_{k=0}^{p-1} \binom{p}{k} |a|^k |b|^{p-k},
\]
which we apply to 
\[
a := (\tau - \theta) \cdot \nabla \hat{f}(\theta), \quad b := \frac{1}{2} (\tau - \theta)^T D^2 \hat{f}(\theta_\tau) (\tau - \theta).
\]
This gives
\[
\begin{aligned}
& \hat{\mathcal{E}}_{\tilde{\varepsilon}_n,\infty}(\hat{f}) 
 \leq \frac{1}{\tilde{\varepsilon}_n^{p+d}} \int_{\mathcal{S}} \int_{\mathcal{S}} \hat{\eta}_\theta \left( \frac{\theta - \tau}{\tilde{\varepsilon}_n} \right) |(\tau - \theta) \cdot \nabla \hat{f}(\theta)|^p \rho_{\mathcal{S}}(\theta) \rho_{\mathcal{S}}(\tau) \, \dd\tau\, \dd\theta\\
& + \sum_{k=0}^{p-1} \binom{p}{k} \frac{1}{\tilde{\varepsilon}_n^{p+d}} \int_{\mathcal{S}} \int_{\mathcal{S}} \hat{\eta}_\theta \left( \frac{\theta - \tau}{\tilde{\varepsilon}_n} \right) |(\tau - \theta) \cdot \nabla \hat{f}(\theta)|^{k} |\frac{1}{2} (\tau - \theta)^T D^2 \hat{f}(\theta_\tau) (\tau - \theta)|^{p-k} \rho_{\mathcal{S}}(\theta) \rho_{\mathcal{S}}(\tau) \, \dd\tau\, \dd\theta.
\end{aligned}
\]
By the change of variables \( z = \frac{\theta - \tau}{\tilde{\varepsilon}_n} \) and using $\rho_\mathcal{S}(\theta+\tilde{\varepsilon}_nz)\leq \rho_\mathcal{S}(\theta)\left(1+\frac{\Lip(\rho_\mathcal{S})\tilde{\varepsilon}_nR}{\inf_{\theta \in \mathcal{S}}\rho_\mathcal{S}(\theta)}\right)$ gives,
\begin{equation*}
\begin{aligned}
\hat{\mathcal{E}}_{\tilde{\varepsilon}_n,\infty}(\hat{f}) 
& \leq \int_{\mathcal{S}} \int_{\mathbb{R}^d} \hat{\eta}_\theta(z) | z \cdot \nabla \hat{f}(\theta) |^p \rho_{\mathcal{S}}(\theta)\rho_{\mathcal{S}}(\theta+\tilde{\varepsilon}_n z) \, \dd z \,\dd\theta \\
& \qquad \qquad + \sum_{k=0}^{p-1} \binom{p}{k} \tilde{\eps}_n^{p-k}\int_{\mathcal{S}} \int_{\mathbb{R}^d} \hat{\eta}_\theta(z) |z \cdot \nabla \hat{f}(\theta) |^k \left| \frac{1}{2} z^T D^2 \hat{f}(\theta_\tau) z \right|^{p-k}  \rho_{\mathcal{S}}(\theta)\rho_{\mathcal{S}}(\theta+\tilde{\varepsilon}_n z) \, \dd z \,\dd\theta \\
& \leq \left(1+C\tilde{\varepsilon}_n\right) \int_{\mathcal{S}} \int_{\mathbb{R}^d} \hat{\eta}_\theta(z) | z \cdot \nabla \hat{f}(\theta) |^p \rho_{\mathcal{S}}^2(\theta) \, \dd z \,\dd\theta \\
& \qquad \qquad + C\sum_{k=0}^{p-1} \binom{p}{k} \tilde{\varepsilon}_n^{p-k} \int_{\mathcal{S}} \int_{\mathbb{R}^d} \hat{\eta}_\theta(z) \|z\|^{2p-k} \|\nabla \hat{f}(\theta)\|^{k} \|D^2 \hat{f}(\theta_\tau)\|^{p-k}  \rho_{\mathcal{S}}^2(\theta) \, \dd z \, \dd\theta \\
& \to \int_{\mathcal{S}} \int_{\mathbb{R}^d} \hat{\eta}_\theta(z) | z \cdot \nabla \hat{f}(\theta) |^p \rho_{\mathcal{S}}^2(\theta) \, \dd z \,\dd\theta = \hat{\cE}_\infty(\hat{f})
\end{aligned}
\end{equation*}
as $n\to\infty$. This completes the proof of the lim-sup inequality.

\end{proof}
\end{document}